\documentclass[twoside,11pt]{article}

\usepackage[letterpaper,margin=0.75in,includefoot,centering]{geometry}
\usepackage{epsfig}
\usepackage{amssymb,amsfonts,mathrsfs,bm,mathtools}
\usepackage{natbib}
\usepackage{graphicx}
\bibpunct{(}{)}{;}{a}{,}{,}
\usepackage[breaklinks]{hyperref}
\usepackage{breakurl}
\usepackage{float}
\newcommand{\BlackBox}{\rule{1.5ex}{1.5ex}}  % end of proof
\newenvironment{proof}{\par\noindent{\bf Proof\ }}{\hfill\BlackBox\\[2mm]}
 
\newtheorem{theorem}{Theorem}
\newtheorem{lemma}[theorem]{Lemma} 
\newtheorem{proposition}[theorem]{Proposition} 

\newtheorem{corollary}[theorem]{Corollary}
\newtheorem{definition}[theorem]{Definition}

\usepackage{amssymb,amsmath,amsfonts,amscd}
\graphicspath{{./figs/}}
\usepackage{ifpdf}
\ifpdf
  \DeclareGraphicsExtensions{.png}
\else
  \DeclareGraphicsExtensions{.eps}
\fi
\usepackage{algorithmic}
\usepackage{algorithm}

\usepackage{subfig}
\usepackage{caption}
\usepackage{booktabs}
\usepackage{url}

\newcommand{\ra}[1]{\renewcommand{\arraystretch}{#1}}

\newcommand{\R}{\mathbb{R}}
\newcommand{\reals}{\mathbb{R}}

\newcommand{\G}{\mathrm{G}}

\newcommand{\E}{\mathbb{E}}
\newcommand{\Sb}{\mathbb{S}}

\newcommand{\diam}{\mathrm{diam}}

\newcommand{\OO}{\mathcal{O}}

\newcommand{\mb}{\mathbf}

\newcommand{\latop}[2]{\genfrac{}{}{0pt}{}{#1}{#2}}

\renewcommand{\dim}[1]{\mathrm{dim}(#1)}

\DeclareMathOperator{\Span}{Sp}
\DeclareMathOperator{\dist}{dist}

\DeclareMathOperator{\argmin}{\mathrm{argmin}}

\renewcommand{\E}{\operatorname{\mathbb{E}}}
\renewcommand{\P}{\operatorname{\mathbb{P}}}

\begin{document}

\title{Riemannian Multi-Manifold Modeling}

\author{Xu~Wang$^1$, Konstantinos~Slavakis$^2$, and Gilad~Lerman$^1$\\\\
       $^1$Dept.\ of Mathematics, University of Minnesota, Minneapolis, MN
       55455, USA\\ 
       \{wang1591,lerman\}@umn.edu\\ 
       $^2$Dept.\ of Electrical \& Computer Eng.\ and Digital Technology Center,\\
       University of Minnesota, Minneapolis, MN 55455, USA\\
       kslavaki@umn.edu
       }

\date{} 

\maketitle

\begin{abstract}
This paper advocates a novel framework for segmenting a dataset in a Riemannian manifold $M$ into clusters lying around low-dimensional submanifolds of $M$. Important examples of $M$, for which the proposed clustering algorithm is computationally efficient, are the sphere, the set of positive definite matrices, and the Grassmannian.
The clustering problem with these examples of $M$ is already useful for numerous application domains such as action identification in video sequences, dynamic texture clustering, brain fiber segmentation in medical imaging, and clustering of deformed images. The proposed clustering algorithm constructs a data-affinity matrix by thoroughly exploiting
the intrinsic geometry and then applies spectral clustering. The intrinsic local geometry is encoded by
local sparse coding
and more importantly by directional information of local tangent spaces and geodesics.
Theoretical guarantees are established for a simplified variant of the algorithm even when the clusters intersect. To avoid complication, these guarantees assume that the underlying submanifolds are geodesic.
Extensive validation on synthetic and real data demonstrates the resiliency of the proposed method against deviations from the theoretical model
as well as its superior performance over state-of-the-art techniques.
\end{abstract}

%\begin{keywords}
%Clustering, multi-manifold modeling, geodesics, Riemannian manifolds, high-dimensional data.
%\end{keywords}

\section{Introduction}
Many modern data sets are of moderate or high dimension, but manifest intrinsically low-dimensional structures. A natural quantitative framework for studying such common data sets is multi-manifold modeling (MMM)
or its special case of hybrid-linear modeling (HLM). In this MMM framework a given dataset is modeled as a union of submanifolds (whereas HLM considers union of subspaces).
When proposing a valid algorithm for MMM,
one assumes an underlying dataset that can be modeled as mixture of submanifolds and tries to prove under some conditions that the proposed algorithm
can cluster the dataset according to the submanifolds.
This framework has been extensively studied and applied for datasets embedded in the Euclidean space or the sphere~\citep{higher-order11, LocalPCA, centingul_vidal09,
  ElhamifarV_nips11, Kushnir06multiscale, icml2013_ho13, Lui12,   wang2011spectral}.

Nevertheless, there is an overwhelming number of application domains, where information is extracted from datasets that lie on Riemannian
manifolds, such as the Grassmannian, the sphere, the orthogonal group, or the manifold
of symmetric positive (semi)definite [P(S)D] matrices. For example,
auto-regressive moving average (ARMA) models are utilized to extract low-rank
linear subspaces (points on the Grassmannian) for identifying spatio-temporal
dynamics in video sequences~\citep{Turaga+11}. Similarly, convolving patches of
images by Gabor filters yields covariance matrices (points on the PD manifold)
that can capture effectively texture patterns in images~\citep{Tou09}.
Nevertheless, current MMM strategies are not sufficiently accurate for handling data in more general Riemannian spaces.

The purpose of this paper is to develop theory and algorithms
for the MMM problem in more general Riemannian spaces that are relevant to important applications.

% Traditional methods using local means of data points, such as
% $K$-means, and even their modern adaptation to Riemannian
% geometry~\citep{And06grassmannclustering} may fail to deal with non-convex
% shapes as well as nearby or intersecting clusters.

%To put this work in context, and prior to stating
%explicitly the advocated contributions, a short overview of state-of-the-art
%manifold clustering schemes is in order.

\paragraph{\bf Related Work.}

Recent advances in parsimonious data representations and their important
implications in dimensionality reduction techniques have effected the
development of non-standard spectral-clustering schemes that result in
state-of-the-art results in modern applications~\citep{LocalPCA,
  spectral_applied, Elhamifar09sparsesubspace, GOH_VIDAL08, zhu08multi,
  dict_learning_grassman13, lrr_long, LBF_journal12}. Such schemes rely on the
assumption that data exhibit low-dimensional structures, such as unions of
low-dimensional linear subspaces or submanifolds embedded in Euclidean spaces.

Several algorithms for clustering on manifolds are generalizations of well-known schemes
developed originally for Euclidean spaces. For example,
\citet{And06grassmannclustering} extended the classical $K$-means algorithm from
Euclidean spaces to Grassmannians, and illustrated an application to
nonnegative matrix factorization. \citet{1541234} capitalized on the Riemannian
distance of $\text{SO}(3)$ to design an efficient mean-shift (MS) algorithm for
multiple 3D rigid motion estimation. \citet{1640882}, as well as
\citet{centingul_vidal09}, extended further the MS algorithm to general analytic
manifolds including Grassmannians, Stiefel manifolds, and matrix Lie
groups. \citet{5771473} showed promising results by using the geodesic distance
of product manifolds in clustering of human expressions, gestures, and actions
in videos. \citet{rathi07segmenting} solved the image segmentation problem,
after recasting it as a matrix clustering problem, via probability distributions
on symmetric PD matrices. \citet{GOH_VIDAL08} extended spectral clustering and
nonlinear dimensionality reduction techniques to Riemannian manifolds. These
previous works are quite successful when the convex hulls of individual clusters
are well-separated, but they often fail when clusters intersect or are closely
located.

HLM and MMM accommodate low-dimensional data structures by unions of subspaces or
submanifolds, respectively, but are restricted to manifolds embedded in
either a Euclidean space or the sphere.   Many strategies have been suggested for solving the
HLM problem, known also as subspace clustering. These strategies include methods
inspired by energy minimization~\citep{Bradley00kplanes, Ho03, Ma07Compression,
  gdm14, Tseng00nearest, MKF_workshop09, LBF_cvpr10}, algebraic
methods~\citep{Boult91factorization-basedsegmentation, Costeira98, Kanatani01,
  Kanatani02, Ma07, Ozay10, Vidal05}, statistical
methods~\citep{Tipping99mixtures, Yang06Robust}, and spectral-type methods with
various types of affinities representing subspace-related
information~\citep{spectral_applied, ssc_elhamifar13, lrr_long, Yan06LSA,
  LBF_journal12}. Recent tutorial papers on HLM are
\citet{SubspaceClustering_Vidal} and \citet{Aldroubi_review_13}. Some
theoretical guarantees for particular HLM algorithms appear
in~\citep{spectral_theory, lp_recovery_part2_11, soltan_candes12,
  solton_elhami_candes14}.  There are fewer strategies for the MMM problem, which is also known as manifold clustering. They include higher-order spectral clustering~\citep{higher-order11},
spectral methods based on local PCA~\citep{LocalPCA, zhu08multi,
  Gong2012, Kushnir06multiscale, wang2011spectral}, sparse-coding-based spectral
clustering in a Euclidean space~\citep{ElhamifarV_nips11} and its modification to the sphere by~\citet{6619442} (the sparse coding encodes local subspace
approximation), energy minimization strategies~\citep{energy07}, methods based on manifold learning algorithms~\citep{polito2001grouping, Souvenir05}, and methods based on clustering dimension or local density~\citep{Barbará00usingthe,
  Gionis:2005:DIC:1081870.1081880, Haro06}. Notwithstanding, only higher-order
spectral clustering and spectral local PCA are theoretically
guaranteed~\citep{higher-order11, LocalPCA}.

In a different context, \citet{rahman05} suggested
multiscale strategies for signals taking values in Riemannian manifolds, in
particular, the sphere, the orthogonal group, the Grassmannian, and the PD
manifold. Even though \citet{rahman05} addresses a completely different problem,
its basic principle is similar in spirit to ours and can be described as
follows. Local analysis is performed in the tangent spaces, where the
exponential and logarithm maps are used to transform data between local
manifold neighborhoods and local tangent space neighborhoods. Information from
all local neighborhoods is then integrated to infer global properties.

\paragraph{\bf Contributions.} Despite the popularity of manifold learning, the
associated literature lacks generic schemes for clustering low-dimensional data
embedded in non-Euclidean spaces. Furthermore, even in the Euclidean setting only few algorithms
for MMM or HLM are theoretically guaranteed. To this end, this paper aims at filling this gap and
provides an MMM approach in non-Euclidean setting with some theoretical guarantees
even when the clusters intersect.
In order to avoid nontrivial theoretical obstacles, the theory assumes that the underlying submanifolds
are geodesic and refer to it as \textit{multi-geodesic modeling} (MGM).
%Specifically,
%MGM considers data points lying on a \textit{known} Riemannian manifold $M$ and assumes that
%they are located around a union of \textit{unknown} geodesic submanifolds of $M$.
Clearly, this modeling paradigm is a direct generalization of HLM from Euclidean spaces
to Riemannian manifolds.
A more practical and robust variant of
the theoretical algorithm is also developed, and its superior performance over
state-of-the-art clustering techniques is exhibited by extensive validation on
synthetic and real datasets. We remark that in practice we require that the logarithm map of $M$
can be computed efficiently and we show that this assumption does not restrict the wide applicability of this work.

We believe that it is possible to extend the theoretical foundations of this work to deal
with general submanifolds by using local geodesic submanifolds (in analogy
to~\citet{LocalPCA}). However, this will significantly increase the complexity of our
proof, which is already not simple.
Nevertheless, the proposed method directly applies to the more general setting (without theoretical guarantees) since geodesics are only used in local neighborhoods and not globally.
Furthermore, our numerical experiments show that the proposed method works well in real practical scenarios that deviate from the theoretical model.

On a more technical level, the paper is distinguished from previous works in
multi-manifold modeling in its careful incorporation of ``directional
information,'' e.g., local tangent spaces and geodesics. This is done for two
purposes: (i) To distinguish submanifolds at intersections; (ii) to filter out
neighboring points that belong to clusters different than the cluster of the
query point. In such a way, the proposed algorithm allows for neighborhoods to
include points from different clusters, while previous multi-manifold algorithms
(e.g., \citet{ElhamifarV_nips11}) need careful choice of neighborhood radii to
avoid points belonging to other clusters.

\section{Theoretical Preliminaries}\label{sec:problem}

We formulate the theoretical problem of MGM and review preliminary background of Riemannian geometry, which is necessary to follow this work.

\subsection{Multi-Geodesic Modeling (MGM)}\label{sec:generativeMGM} MGM assumes
that each point in a given dataset $X= \{x_i\}_{i=1}^N$ lies in the tubular
neighborhood of some unknown geodesic submanifold $S_k$, $1\leq k\leq K$, of a
Riemannian manifold, $M$.\footnote{The tubular neighborhood with radius $\tau>0$ of $S_k$ in $M$ (with metric tensor $g$ and induced distance $\dist_g$) is $S_k^{\tau}=\{ x\in M : \dist_g(x,s)<\tau \text{ for some } s \in S_k \}$.} The goal is to cluster the dataset $X$ into $K$
groups $X_1, \ldots, X_K\subset M$ such that points in $X_k$ are associated with
the submanifold $S_k$. Note that if $M$ is a Euclidean space, geodesic
submanifolds are subspaces and MGM boils down to HLM, or equivalently, subspace
clustering~\citep{Elhamifar09sparsesubspace, SubspaceClustering_Vidal,
  LBF_journal12}.

For theoretical purposes, we assume the following data model, which we refer to
as uniform MGM: The data points are i.i.d.~sampled w.r.t.~the uniform
distribution on a fixed tubular neighborhood of $\cup_{k=1}^K  S_k$. We denote
the radius of the tubular neighborhood by $\tau$ and refer to it as the noise
level.
% We note that the Riemannian distance between any point $x\in X$ with $\displaystyle \cup_{k=1}^K S_k$ is less than $\tau$.
Figure~\ref{fig:multi-geodesic} illustrates data generated from uniform MGM with
two underlying submanifolds ($K=2)$.

\begin{figure*}[htb!]
\centering
\includegraphics[width=.45\textwidth]{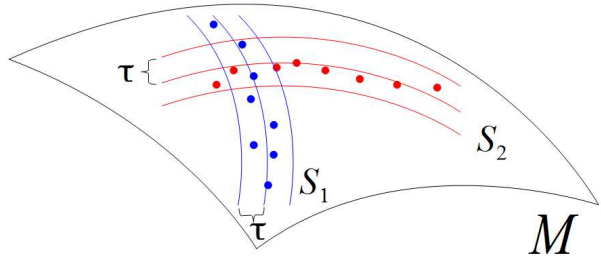}
\caption{Illustration of data generated from a uniform MGM when $K=2$.}
\label{fig:multi-geodesic}
\end{figure*}

%changed
%Dataset $X$ is assumed to be i.i.d.~sampled with respect to the
%  uniform distribution from the $\tau$-tubular neighborhood of $\cup_{k=1}^K
%  S_k$ (cf., Figure~\ref{fig:multi-geodesic}). The task is to cluster $X$ into
%  $K$ groups $\displaystyle \cup_{k=1}^K X_k\subset M$ such that points in $X_k$
%  are generated from the submanifold $S_k$. Note that if $M$ is a Euclidean
%  space, geodesic submanifolds are subspaces. As a result, MGM boils down to
%  hybrid linear modeling~\citep{LBF_journal12}.

%the data is not necessarily i.i.d., just for justifying the theory.

The MGM problem only serves our theoretical justification.
The numerical experiments show that the proposed algorithm works well under a more general MMM setting.
Such a setting may include more general submanifolds (not necessarily geodesic), non-uniform sampling and different kinds and levels of noise.

\subsection{Basics of Riemannian Geometry}\label{sec:notation}

This section reviews basic concepts from Riemannian geometry; for extended and
accessible review of the topic we recommend the textbook
by~\citet{docarmo92}. Let $(M,g)$ be a $D$-dimensional Riemannian manifold with
a metric tensor $g$. A geodesic between $x, y\in M$ is a curve in $M$ whose
length is locally minimized among all curves connecting $x$ and $y$. Let
$\dist_g(x,y)$ be the Riemannian distance between $x$ and $y$ on $M$. If
$T_{x}M$ denotes the tangent space of $M$ at $x$, then $T_{x}S$ stands for the
tangent subspace of a $d$-dimensional geodesic submanifold $S$ at $x$. As shown
in Figure~\ref{fig:TangentOfS}, $T_{x}S$ is a linear subspace of $T_{x}M$. The
exponential map $\exp_x$ maps a tangent vector $\mb{v}\in T_x M$ to a point
$\exp_x (\mb{v})\in M$, which provides local coordinates around $x$. By
definition, the geodesic submanifold $S$ is the image of $T_x S$ under $\exp_x$
(cf., Definition~\ref{def:geodesic}). The functional inverse of $\exp_x$ is the
logarithm map $\log_x$ from $M$ to $T_x M$, which maps $x$ to the origin
$\mb{O}$ of $T_{x}M$. Let $\mb{x}_j^{(i)}$ denote the image of a data point $x_j$ in $T_{x_i}M$ by the
logarithm map at $x_i$; that is, $\mb{x}_j^{(i)} =\log_{x_i}(x_j)$.

\begin{figure*}[htb!]
\centering
\subfloat[\footnotesize Tangent space]
{\label{fig:TangentOfS} \includegraphics[width=.45\textwidth]{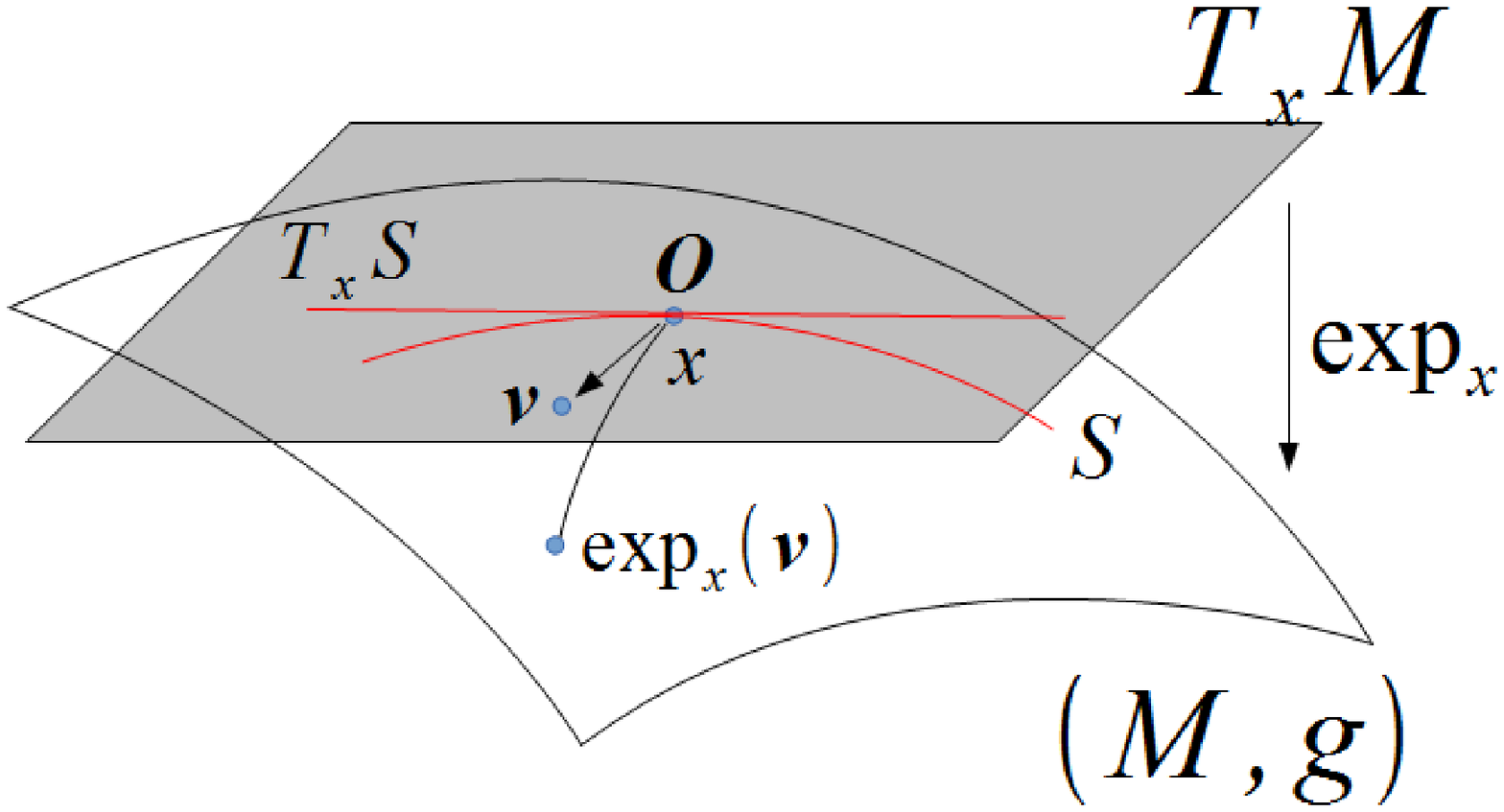}}
\subfloat[\footnotesize Logarithm map]
{\label{fig:logarithm} \includegraphics[width=.45\textwidth]{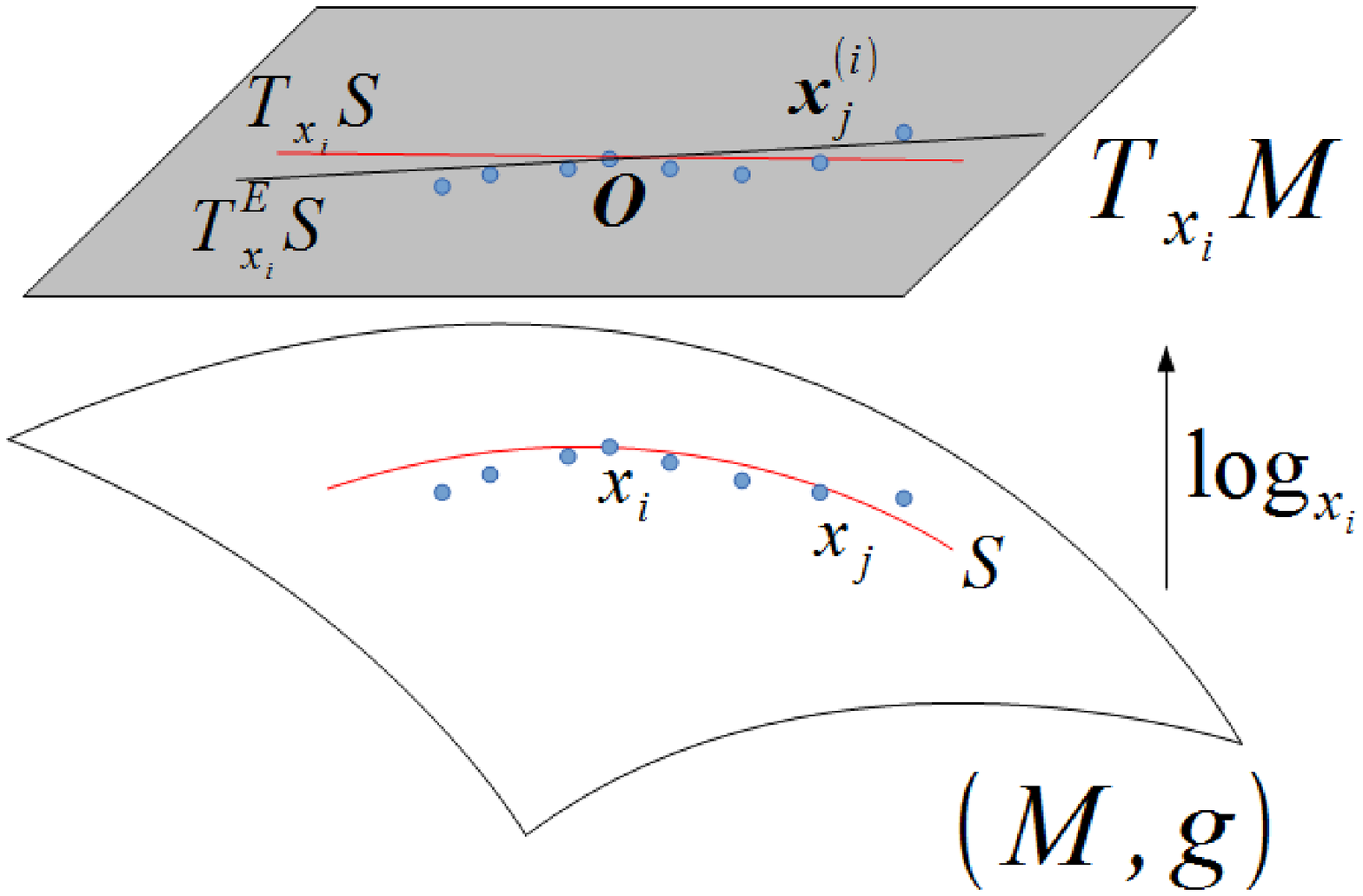}}
\caption{Demonstration of the exponential and logarithm maps as well as the
  tangent and estimated subspaces. (a) The tangent space and the exponential map
  of a manifold $(M,g)$ at a point $x \in M$. Note that the tangent subspace $T_x
  S$ is a pre-image of $S$ under the exponential map. (b) The logarithm map $\log_{x_i}$
  w.r.t.~$x_i \in S$ and the images by $\log_{x_i}$ of data
  points in a local neighborhood of $x_i$, in particular, $\mb{x}_j^{(i)}$ , the image of
  $x_j$.  Note the difference between $T_{x_i}S$, which is the image of $S$
  under $\log_{x_i}$, and the subspace $T_{x_i}^E S$ estimated by the images of
  the data points in the local neighborhood.  }
\end{figure*}

%\paragraph{\bf Extrinsic embedding}\label{example:embedding} One solution is to embed Manifolds into Euclidean spaces. However, the following example shows that the embedding can enlarge the curvature. Suppose the ambient manifold $M=[-1, 1]\times[0, \pi/2]$, the line segment $S=\{0\}\times[0, \pi/2]$. $M$ is isometrically embedded into $\R^3$ as follows.
%\[
%(x,y)\rightarrow (x, \cos(y),\sin(y)) \qquad \forall (x,y)\in M.
%\]
%Extra curvature is introduced since $S$ are mapped to a curve in $\R^3$.

\section{Solutions for the MGM (or MMM) Problem in $M$}\label{sec:algorithm}

We suggest solutions for the MMM problem in $M$ with theoretical guarantees supporting one of these solutions when restricting the problem to MGM.
Section~\ref{sec:directional} defines two key quantities for quantifying directional information: Estimated local tangent subspaces and geodesic angles.
Section~\ref{sec:two_solutions} presents the two solutions and discusses their properties.

\subsection{Directional Information}\label{sec:directional}

\paragraph{\bf The Estimated Local Tangent Subspace $T_{x_i}^E S$.}
Figure~\ref{fig:logarithm} demonstrates the main quantity defined here
($T^E_{x_i} S$) as well as related concepts and definitions.
It assumes a dataset $X=\{x_j\}_{j=1}^N \subset M$
generated by uniform MGM with a single geodesic submanifold $S$. The dataset is thus
contained in a tubular neighborhood of a $d$-dimensional geodesic submanifold
$S$. Since $S$ is geodesic, for any $1\leq i \leq N$ the set
$\{\mb{x}_j^{(i)}\}_{j=1}^N$ of images by the logarithm map is contained in a  tubular neighborhood of the
$d$-dimensional subspace $T_{x_i} S$ (possibly with a different radius than
$\tau$).

Since the true tangent subspace $T_{x_i} S$ is unknown, an estimation of it,
$T_{x_i}^E S$, is needed. Let $B(x_i,r)\subset M$ be the neighborhood of $x_i$
with a fixed radius $r>0$. Let
\begin{equation}
\label{eq:def_J}
J(x,r) := \{ j: x_j\in B(x,r)\cap X\}.
\end{equation}
Moreover, let $\mb{C}_{x_j}$ denote the local sample covariance matrix of
the dataset $\displaystyle \{\mb{x}_j^{(i)}\}_{j \in J(x_i, r)}$ on $T_{x_i}M$, and
$\|\mb{C}_{x_j}\|$ the spectral norm of $\mb{C}_{x_j}$, i.e., its maximum
eigenvalue. Since $\{\mb{x}_j^{(i)}\}_{j=1}^N$ is in a tubular neighborhood of a
$d$-dimensional subspace, estimates of the intrinsic dimension $d$ of the local
tangent subspace, which is also the dimension of $S$, can be formed by
bottom eigenvalues of $\mb{C}_{x_j}$ (cf., \citet{LocalPCA}). We
adopt this strategy of dimension estimation and define the estimated local
tangent subspace, $T_{x_i}^E S$, as the span in $T_{x_i} M$ of the top eigenvectors
of $\mb{C}_{x_j}$. In theory,
the number of top eigenvectors is the number of eigenvalues of $\mb{C}_{x_i}$ that exceed $\eta \|\mb{C}_{x_j}\|$ for some
fixed $0<\eta<1$ (see Theorem~\ref{theorem:all} and its proof for the choice of $\eta$).
In practice, the number of top eigenvectors is the number of top eigenvalues $\mb{C}_{x_i}$ until the largest gap occurs.

\paragraph{\bf Empirical Geodesic Angles.} Let $l(x_i,x_j)$ be the shortest
geodesic (global length minimizer) connecting $x_i$ and $x_j$ in $(M,g)$. Let
$\mb{v}_{ij}\in T_{x_i}M$ be the tangent vector of $l(x_i,x_j)$ at $x_i$. In
other words, $\mb{v}_{ij}$ shows the direction at $x_i$ of the shortest path
from $x_i$ to $x_j$. Given a dataset $X=\{x_j\}_{j=1}^N$, the empirical geodesic
angle $\theta_{ij}$ is the elevation angle (cf., (9) of~\citet{LW-semimetric}) between the vector $\mb{v}_{ij}$ and
the subspace $T_{x_i}^E S$ in the Euclidean space $T_{x_i} M$.

\subsection{Proposed Solutions}
\label{sec:two_solutions}
In Section~\ref{sec:TGCT}, we propose a theoretical solution for data sampled according to uniform MGM. We start with its basic motivation, then describe the proposed algorithm and
at last formulate its theoretical guarantees. In Section~\ref{sec:GCT}, we propose a practical algorithm. At last, Section~\ref{sec:comp_complex} discusses the numerical complexity of both algorithms.

\subsubsection{Algorithm~\ref{alg:theory}: Theoretical Geodesic Clustering with Tangent information (TGCT)}
\label{sec:TGCT}
The proposed solution for the MGM-clustering task applies spectral
clustering with carefully chosen weights. Specifically, a similarity graph is constructed whose
vertices are data points and whose edges represent the similarity between data
points. The challenge is to construct a graph such that two points are locally
connected only when they come from the same cluster. This way spectral clustering will recover exactly the
underlying clusters.

For the sake of illustration, let us assume only two underlying geodesic
submanifolds $S_1$ and $S_2$. We also assume that the data was sampled from $S_1 \cup S_2$ according to uniform MGM. Given a point $x_0\in S_1$ one wishes to connect
to it the points from the same submanifold within a local neighborhood
$B(x_0,r)$ for some $r>0$. Clearly, it is not realistic to assume that all
points in $B(x_0,r)$ are from the same submanifold of $x_0$ (due to nearness and
intersection of clusters as demonstrated in Figures~\ref{fig:angle_threshold}
and~\ref{fig:intersection}).
  %Thus, it is not realistic to assume
  %that all points in $B(x_0,r)$ come necessarily from the same cluster.

We first assume no intersection at $x_0$ as demonstrated in
Figure~\ref{fig:angle_threshold}.  In order to be able to identify the points in
$B(x_0,r)$ from the same submanifold of $x_0$, we use local tangent information
at $x_0$.  If $x \in B(x_0,r)$ belongs to $S_2$, then the geodesic $l(x_0,x)$
has a large angle with the tangent space $T_{x_0}S_1$ at $x_0$. On the other
hand, if such $x$ belongs to $S_1$, then the geodesic has an angle close to
zero.  Therefore, thresholding the empirical geodesic angles may become
beneficial for eliminating neighboring points belonging to a different
submanifold (cf., Figure~\ref{fig:angle_threshold}).

\begin{figure*}[htb!]
\centering
\subfloat[\footnotesize Angle filtering]
{\label{fig:angle_threshold} \includegraphics[width=.45\textwidth]{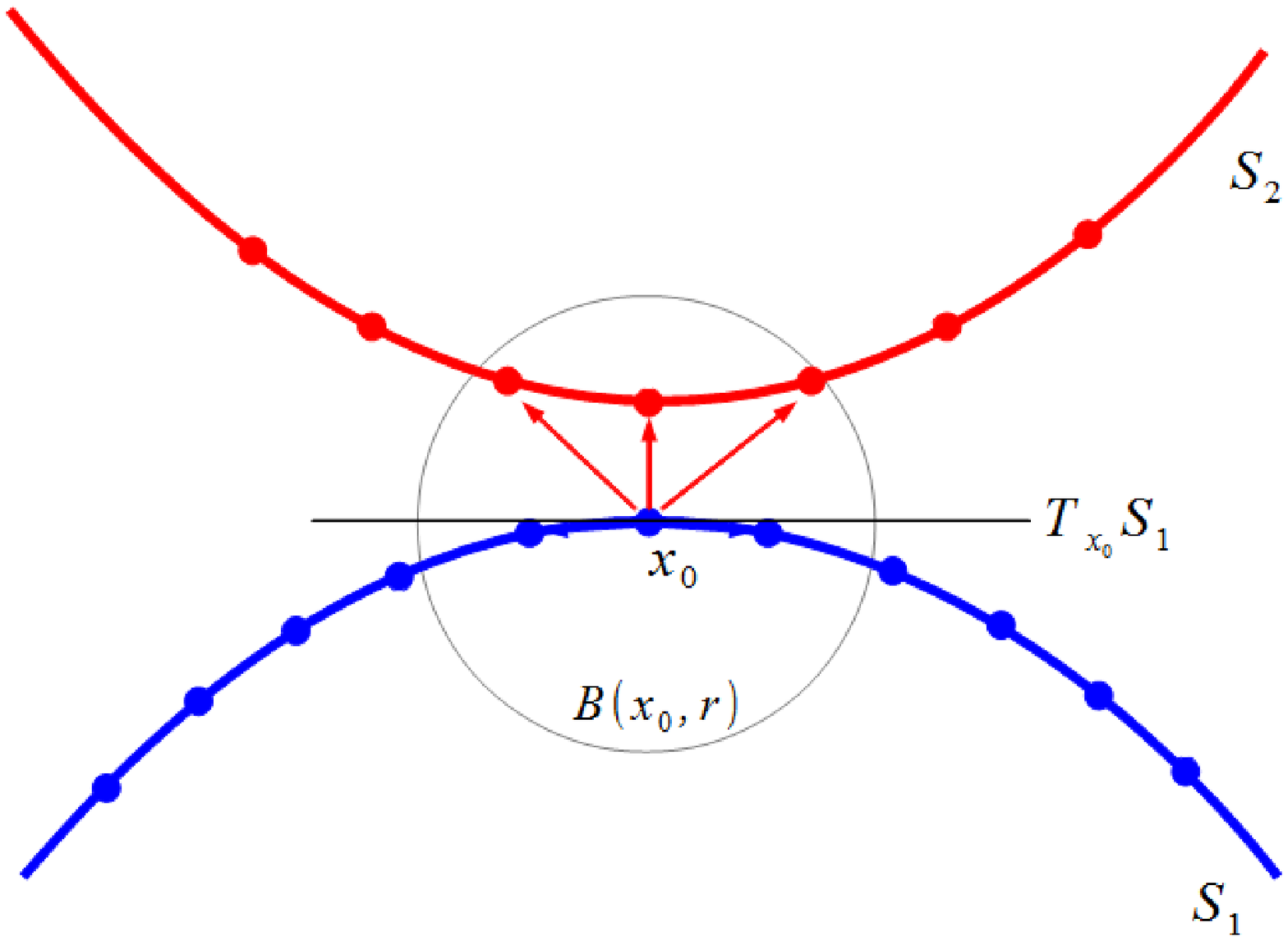}}
\subfloat[\footnotesize Intersection]
{\label{fig:intersection} \includegraphics[width=.45\textwidth]{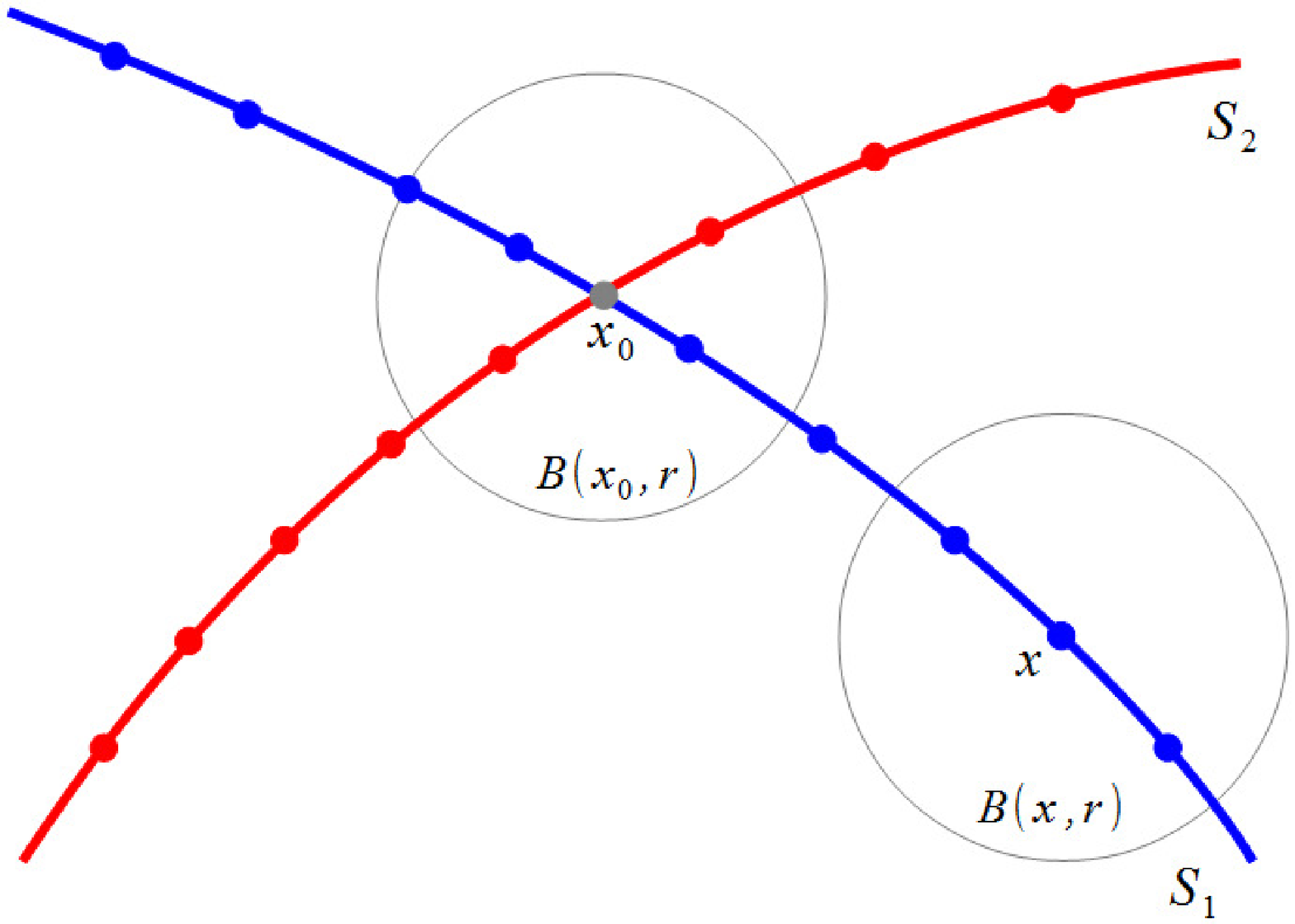}}
\caption{In Figures~\ref{fig:angle_threshold} and~\ref{fig:intersection}, blue
  points lie on the submanifold $S_1$ and red points lie on the submanifold
  $S_2$.  In Figure~\ref{fig:angle_threshold}, a local neighborhood, which is a
  disk of radius $r$, around the blue point $x_0$ is observed and the goal is to
  exclude the red points in $B(x_0,r)$. This can be done by thresholding the
  angles between geodesics and the tangent subspace $T_{x_0}S_1$.  Indeed, the
  angles w.r.t.~blue points are close to zero and the angles w.r.t.~red points
  are sufficiently large.  In Figure~\ref{fig:intersection}, a point $x_0$ is
  in $S_1 \cup S_2$ and an arbitrary point $x$ sufficiently far from it. The
  goal is to assure that $x$ is not connected to $x_0$.
  This can be done by comparing local estimated
  dimensions. The estimated dimension in $B(x_0,r)$ is $\dim {S_1}+\dim {S_2}$,
  while the estimated dimension in $B(x,r)$ is $\dim{S_1}$. Due to the dimension
  difference, the intersection is disconnected from the two submanifolds.}
\label{fig:ingredients}
\end{figure*}

If $x_0$ is at or near the intersection, it is hard to estimate correctly the
tangent spaces of each submanifold and the geodesic angles may not be reliable.
Instead, one may compare the dimensions of estimated local tangent
subspaces. The estimated dimensions of local neighborhoods of data points, which
are close to intersections, are larger than the estimated dimensions of local
neighborhoods of data points further away from intersections (cf.,
Figure~\ref{fig:intersection}). The algorithm thus connects $x_0$ to other
neighboring points only when their ``local dimensions'' (linear-algebraic
dimension of the estimated local tangent) are the same. In this way, the
intersection will not be connected with the other clusters.

The dimension difference criterion, together with the angle filtering procedure,
guarantee that there is no false connection between different clusters (the
rigorous argument is established in the proof of Theorem~\ref{theorem:all}).  We
use these two simple ideas and the common spectral-clustering procedure to form
the Theoretical Geodesic Clustering with Tangent information (TGCT) in
Algorithm~\ref{alg:theory}.

\begin{algorithm}
\caption{Theoretical Geodesic Clustering with Tangent information (TGCT)}\label{alg:theory}
\begin{algorithmic}
\REQUIRE Number of clusters: $K\geq 2$, a dataset $X$ of $N$ points, a
neighborhood radius $r$, a projection threshold $\eta$ for estimating tangent
subspaces, a distance threshold $\sigma_d$ and an angle threshold $\sigma_a$.
\ENSURE Index set $\{\text{Id}_i\}_{i=1}^N$ such that $\text{Id}_i  \in
\{1,\ldots,K\}$ is the cluster label assigned to $x_i$ \\
\textbf{Steps}:\\

$\bullet$ Compute the following geometric quantities around each point:\\
\FOR {$i=1,\ldots,N$}
\STATE $\circ$  For $j \in J(x_i,r)$ (c.f., \eqref{eq:def_J}), compute $\mb{x}_j^{(i)}=\log_{x_i}(x_j)$\\
$\circ$  Compute the sample covariance matrix $\mb{C}_{x_i}$ of
$\{\mb{x}_j^{(i)}\}_{j \in  J(x_i,r)}$\\
$\circ$  Compute the eigenvectors of $\mb{C}_{x_i}$ whose eigenvalues exceed
$\eta\cdot \|\mb{C}_{x_i}\|$ (their span is $T_{x_i}^E S$)\\
$\circ$  For all $j=1,\ldots,N$, compute the empirical geodesic angles
$\theta_{ij}$ (see Section~\ref{sec:directional})\\
% between $\mb{v}_{ij}$ and $T_{x_i}^E S$\\
 \ENDFOR

$\bullet$ Form the following $N \times N$ affinity matrix
$\mb{W}$: $$\mb{W}_{ij}=\mathbf{1}_{\dist_g(x_i,x_j)<\sigma_d}\mathbf{1}_{\dim
  {T_{x_i}^E S}=\dim {T_{x_j}^E S}}\mathbf{1}_{(\theta_{ij}+\theta_{ji})<\sigma_a}$$\\
$\bullet$ Apply spectral clustering to the affinity matrix $\mb{W}$ to determine
the output $\{\text{Id}_i\}_{i=1}^N$\\

\end{algorithmic}
\end{algorithm}

The following theorem asserts that TGCT achieves correct clustering with high
probability. Its proof is in Section~\ref{sec:theory}. Its statement relies on
the constants $\{C_i\}_{i=0}^6$ and $C'_0$, which are clarified in the proof and
depend only on the underlying geometry of the generative model. For simplicity,
the theorem assumes that there are only two geodesic submanifolds and that they
are of the same dimension. However, it can be extended to $K$ geodesic
submanifolds of different dimensions.

\begin{theorem}\label{theorem:all}
Consider two smooth compact $d$-dimensional geodesic submanifolds, $S_1$ and
$S_2$, of a Riemannian manifold and let $X$ be a dataset generated according to
uniform MGM w.r.t.~$S_1 \cup S_2$ with noise level $\tau$. % (cf.,
                                % Section~\ref{sec:generativeMGM}).
If the positive parameters of the TGCT algorithm, $r$, $\sigma_d$, $\sigma_a$
and $\eta$, satisfy the inequalities
\begin{align}
\label{eq:constant_all_requirements}
& \eta < C_2^{-\frac{d+2}{2}}, \sigma_d<{C_4}^{-\frac{1}{2}}, \ r > \tau / C_5, \ r < \min(\eta, \sigma_d, \sigma_a)/C_1 \ \text{ and }\\
&\sigma_a < \min(\sin^{-1}(r\sqrt{1-C_2\eta^{\frac{2}{d+2}}}/(2\sigma_d))-C_3 \eta^{\frac{d}{d+2}}-C_3 r, \pi/6),
\nonumber
\end{align}
then with probability at least $1-C_0 N\exp[-Nr^{d+2}/C'_0]$, the TGCT algorithm
can cluster correctly a sufficiently large subset of $X$, whose relative
fraction (over $X$) has expectation at least $1- C_6 (r+\tau)^{d-\dim{S_1 \cap S_2}}$.
\end{theorem}

\subsubsection{Algorithm~\ref{alg:experiment}: Geodesic Clustering with Tangent information (GCT)}
\label{sec:GCT}
A practical version of the TGCT algorithm, which we refer to as Geodesic Clustering
with Tangent information (GCT), is described in Algorithm~\ref{alg:experiment}.
This is the algorithm implemented for the experiments in
Section~\ref{sec:experiment} and its choice of parameters is clarified in Section~\ref{sec:clusterings_param}.
GCT differs from TGCT in three different ways. First,
hard thresholds in TGCT are replaced by soft ones, which are more
flexible. Second, the dimension indicator function is dropped from the affinity
matrix $W$. Indeed, numerical experiments indicate that the algorithm works
properly without the dimension indicator function, whenever there is only a
small portion of points near the intersection. This numerical observation makes
sense since the dimension indicator is only used in theory to avoid connecting
intersection points to points not in intersection.  At last, pairwise distances
are replaced by weights resulting from sparsity-cognizant optimization
tasks. Sparse coding takes advantage of the low-dimensional structure of
submanifolds and produces larger weights for points coming from the same
submanifold~\citep{ElhamifarV_nips11}.

\begin{algorithm}
\caption{Geodesic Clustering with Tangent information (GCT)}\label{alg:experiment}
\begin{algorithmic}
\REQUIRE Number of clusters: $K\geq 2$, a dataset $X$ of $N$ points, a neighborhood radius $r$, a distance threshold $\sigma_d$ (default: $\sigma_d=1$) and an angle threshold $\sigma_a$ (default $\sigma_a=1$)
\ENSURE Index set $\{\text{Id}_i\}_{i=1}^N$ such that $\text{Id}_i  \in \{1,\ldots,K\}$ is the cluster label assigned to $x_i$\\
\textbf{Steps}:\\
\FOR{$i=1,\ldots,N$}
\STATE $\circ$   For $j \in J(x_i,r)$, compute $\mb{x}_j^{(i)}=\log_{x_i}(x_j)$\\
$\circ$   Compute the weights $\{\mb{S}_{ij}\}_{j\in J(x_i,r)}$ that minimize
\begin{equation}\label{equ:l1minimization}
\|\mb{x}_i^{(i)}-\sum_{\latop{j\in J(x_i,r)}{j\not=i}} \mb{S}_{ij}
\mb{x}_j^{(i)}\|_2^2+\sum_{\latop{j\in J(x_i,r)}{j\not=i}}
e^{\|\mb{x}_i^{(i)}-\mb{x}_j^{(i)}\|_2/\sigma_d} |\mb{S}_{ij}|
\end{equation}
among all $\{\mb{S}_{ij}\}_{j\in J(x_i,r)}$ such that $\mb{S}_{ii}=0$ and
$\sum_{\latop{j\in J(x_i,r)}{j\not=i}} \mb{S}_{ij} =1$\\
$\circ$   Complete these weights as follows: $\mb{S}_{ij} =0$ for $j \notin J(x_i,r)$\\
$\circ$   Compute the sample covariance matrix $\mb{C}_{x_i}$ of
$\{\mb{x}_j^{(i)}\}_{j \in  J(x_i,r)}$\\
$\circ$   Find the largest gap between eigenvalues $\lambda_m$ and $\lambda_{m+1}$ of $\mb{C}_{x_i}$ and compute
the top $m$ eigenvectors of $\mb{C}_{x_i}$ (their span is $T_{x_i}^E S$)\\
$\circ$   For all $j=1,\ldots,N$, compute the empirical geodesic angles
$\theta_{ij}$ (see Section~\ref{sec:directional})\\
\ENDFOR

$\bullet$  Form the following $N \times N$ affinity matrix $\mb{W}$:
\begin{equation}\label{equ:sparseweights}
\mb{W}_{ij}=e^{|\mb{S}_{ij}|+|\mb{S}_{ji}|}e^{-(\theta_{ij}+\theta_{ji})/\sigma_a}
\end{equation}
$\bullet$  Apply spectral clustering to the affinity matrix $W$ to determine the
output $\{\text{Id}_i\}_{i=1}^N$\\

\end{algorithmic}
\end{algorithm}

Algorithm~\ref{alg:experiment} solves a sparse coding task in
\eqref{equ:l1minimization}. The penalty used is non-standard since the codes
$|\mb{S}_{ij}|$ are multiplied by $e^{\|\mb{x}_i^{(i)}-\mb{x}_j^{(i)}\|_2/\sigma_d}$
(where in~\citet{6619442}, these latter terms are all 1). These weights were
chosen to increase the effect of nearby points (in addition to their
sparsity). In particular, it avoids sparse representations via far-away points
that are unrelated to the local manifold structure (see further explanation in
Figure~\ref{fig:sparse_weight}). Similarly to~\citet{6619442}, the
clustering weights in~\eqref{equ:sparseweights} exponentiate the sparse-coding weights.

\begin{figure*}[tb!]
\centering
\includegraphics[width=.45\textwidth]{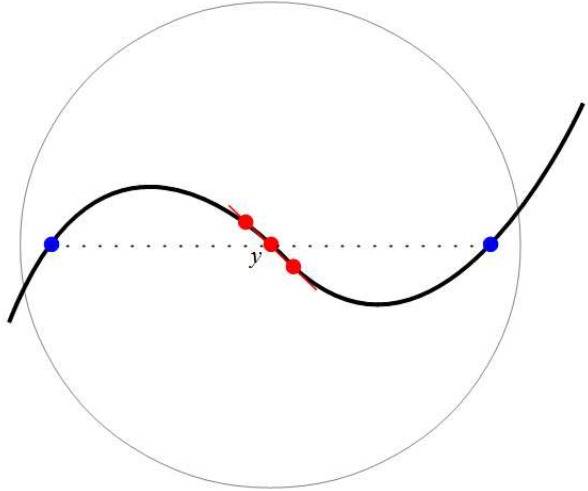}
\caption{Illustration of the need for weighted sparse optimization in~\eqref{equ:l1minimization}. The non-weighted
  sparse optimization may fail to detect the local structure at $y$ in the
  manifold setting. The term $e^{\|\mb{x}_i^{(i)}-\mb{x}_j^{(i)}\|_2/\sigma_d}$
  is used to avoid assigning large weights to the far-away blue points.}
\label{fig:sparse_weight}
\end{figure*}

\subsubsection{Computational Complexity of GCT and TGCT}
\label{sec:comp_complex}
We briefly discuss the computational complexity of GCT and TGCT, while leaving many technical
details to Appendices~\ref{sec:logarithmmaps} and~\ref{sec:complexity.GCT}.
The computational complexity of GCT is
$$\OO(N^2(\text{CR} +\text{CL}+D)+kN\log(N)+ND+Nk^3),$$
where $k$ bounds the number of nearest neighbors in a neighborhood (typically $k=30$ by the choice of parameters), CR is the cost of computing the Riemannian distances between any two points and CL is the cost of computing the logarithm map of a given point w.r.t.~another point. Furthermore, once $CL$ was computed, CR$= \OO(D)$. The complexity of CL depends on the Riemannian manifold $M$. If $M=\mathbb{S}^D$, then CL$=\OO(D)$. If $M$ is the space of symmetric PD matrices and $\dim{M}=D$, then CL=$\OO(D^{1.5})$. If $M$ is the Grassmannian, $\dim{M}=D$ and $d$ is chosen to be of the same order as the dimension of the subspaces in $M$, then CL=$d(d+D/d)^2$. In all applications of Riemannian multi-manifold modeling we are aware of $M$ is known and it is one of these examples. For more general or unknown $M$, estimation of the logarithm map is discussed in~\cite{MemoliS05} (this estimation is rather slow).

It is possible to reduce the total computational cost under some assumptions. In particular, in theory, it is possible to implement TGCT (or more precisely an approximate variant of it) for the sphere or the Grassmannian with computational complexity of order $$\OO(N^{1+\rho} \text{CR}+(k+1)N\log(N)+kN(\text{CL}+D)+Nk^3),$$ where $\rho>0$ is near zero.

\section{Numerical Experiments}\label{sec:experiment}

To assess performance on both synthetic and real datasets, the GCT algorithm is
compared with the following algorithms: Sparse manifold clustering (SMC)~\citep{ElhamifarV_nips11,6619442},
which is adapted here for clustering within a Riemannian manifold and still referred to as SMC, spectral clustering
with Riemannian metric (SCR) of~\citet{GOH_VIDAL08}, and embedded $K$-means
(EKM). The three methods and choices of parameters for all four methods are
reviewed in Appendix~\ref{sec:clusterings_param}.

The ground truth labeling is given in each experiment. To measure the accuracy of
each method, the assigned labels are first permuted to have the maximal match
with the ground truth labels. The clustering rate is computed for that permuted
labels as follows:
\[
\displaystyle \text{clustering rate} = \frac{\# \text{ of points whose group
    labels are the same as ground truth labels}}{\# \text{ of total points}}.
\]

\subsection{Experiments with Synthetic Datasets}\label{subsec:synthetic}
%First, GCT is evaluated by
%  using synthetic datasets from three types of manifolds (to be explained in
%  detail in the sequel). Here, the main objective is to demonstrate that GCT is
%  capable of detecting submanifold directions and of ruling out neighboring
%  points not in the same group. Such obstacles (e.g., intersections, neighboring
%  points from other groups) are often met in real data.

Six datasets were generated. Dataset I and II are from the
  Grassmannian $\G(6,2)$, datasets III and IV are from $3\times 3$ symmetric
  positive-definite (PD) matrices, and datasets V and VI are from the sphere
  $\Sb^2$. Each dataset contains $260$ points generated from two ``parallel'' or
  intersecting submanifolds (130 points on each) and cropped by white Gaussian noise. The exact constructions
are described below.

%\subsubsection{Constructions of the Synthetic Datasets}
%
%The six synthetic datasets are
%generated according to the following models.

\paragraph{\bf Datasets I and II} The first two datasets are on the Grassmannian
$\G(6,2)$. In dataset I, 130 pairs of subspaces are drawn from the following non-intersecting submanifolds:
%The notation $\Span(\mb{v}_1, \mb{v}_2)$ stands
% for the subspace spanned by the two vectors $\mb{v}_1, \mb{v}_2$.
\begin{equation*}
  \begin{aligned}
    & \mb{x}_1=\Span\{(\cos(\theta), 0, \sin(\theta), 0, 0,
    0)+\frac{1}{40}\pmb{\epsilon}_{1\times 6}, (0, \cos(\theta), 0, \sin(\theta),
    0, 0)+\frac{1}{40}\pmb{\epsilon}_{1\times 6}\},\\
    & \mb{x}_2=\Span\{(\cos(\phi), 0, \sin(\phi), 0, 0.5,
    0)+\frac{1}{40}\pmb{\epsilon}_{1\times 6}, (0, \cos(\phi), 0, \sin(\phi), 0.5,
    0)+\frac{1}{40}\pmb{\epsilon}_{1\times 6}\},
  \end{aligned}
\end{equation*}
where $\theta, \phi$ are equidistantly drawn from $[-\pi/3, \pi/3]$ and the noise
vector $\pmb{\epsilon}_{1\times 6}$ comprises i.i.d.\ normal random
variables $\mathcal{N}(0,1)$.

In dataset II, 130 pairs of subspaces lie around two intersecting submanifolds
as follows:
\begin{equation*}
\begin{aligned}
& \mb{x}_1=\Span\{(\cos(\theta), 0, \sin(\theta), 0, 0,
0)+\frac{1}{40}\pmb{\epsilon}_{1\times 6}, (0, \cos(\theta), 0, \sin(\theta), 0,
0)+\frac{1}{40}\pmb{\epsilon}_{1\times 6}\},\\
& \mb{x}_2=\Span\{(\cos(\phi), 0, 0, 0, \sin(\phi),
0)+\frac{1}{40}\pmb{\epsilon}_{1\times 6}, (0, \cos(\phi), 0, 0, 0,
\sin(\phi))+\frac{1}{40}\pmb{\epsilon}_{1\times 6}\},
\end{aligned}
\end{equation*}
where $\theta, \phi$ are equidistantly drawn from $[-\pi/3, \pi/3]$ and the
noise vector $\pmb{\epsilon}_{1\times 6}$ comprises, again, i.i.d.\ normal
random variables $\mathcal{N}(0,1)$.

% In case III, 260 points are generated from the same model as in case II except
% that $\theta, \phi$ are uniformly drawn from $[0, \pi/2]$.

\paragraph{\bf Datasets III and IV}
The next two datasets
are contained in the manifold of $3\times 3$ symmetric PD matrices.

In dataset III, 130 pairs of matrices of two intersecting groups are generated from the model
\begin{equation}
\begin{aligned}
&\mb{A}_1=\left( \begin{array}{ccc}
4 & 4\cos(\theta+\pi/4) & 4\sin(\theta+\pi/4) \\
4\cos(\theta+\pi/4) & 4 & 0 \\
4\sin(\theta+\pi/4) & 0 & 4 \end{array} \right)+ \pmb{\epsilon}_{3\times 3}/40, \\
&\mb{A}_2=\left( \begin{array}{ccc}
4 & 0 & 4\cos(\theta-\pi/4) \\
0 & 4 & 4\sin(\theta-\pi/4) \\
4\cos(\theta-\pi/4) & 4\sin(\theta+\pi/4) & 4 \end{array} \right)+ \pmb{\epsilon}_{3\times 3}/40,
\end{aligned}
\end{equation}
where $\theta$ is
equidistantly drawn from $[0, \pi]$ and $\pmb{\epsilon}_{3\times 3}$ is a symmetric matrix whose entries are
i.i.d.~normal random variables with distribution $\mathcal{N}(0,1)$.

In dataset IV, 130 pairs of matrices of two non-intersecting groups are generated from the model
\begin{equation*}
\mb{A}_1=\left( \begin{array}{ccc}
10\alpha & 0 & 0 \\
0 & 10\alpha & 0 \\
0 & 0 & 10\alpha \end{array} \right)+ \pmb{\epsilon}_{3\times 3}/40, \quad
\mb{A}_2=\left( \begin{array}{ccc}
10\beta & 0 & 0 \\
0 & 10\beta^2 & 0 \\
0 & 0 & 10\beta^3 \end{array} \right)+ \pmb{\epsilon}_{3\times 3}/40,
\end{equation*}
where $\alpha, \beta$ are equidistantly drawn from $[0.5, 1]$ respectively and $\pmb{\epsilon}_{3\times 3}$
 is a symmetric matrix whose entries are
i.i.d.~normal random variables with distribution $\mathcal{N}(0,1)$.

\paragraph{\bf Datasets V and VI} Two datasets are constructed on the unit sphere $\Sb^2$ of
the 3-dimensional Euclidean space. Dataset V comprises of vectors lying around
the following two parallel arcs:
\begin{equation*}
\begin{aligned}
&\mb{x}_1=[\cos(\theta), \sin(\theta), 0]+\pmb{\epsilon}_{1\times 3},\\
&\mb{x}_2=[\sqrt{0.97}\cos(\phi), \sqrt{0.97}\sin(\phi),\sqrt{0.03}] +
\pmb{\epsilon}_{1\times 3},
\end{aligned}
\end{equation*}
where $\theta, \phi$ are equidistantly drawn from $[0, \pi/2]$. To ensure membership in $\Sb^2$, vectors
generated by $\mb{T}_1$ and $\mb{T}_2$ are normalized to unit length. On the
other hand, dataset VI considers the following two intersecting arcs:
\begin{equation*}
\begin{aligned}
&\mb{x}_1=[\cos(\theta+\pi/4), \sin(\theta+\pi/4), 0]+\frac{1}{40}\pmb{\epsilon}_{1\times
  3},\\
&\mb{x}_2=[0, \cos(\phi-\pi/4), \sin(\phi-\pi/4)]+\frac{1}{40}\pmb{\epsilon}_{1\times 3}.
\end{aligned}
\end{equation*}

\subsubsection{Numerical Results}

Each one of the six datasets is generated according to the postulated models
above, and the experiment is repeated $30$ times. Table~\ref{table:synthetic}
shows the average clustering rate for each method. GCT, SMC and SCR are all
based on the spectral clustering scheme. However, when a dataset has
low-dimensional structures, GCT's unique procedure of filtering neighboring
points ensures that it yields superior performance over the other methods.  This
is because both SMC and SCR are sensitive to the local scale $\sigma$, and
require each neighborhood not to contain points from different groups. This
becomes clear by the results on datasets I, IV, and V of non-intersecting
submanifolds. SMC only works well in dataset I, where most of the neighborhoods
$B(x_0,r)$ contain only points from the same cluster, while neighborhoods
$B(x_0,r)$ in datasets IV and V often contain points from different
ones. Embedded $K$-means generally requires that the intrinsic means of
different clusters are located far from each other. Its performance is not as
good as GCT when different groups have low-dimensional structures.

\begin{table*}\centering
\ra{1.3}
\begin{tabular}{@{}rrrrrrr@{}}\toprule
Methods & Set I & Set II & Set III & Set IV & Set V & Set VI\\ \midrule
GCT & \bf{1.00} $\pm$0.00 & \bf{0.98} $\pm$0.01  & \bf{0.98} $\pm$0.00 & \bf{0.95} $\pm$0.01 & \bf{0.98} $\pm$0.01 & \bf{0.96} $\pm$0.01 \\
SMC & 0.97 $\pm$0.04 & 0.66 $\pm$0.08  & 0.88 $\pm$0.03 & 0.80 $\pm$0.02 & 0.55 $\pm$0.06 & 0.69 $\pm$0.05 \\
SCR & 0.51 $\pm$0.00 & 0.66 $\pm$0.07 & 0.84 $\pm$0.00 & 0.80 $\pm$0.00 & 0.50 $\pm$0.00 & 0.53 $\pm$0.07  \\
EKM & 0.50 $\pm$0.00 & 0.50 $\pm$0.00 & 0.67 $\pm$0.00 & 0.50 $\pm$0.00 & 0.50 $\pm$0.00 & 0.67 $\pm$0.06 \\
\bottomrule
\end{tabular}
\caption{Average clustering rates on the six synthetic datasets of
    Section~\ref{subsec:synthetic}.}\label{table:synthetic}
\end{table*}

\subsection{Robustness to Noise and Running Time}\label{subsec:robust_runningtime}

Section~\ref{subsec:synthetic} illustrated GCT's superior performance over the
competing SMC, SCR, and EKM on a variety of manifolds. This section further
investigates GCT's robustness to noise and computational cost pertaining to running
time. In summary, GCT is shown to be far more robust than SMC in the presence of
noise, at the price of a small increase of running time.

\subsubsection{Robustness to Noise}
\label{sec:robust}
The proposed tangent filtering scheme enables GCT to successfully eliminate
neighboring points that originate from different groups. As such, it exhibits
robustness in the presence of noise and/or whenever different groups are close
or even intersecting. On the other hand, SMC appears to be sensitive to noise
due to its sole dependence on sparse weights. Figures~\ref{fig:robust_grass}
and~\ref{fig:robust_sph} demonstrate the performance of GCT, SMC, SCR, and EKM
on the Grassmannian and the sphere for various noise levels (standard deviations of Gaussian noise).

\begin{figure}[htb!]
  \centering
  \includegraphics[width=.70\linewidth]{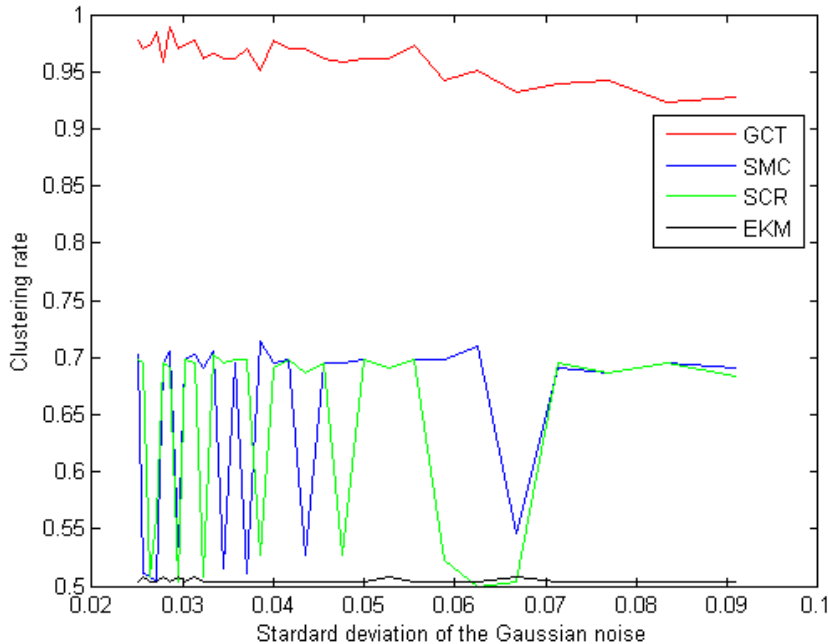}
  \caption{Performance of clustering methods on the Grassmannian for various
    noise levels. Datasets are generated according to the model of dataset II,
    but with an increasing standard deviation of the noise.}
\label{fig:robust_grass}
\end{figure}

The datasets in Figure~\ref{fig:robust_grass} are generated on the Grassmannian
according to the model of dataset II in Section~\ref{subsec:synthetic} but with different noise levels (in Section~\ref{subsec:synthetic}
the noise level was $0.025$). Both SMC
and SCR appear to be volatile over different datasets, with their best
performance never exceeding $0.75$ clustering rate. It is worth noticing that EKM shows poor
clustering accuracy. On the contrary, GCT exhibits remarkable robustness to noise,
achieving clustering rates above $0.9$ even when the standard deviation of the
noise approaches $0.1$.

\begin{figure}[htb!]
 \centering
 \includegraphics[width=.70\linewidth]{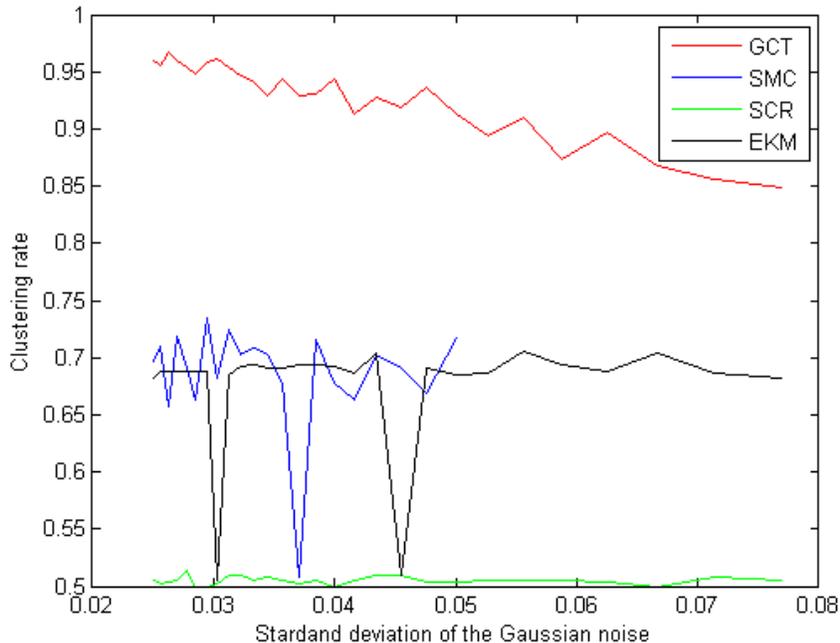}
  \caption{Performance of clustering methods on the sphere for various
    noise levels. Datasets are generated according to the
    model of dataset VI, but with an increasing standard deviation of the noise.
}
\label{fig:robust_sph}
\end{figure}

GCT's robustness to noise is also demonstrated in Figure~\ref{fig:robust_sph}, where
datasets are generated on the unit sphere according to the model of the dataset
VI, but with different noise levels. SMC appears to be volatile also in this
setting; it collapses when the standard deviation of noise exceeds $0.05$, since
its affinity matrix precludes spectral clustering from identifying eigenvalues
with sufficient accuracy (see further explanation on the collapse of SMC at the end of Section~\ref{sec:clusterings_param}).

\subsubsection{Running time}\label{sec:running.time}

This section demonstrates that GCT outperforms SMC at the price of a small
increase in computational complexity. Similarly to any other manifold clustering
algorithm, computations have to be performed per local neighborhood, where local
linear structures are leveraged to increase clustering accuracy. The overall
complexity scales quadratically w.r.t.\ the number of data-points due to the
last step of Algorithm~\ref{alg:experiment}, which amounts to spectral
clustering of the $N\times N$ affinity matrix $\mb{W}$. Both the optimization
task of \eqref{equ:l1minimization} and the computation of a few principal
eigenvectors of the covariance matrix $\mb{C}_{x_i}$ in
Algorithm~\ref{alg:experiment} do not contribute much to the complexity since
operations are performed on a small number of points in the neighborhood
$J(x_i,r)$. The computational complexity of GCT is detailed in
Appendix~\ref{sec:complexity.GCT}. It is also noteworthy that GCT can be fully
parallelized since computations per neighborhood are independent. Nevertheless,
such a route is not followed in this section.

Compared with SMC, GCT has one additional component: identifying tangent
spaces through local covariance matrices---a task that entails local calculation of a
few principal eigenvectors. Nevertheless, it is shown in Appendix~\ref{sec:algebraic_trick}
that for $k$ neighbors it can be calculated with $\OO(D+k^3)$ operations.

The ratios of running times between GCT and SMC for all three types of
manifolds are illustrated in Table~\ref{table:runningtime}. It can be readily
verified that the extra step of identifying tangent spaces in GCT increases
running time by less than $11\%$ of the one for SMC.
\begin{table}\centering
\ra{1.3}
\begin{tabular}{@{}cccc@{}}\toprule
Running-time ratio & $\G(6,2)$ & $PD_{3\times 3}$ & $\Sb^2$ \\ \midrule
GCT/SMC & 1.06 & 1.05 & 1.11 \\
\bottomrule
\end{tabular}
\caption{Ratio of running times of GCT and SMC for instances of the synthetic datasets I, IV and VI}\label{table:runningtime}
\end{table}

Ratios of running times were also investigated for increasing ambient dimensions of the sphere.
More precisely, dataset VI of Section~\ref{subsec:synthetic}, which lies in $S^2$, was embedded via a random orthonormal matrix
into the unit sphere $\Sb^D$, where $D$ ranged from $100$ to $3,000$. Figure~\ref{fig:runtime_relative} shows
the ratios of the running time of GCT over that of SMC as a function of $D$. We
observe that the extra cost of computing the eigendecomposition in GCT is
mostly less than $20\%$ of SMC, and never exceeds $30\%$, even when the ambient
dimension is as large as $3,000$.

\begin{figure}[htb!]
  \centering
  \includegraphics[width=.60\linewidth]{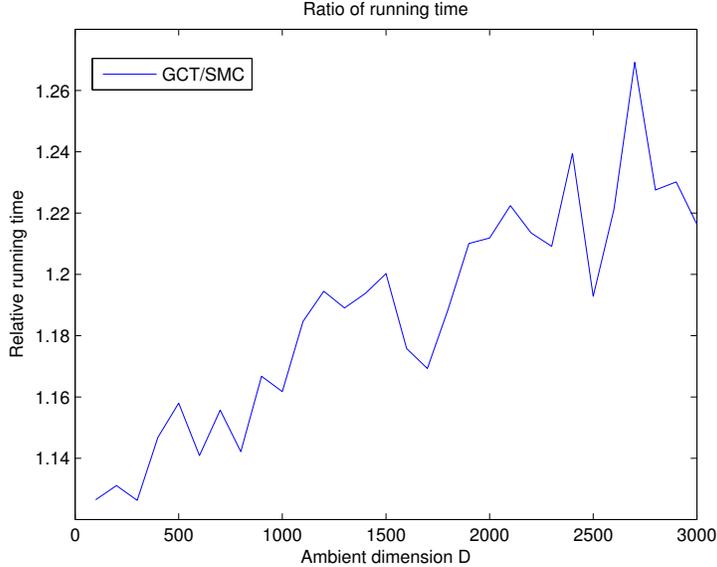}
  \caption{Relative running times of GCT w.r.t.\ SMC as the ambient dimension increases.
   With dimensions $D$ ranging from $100$ to $3,000$, dataset VI of Section~\ref{subsec:synthetic} was embedded via a random orthonormal matrix
into the unit sphere $\Sb^D$.
  Dataset VI of Section~\ref{subsec:synthetic} matrix with two random }
\label{fig:runtime_relative}
\end{figure}

\subsection{Synthetic Brain Fibers Segmentation}

\citet{6619442} cast the problem of segmenting diffusion magnetic resonance
imaging (DMRI) data of different fiber tracts as a clustering problem on
$\Sb^D$. The crux of the methodology lies on the transformation of diffusion
images, associated with different views of the same object, into orientation
distribution functions (ODFs), which are nothing but probability density
functions on $\Sb^2$. The discretized ODF (dODF) is a probability mass function
(pmf) $\mathbf{f}: (\Sb^2)^{D+1}\rightarrow \R_{+}^{D+1}:
(\mb{s}_1,\ldots,\mb{s}_{D+1}) \mapsto
\mathbf{f}(\mb{s}_1,\ldots,\mb{s}_{D+1}):= [f_1(\mb{s}_1), \ldots,
f_{D+1}(\mb{s}_{D+1})]^T$, with $\sum_{i=1}^{D+1} f_i(\mb{s}_i)=1$, that
describes the water diffusion pattern at a corresponding location of the
object's image according to the viewing directions
$\{\mb{s}_i\}_{i=1}^{D+1}$. Given $\{\mb{s}_i\}_{i=1}^{D+1}$ and a fixed
location, the \textit{square-root} (SR)dODF is the vector
$\sqrt{\mathbf{f}}(\mb{s}_1,\ldots,\mb{s}_{D+1}) :=
[\sqrt{f_1(\mb{s}_1)},\ldots,\sqrt{f_{D+1}(\mb{s}_{D+1})}]^T$, which lies on the
sphere $\Sb^D$ since $\textbf{f}$ is a pmf. In this way, pixels of diffusion
images of the same object at a given location are mapped into an element of
$\Sb^D$.  \citet{6619442} assume that each fiber tract is mapped into a
submanifold of $\Sb^D$ and thus try to identify different fiber tracts by
multi-manifold modeling on $\Sb^D$.

As suggested in~\citet{6619442}, to differentiate pixels with similar diffusion
patterns but located far from each other in an image, one has to incorporate
pixel spatial information in the segmentation algorithm. Therefore, for GCT, SMC
and SCR, the similarity entry $\mb{W}_{ij}$ of two pixels
$\mathbf{x}_i,\mathbf{x}_j\in \R^2$ is modified as
\[
\mb{W}^{\text{new}}_{ij}=\mb{W}_{ij}\cdot e^{-\|\mathbf{x}_i-\mathbf{x}_j\|_2^2/\sigma},
\]
where $\mb{W}$ is the similarity matrix before modification (e.g., for GCT, it
is described in Algorithm~2), $\sigma=0.1$ and $\|\mathbf{x}_i-\mathbf{x}_j\|_2$
is the Euclidean distance between two pixels. For EKM, where no spectral
clustering is employed, the dODF is simply augmented with the
spatial coordinates of $\mathbf{x}_i$ and $\mathbf{x}_j$.

\begin{figure}[htb!]
  \centering \subfloat[\footnotesize Randomly sampled 6 points from the colored
  regions of the \text{$[0,1]\times[0,1]$} domain.]
  {\label{fig:points} \includegraphics[width=.45\linewidth]{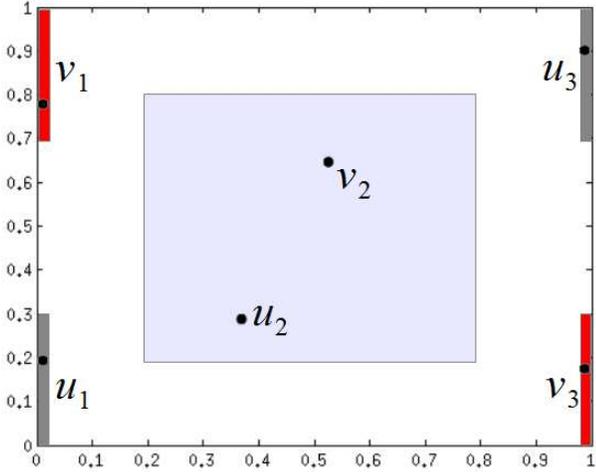}}
  \hfill
  \subfloat[\footnotesize A configuration of two intersecting fibers generated according to points in Figure~\ref{fig:points}]
  %\hspace{0.5em}
  {\label{fig:fiber} \includegraphics[width=.45\linewidth]{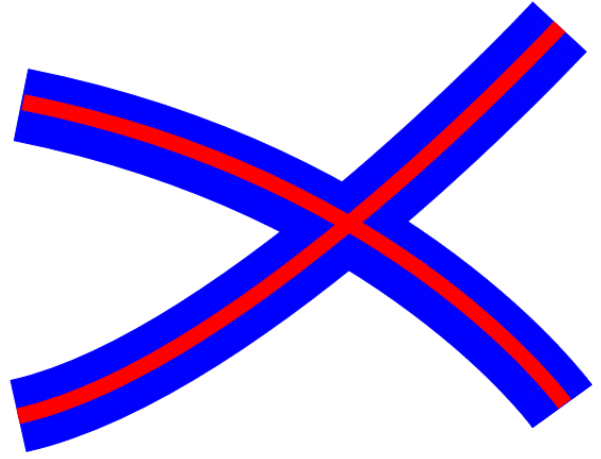}}
  \caption{Demonstration of fiber generation. Two fibers are generated in Figure~\ref{fig:fiber} by fitting two
    cubic splines to $\{\mb{u}_i\}_{i=1}^3$ and $\{\mb{v}_i\}_{i=1}^3$ in
    Figure~\ref{fig:points}, respectively.
% Figure~\ref{fig:odf} plots the ODFs
 %   for pixels on each fiber. ODFs are generated according to the software code
 %   provided by~\citet{fiber_web}
}
\label{fig:dataset}
\end{figure}

Following~\citet{6619442}, we consider here
the problem of segmenting or clustering two 2D synthetic fiber tracts in the
$[0, 1]\times[0, 1]$ domain. To generate the fibers, six
points $\mb{u}_1,\mb{u}_2,\mb{u}_3,\mb{v}_1,\mb{v}_2,\mb{v}_3$ are randomly
chosen in the colored region of Figure~\ref{fig:points}. Two cubic splines
passing through $\{\mb{u}_1,\mb{u}_2,\mb{u}_3\}$ and
$\{\mb{v}_1,\mb{v}_2,\mb{v}_3\}$, respectively, are set to be the center of the
fibers (cf., red curves in Figure~\ref{fig:fiber}). Fibers are defined as the
curved bands around the splines with bandwidth $0.12$ (cf., blue region
in Figure~\ref{fig:fiber}).

\begin{table*}\centering
\ra{1.3}
\begin{tabular}{@{}rrrrr@{}}\toprule
Methods & SNR=40 & SNR=30 & SNR=20 & SNR=10\\ \midrule

GCT & \bf{0.80} $\pm$ 0.12 & \bf{0.82} $\pm$ 0.12 & \bf{0.78} $\pm$ 0.14 & \bf{0.80} $\pm$ 0.13\\
SMC & 0.73 $\pm$ 0.14 & 0.73 $\pm$ 0.13 & 0.70 $\pm$ 0.13 & 0.67 $\pm$ 0.13\\
SCR & 0.66 $\pm$ 0.11 & 0.66 $\pm$ 0.11 & 0.68 $\pm$ 0.11 & 0.66 $\pm$ 0.11\\
EKM & 0.59 $\pm$ 0.08 & 0.58 $\pm$ 0.08 & 0.61 $\pm$ 0.08 & 0.59 $\pm$ 0.08\\
\bottomrule
\end{tabular}
\caption{Mean $\pm$ standard deviation of accuracy rates for $100$
  experiments on clustering synthetic brain fibers.}\label{table:fiber}
\end{table*}

\begin{figure*}[htb!]
\centering
\includegraphics[width=0.8\textwidth]{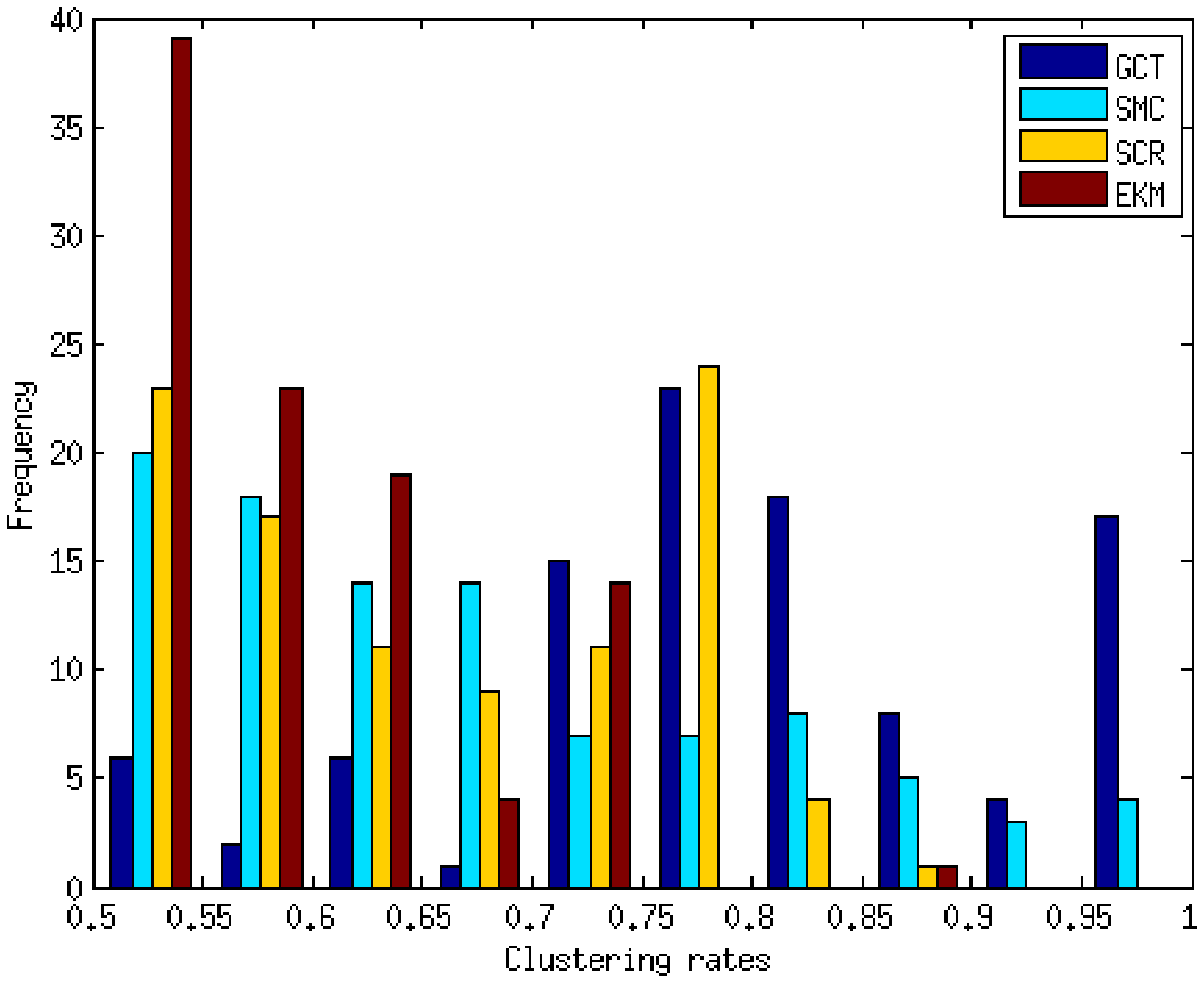}
\caption{Histogram of clustering rates for the noise level $\text{SNR}=10$ over
  a total number of $100$ experiments. Each bar shows the number of experiments
  whose rates fall within one of the ten intervals of length 0.05 in the
  partition. For example, since the tallest bar within the $[0.95,1]$ range is
  the blue one, GCT is the most likely method to achieve almost accurate
  clustering. On the contrary, the brown bar is the tallest one between the
  range of $0.5$ and $0.55$, meaning that a clustering rate within $[0.5,0.55]$
  is the most likely one to be achieved for EKM over $100$ experiments.}
\label{fig:visual_pdf}
\end{figure*}

Given a pair of such fibers, the next step is to map each pixel (e.g., both red
and blue ones in Figure~\ref{fig:fiber}) to a point (SRdODF) in
$\Sb^D$. To this end, the software code provided by \citet{fiber_web} is used to
generate SRdODFs on $\Sb^{100}$, where diffusion images $\{S_n\}_{n=1}^G$ at $G=70$
gradient directions, with baseline image $S_0=100$ and $b=4,000\text{s/mm}^2$,
are considered. The dimensionality of the generated SRdODFs corresponds to $100$
directions. Moreover, Gaussian noise $\mathcal{N}(0,\sigma^2)$ was added in the
ODF-generation mechanism, resulting in a signal-to-noise ratio $\text{SNR} =
S_0/\sigma$ (more details on the construction can be found in
\citet{6619442}). Typical noise levels for real-data brain images are
considered: $\text{SNR}=10, 20, 30$, and $40$ (i.e., $\sigma=10, 5, 10/3,
2.5$).

Once SRdODFs are formed, clustering is carried out on the Riemannian manifold
$\mathbb{S}^D$. This in turn provides a segmentation of pixels according to
different fiber tracts. A total number of $100$ pairs of synthetic brain fibers
are randomly generated, and clustering is performed for each
pair. Table~\ref{table:fiber} reports the mean $\pm$ standard deviation of the
clustering accuracy rates. Results clearly suggest that GCT outperforms the
other three clustering methods. For the case of $\text{SNR}=10$,
Figure~\ref{fig:visual_pdf} plots sample distributions of accuracy rates and shows
that GCT demonstrates the highest probability of achieving almost accurate
clustering among competing schemes.

\subsection{Experiments with Real Data}

In this section, GCT performance is assessed on real
  datasets. Scenarios where data within each cluster have submanifold structures
  are demonstrated.

%It turns out that GCT can effectively handle the issues of
%  non-convex shape of each cluster and intersections across
%  different clusters.
% Texture clustering involves clustering on symmetric PD
%  matrices. Readers are referred to the supplementary material for other experiments
%including action dynamics clustering which involves clustering on Grassmann manifolds.

\subsubsection{Stylized Application: Texture Clustering}\label{exp:texture}
We cluster local covariance matrices obtained from various transformations of images
 of the Brodatz database~\citep{brodatz_web} where the goal is to be able to distinguish
 between the different images independently of the transformation.
%is examined under different lighting and shape deformations.
%Each texture pattern is represented by a region covariance
%matrix, obtained via Gabor filters, and clustering is performed on the
%manifold of symmetric PD matrices.

%Firstly, a short description of the dataset and the way to
%  construct covariance matrices is in order.

The Brodatz database contains $112$
  images of $640 \times 640$ pixels with different textures (e.g., brick wall,
  beach sand, grass) captured under uniform lighting and in frontview
  position. We apply three simple deformations to these images, which mimic real settings:
 different lighting conditions, stretching (obtained by shearing) and different viewpoints (obtained by affine transformation). Figure~\ref{fig:Brodatz}
 shows sample images in the Brodatz database and their deformations.
%For type I,
%  lighting is linearly dimmed from right to left; for type II, each image is
%  horizontally shifted by different angles; for type III, each image is affinely
%  transformed. These transformations often result in large variance within each
%  texture cluster.
\begin{figure*}[htb!]
\captionsetup[subfigure]{labelformat=empty}
\centering
\subfloat[]
{\includegraphics[width=.13\linewidth]{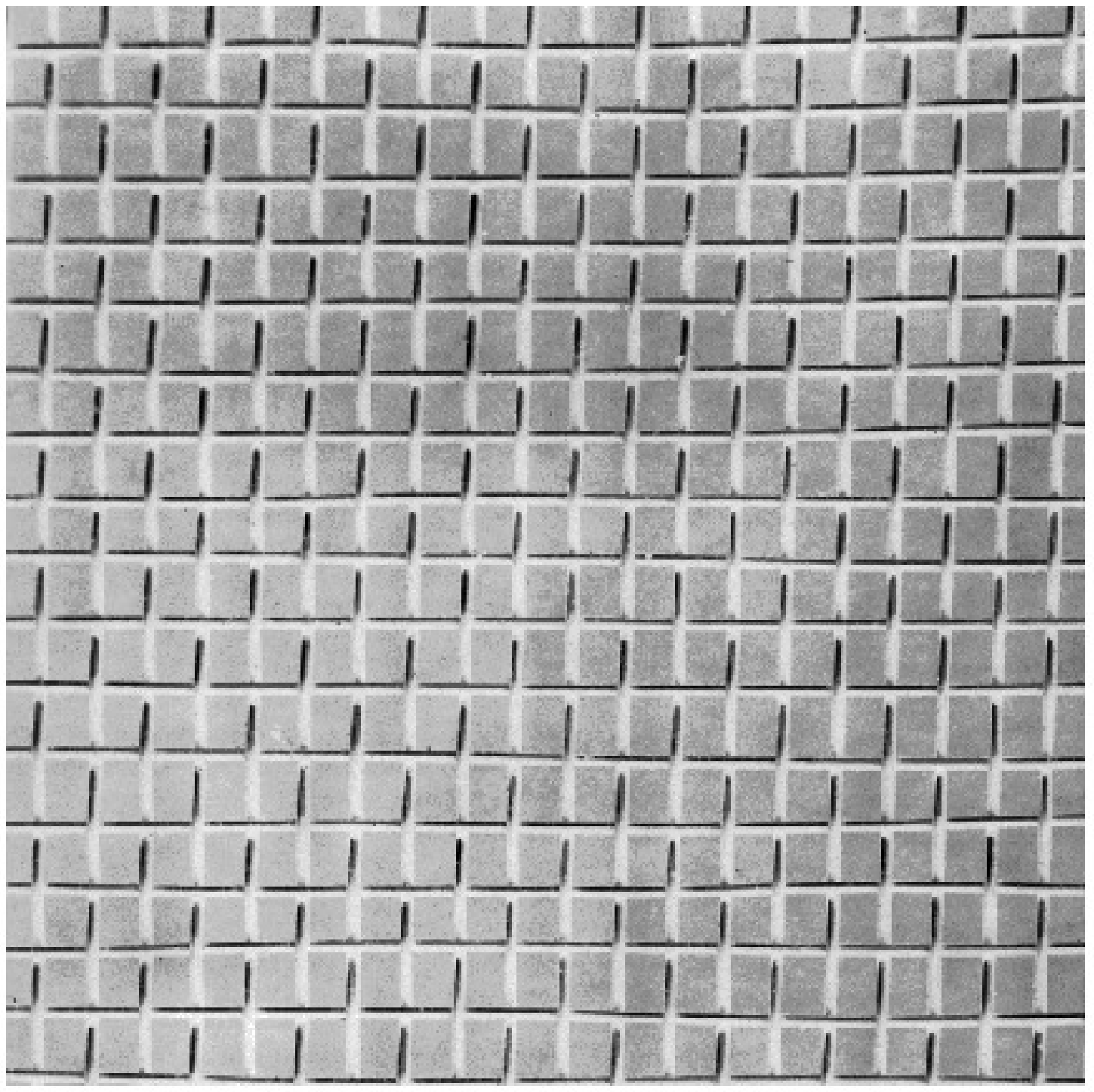}} \hspace{-0.6em}
\subfloat[]
{\includegraphics[width=.13\linewidth]{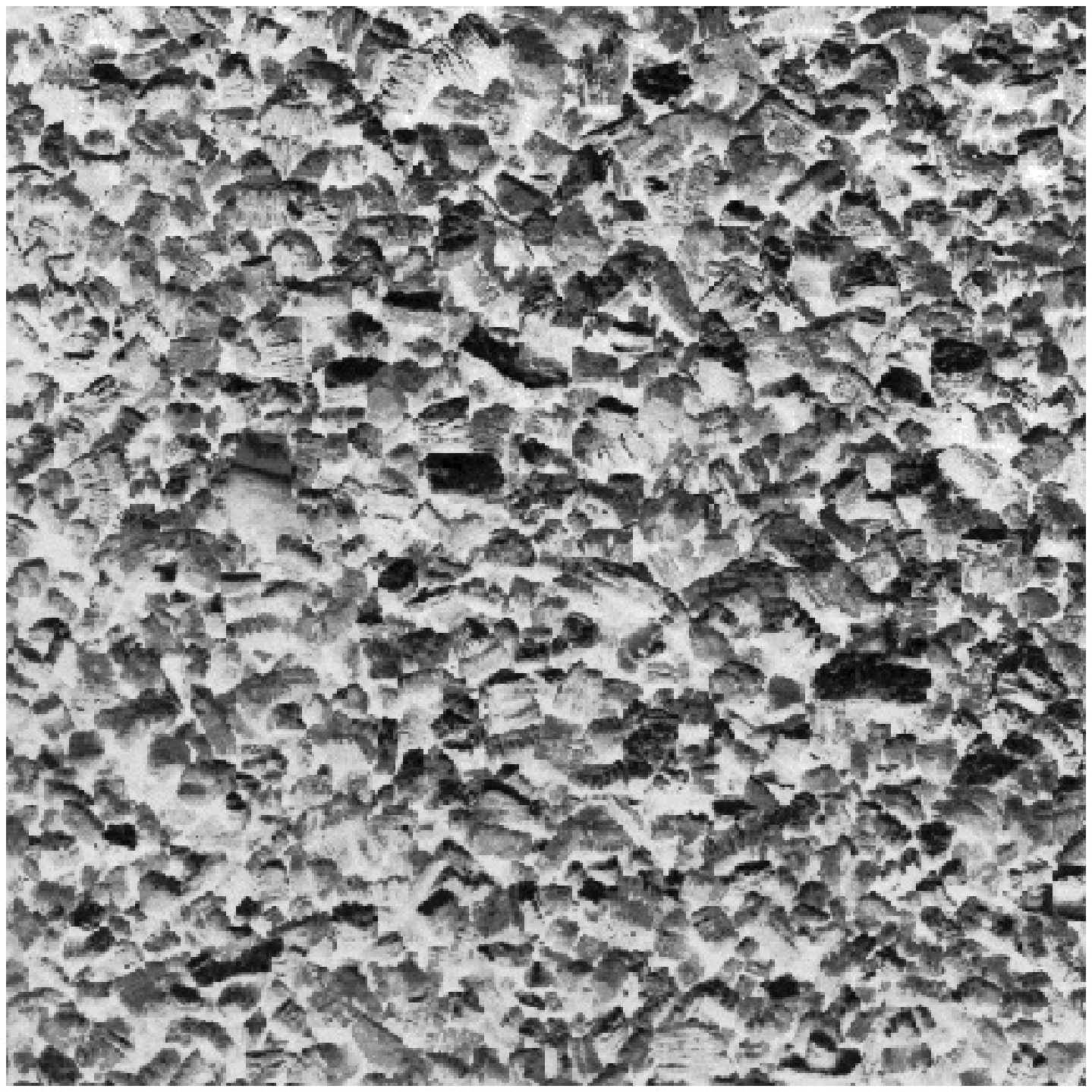}} \hspace{-0.6em}
\subfloat[]
{\includegraphics[width=.13\linewidth]{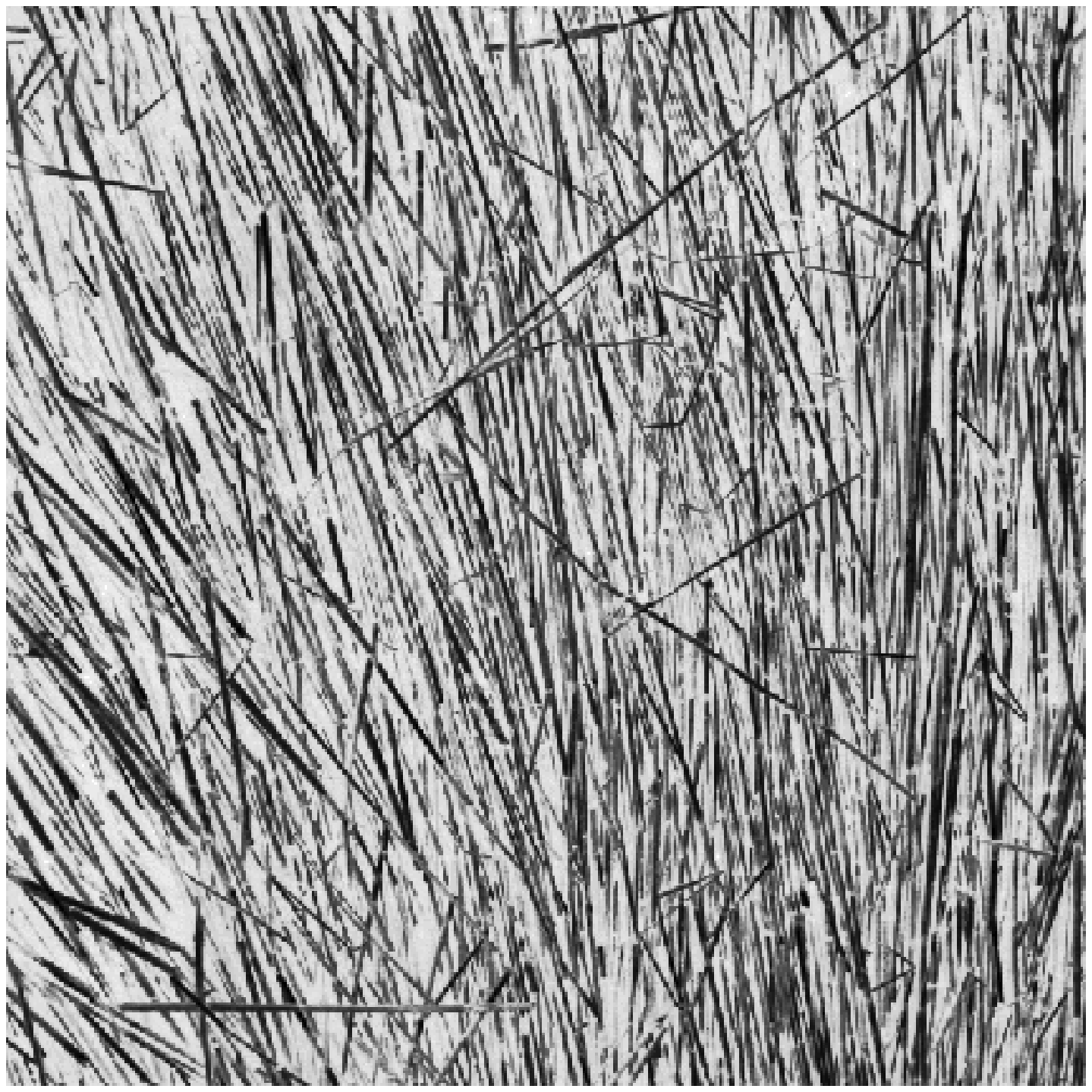}} \hspace{-0.6em}
\subfloat[]
{\includegraphics[width=.13\linewidth]{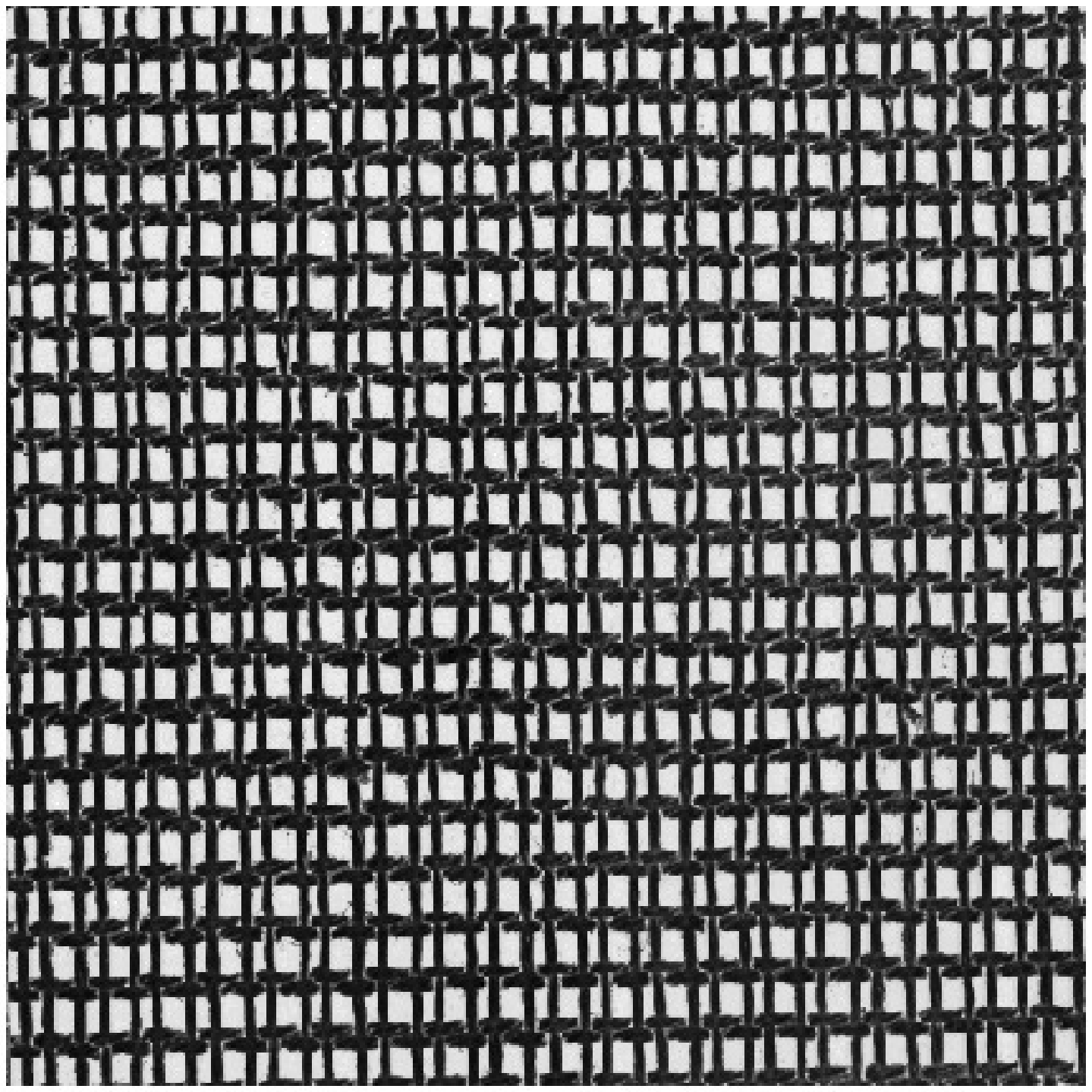}} \hspace{-0.6em}
\subfloat[]
{\includegraphics[width=.13\linewidth]{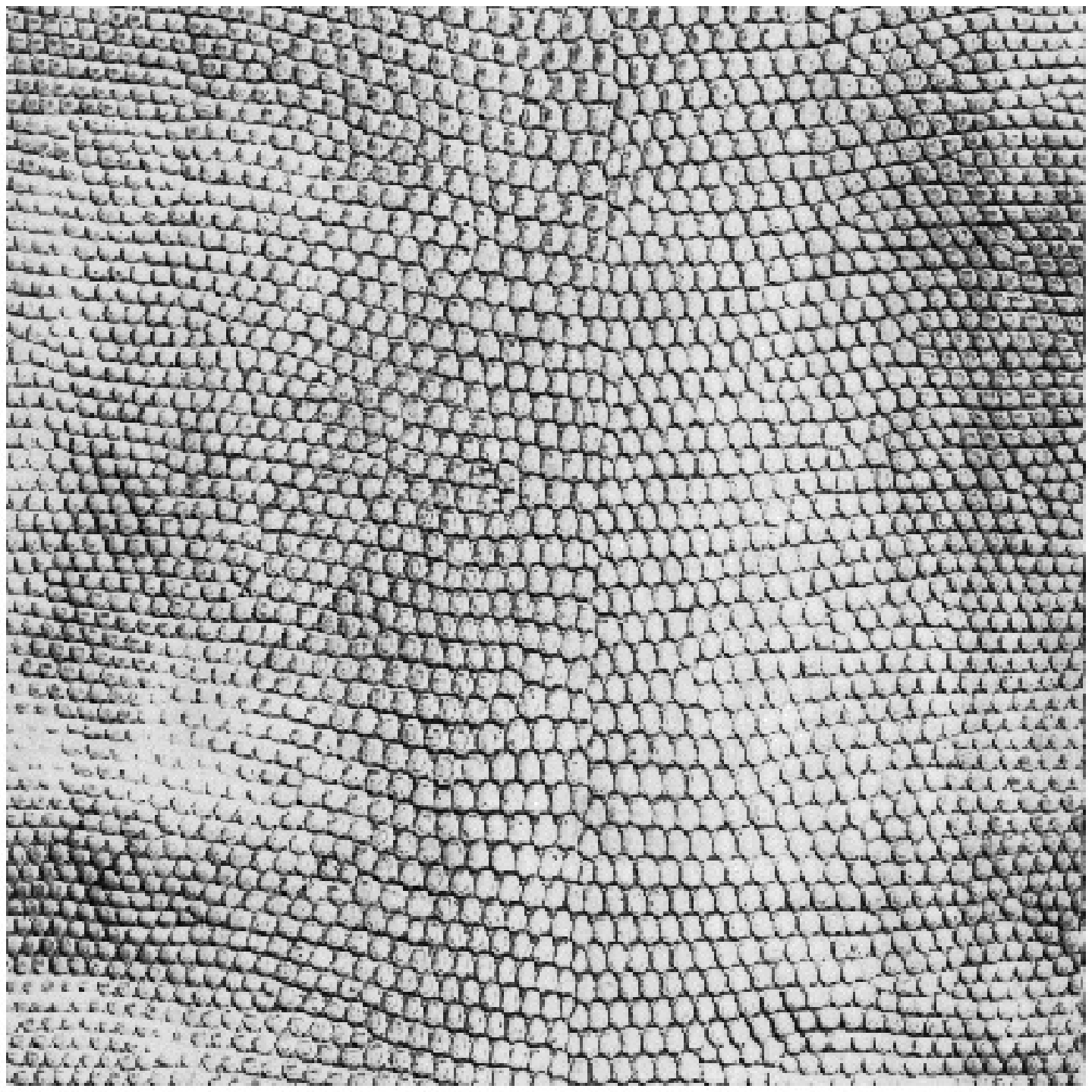}} \hspace{-0.6em}
\subfloat[]
{\includegraphics[width=.13\linewidth]{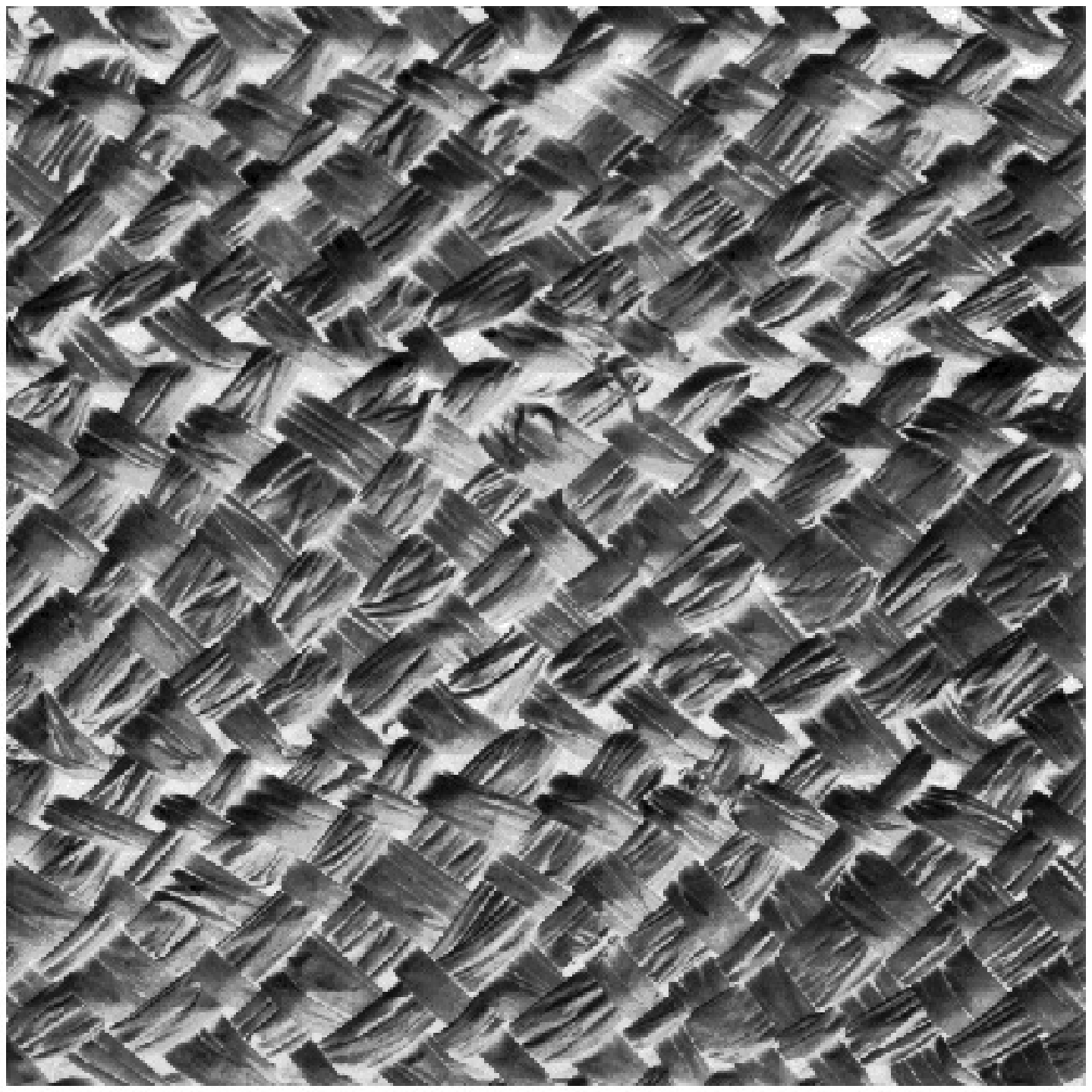}}\\ \vspace{-2.5em}
\subfloat[]
{\includegraphics[width=.13\linewidth]{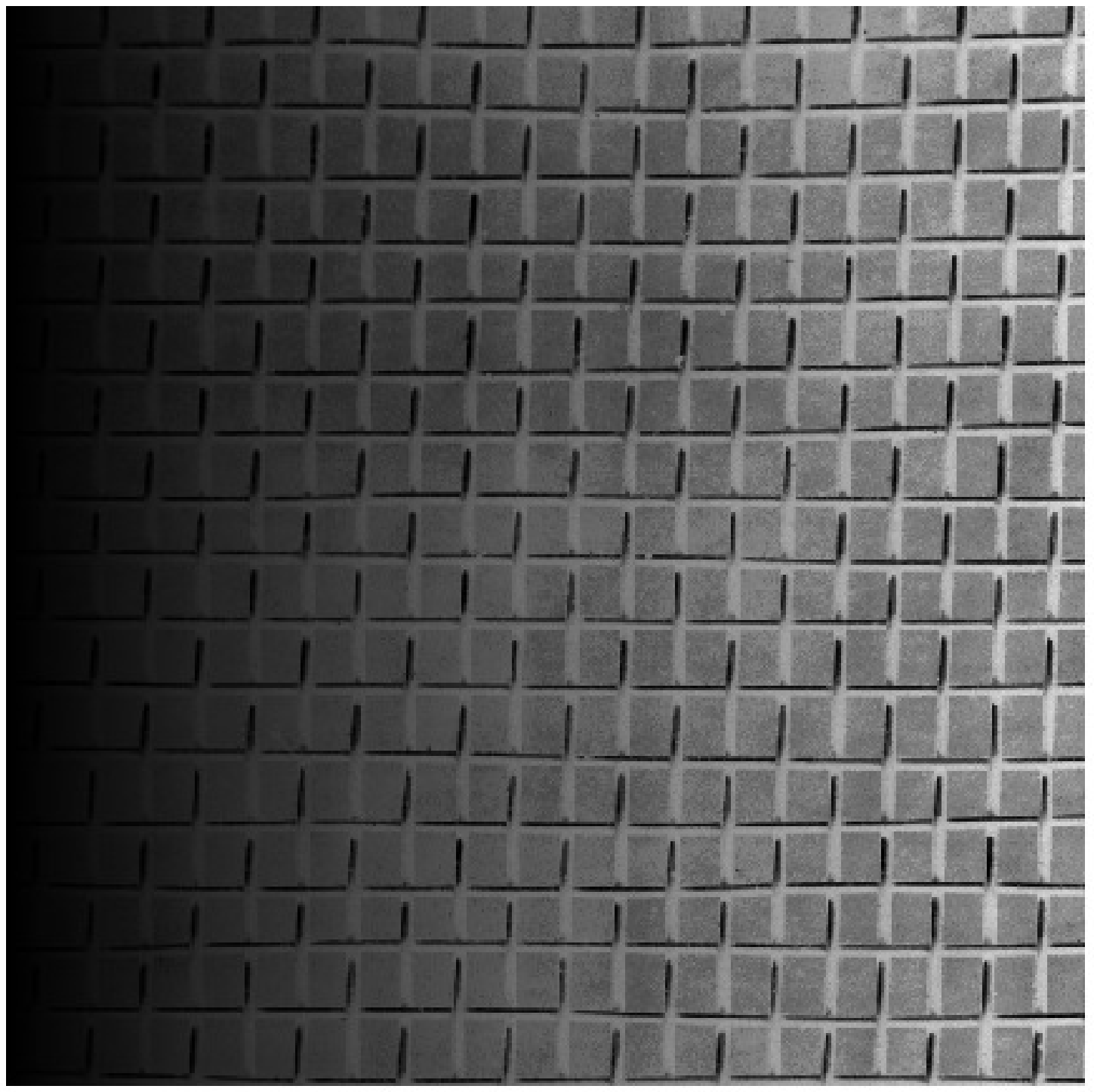}} \hspace{-0.6em}
\subfloat[]
{\includegraphics[width=.13\linewidth]{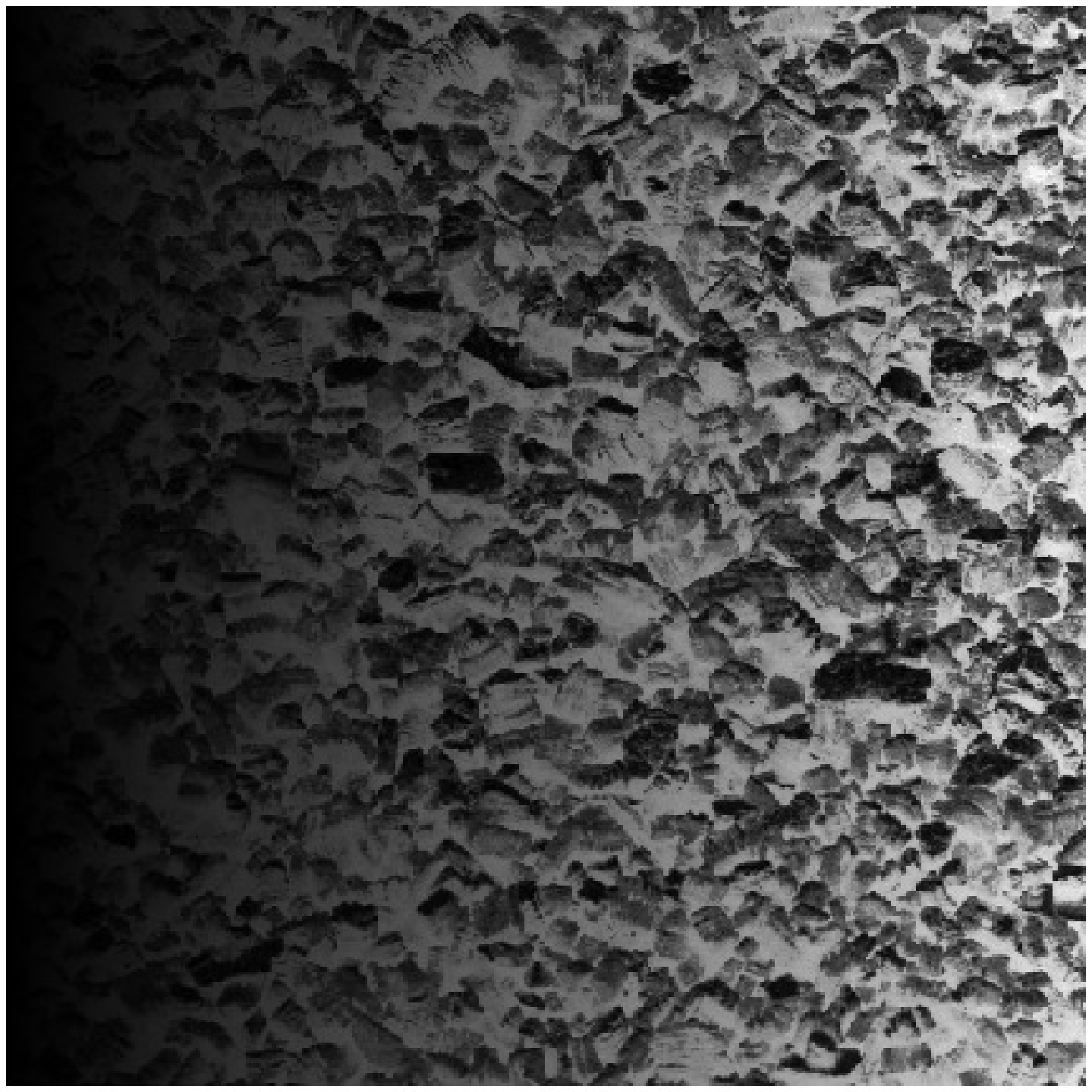}} \hspace{-0.6em}
\subfloat[]
{\includegraphics[width=.13\linewidth]{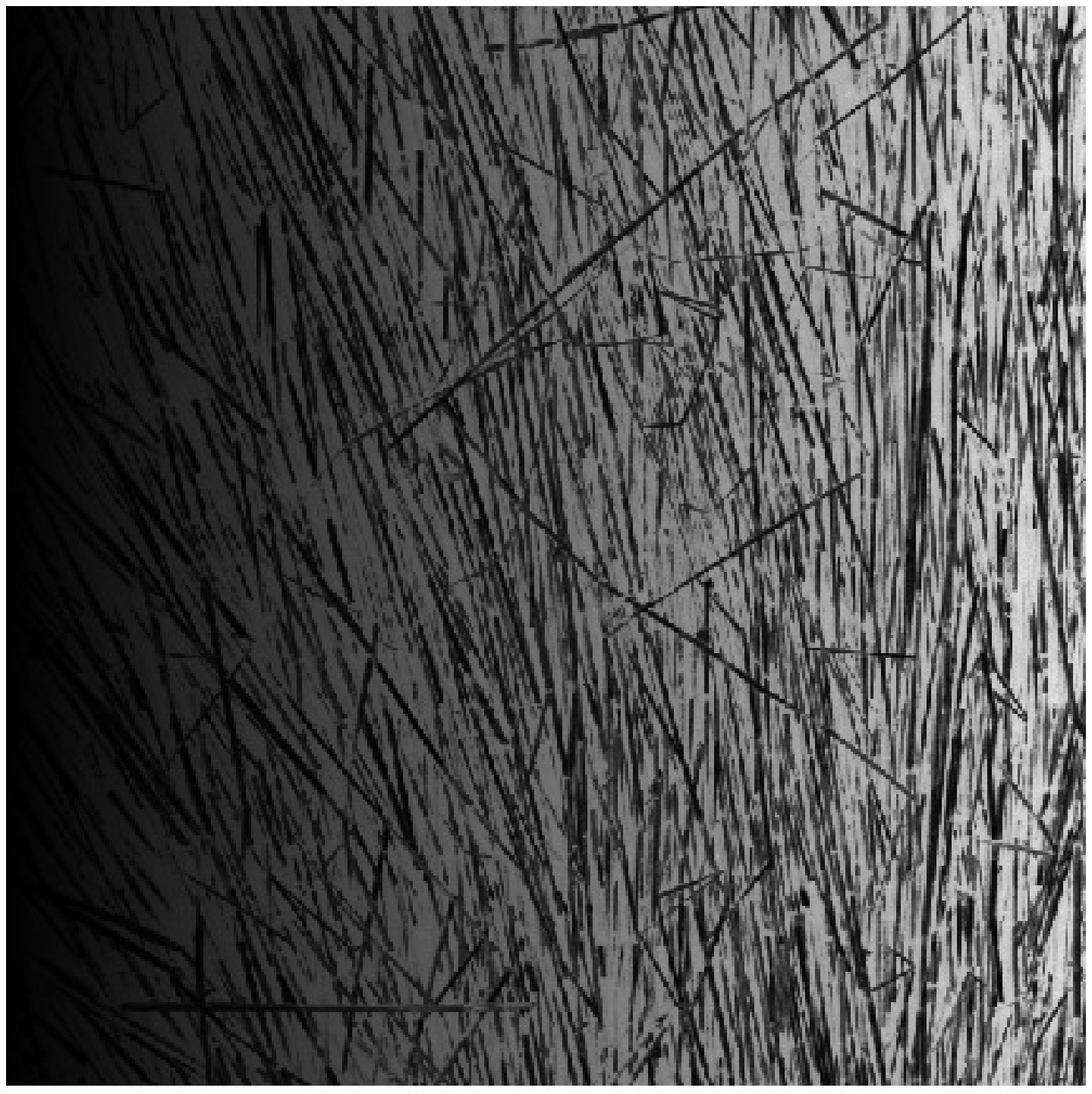}} \hspace{-0.6em}
\subfloat[]
{\includegraphics[width=.13\linewidth]{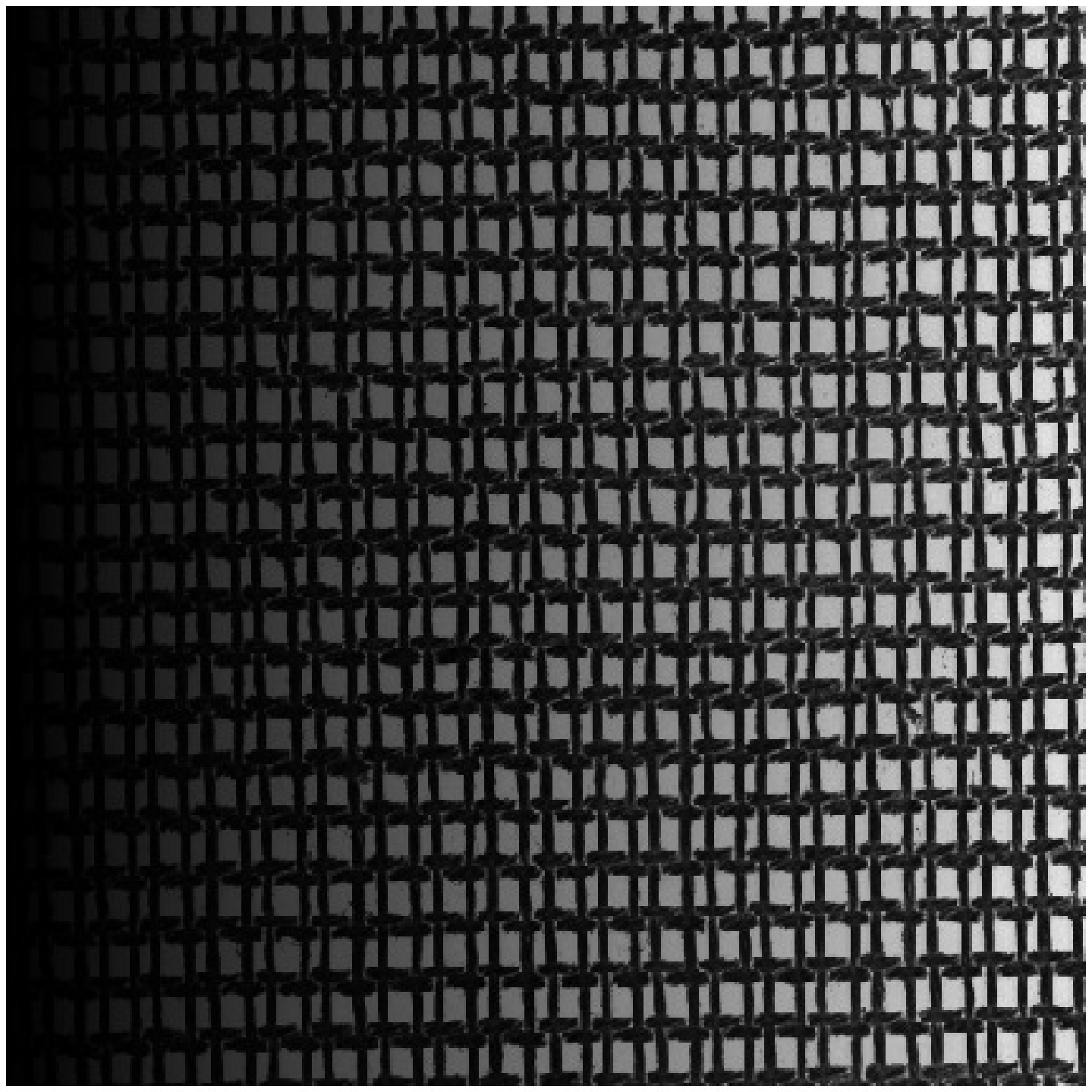}} \hspace{-0.6em}
\subfloat[]
{\includegraphics[width=.13\linewidth]{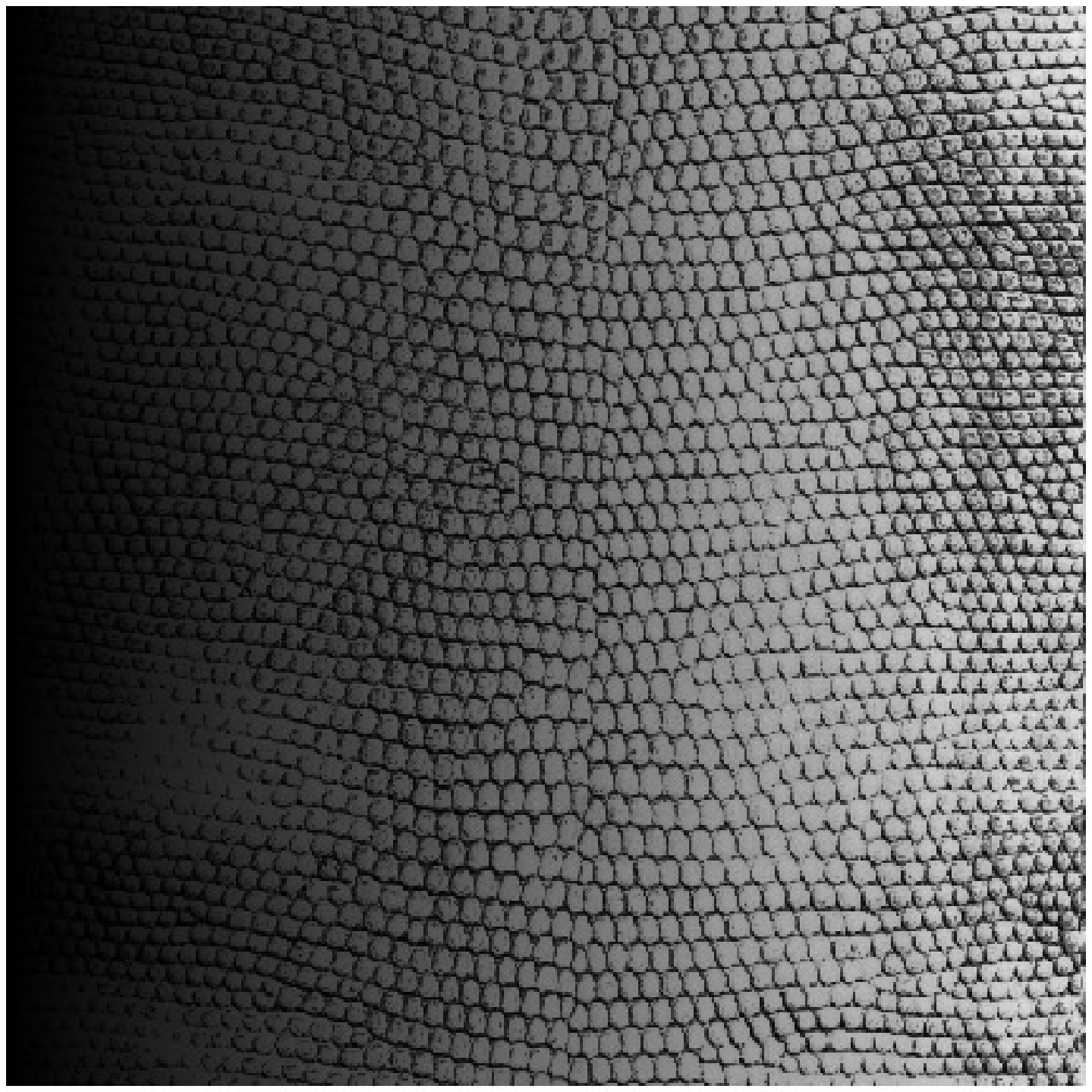}} \hspace{-0.6em}
\subfloat[]
{\includegraphics[width=.13\linewidth]{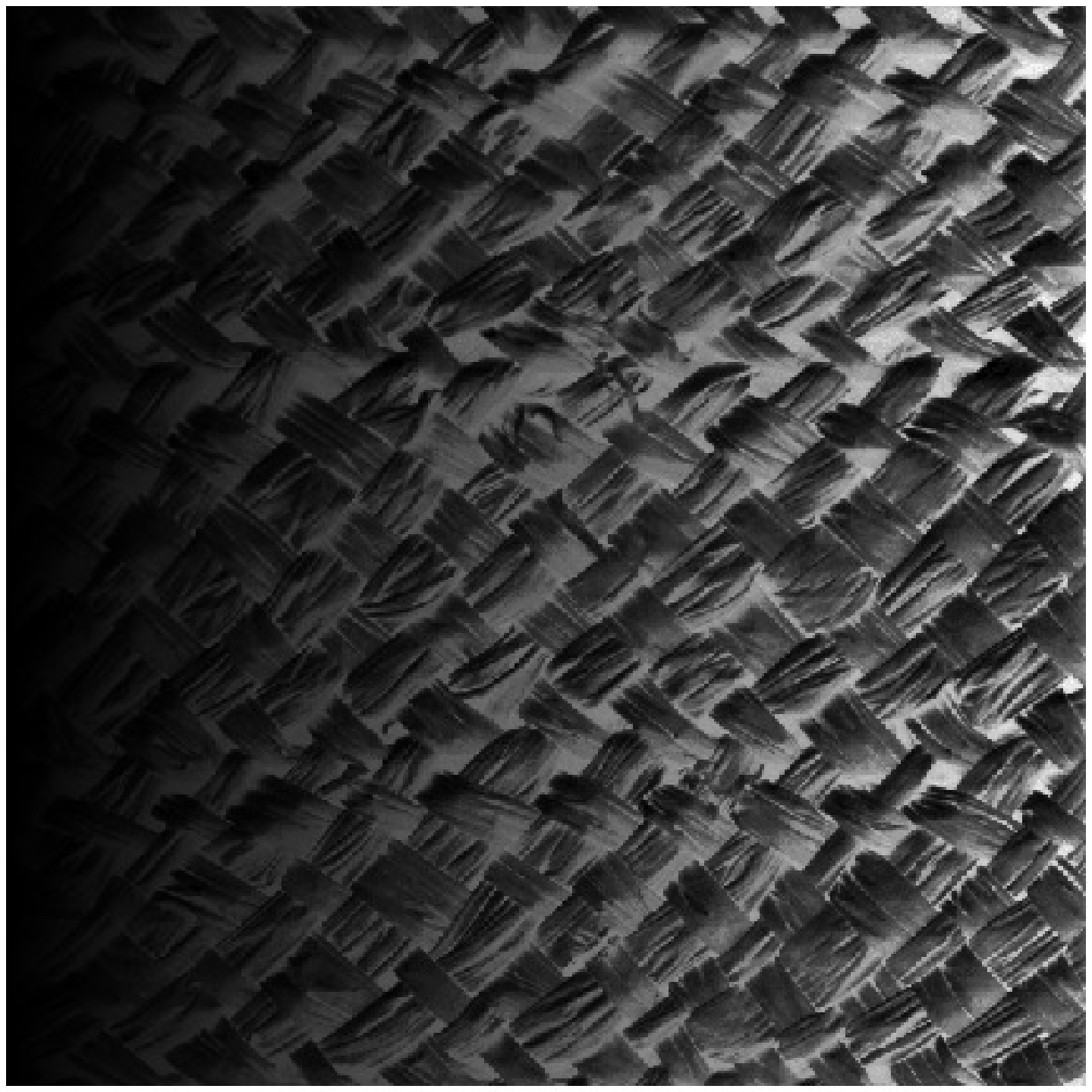}}\\ \vspace{-2.5em}
\subfloat[]
{\includegraphics[width=.13\linewidth]{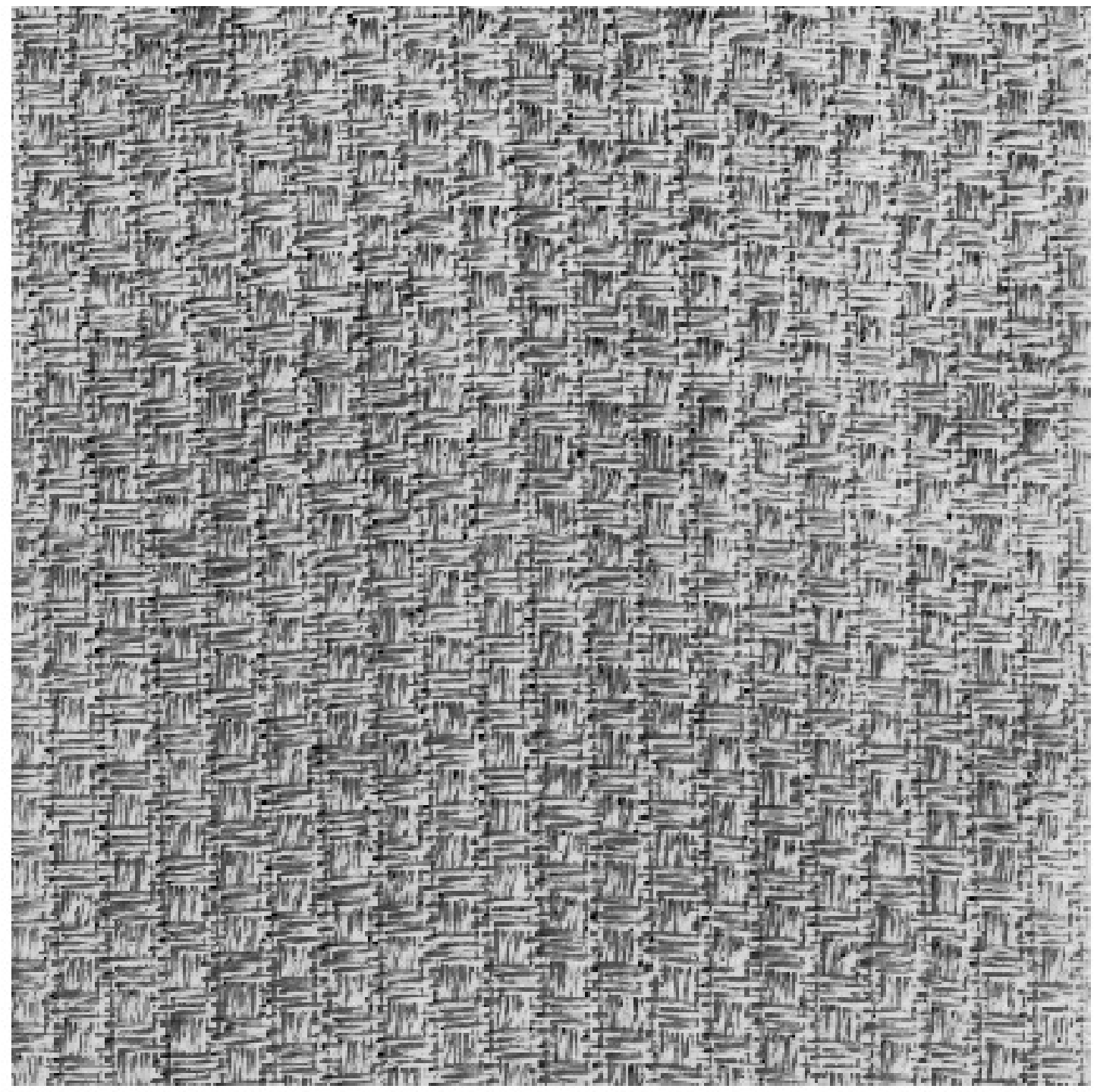}} \hspace{-0.6em}
\subfloat[]
{\includegraphics[width=.13\linewidth]{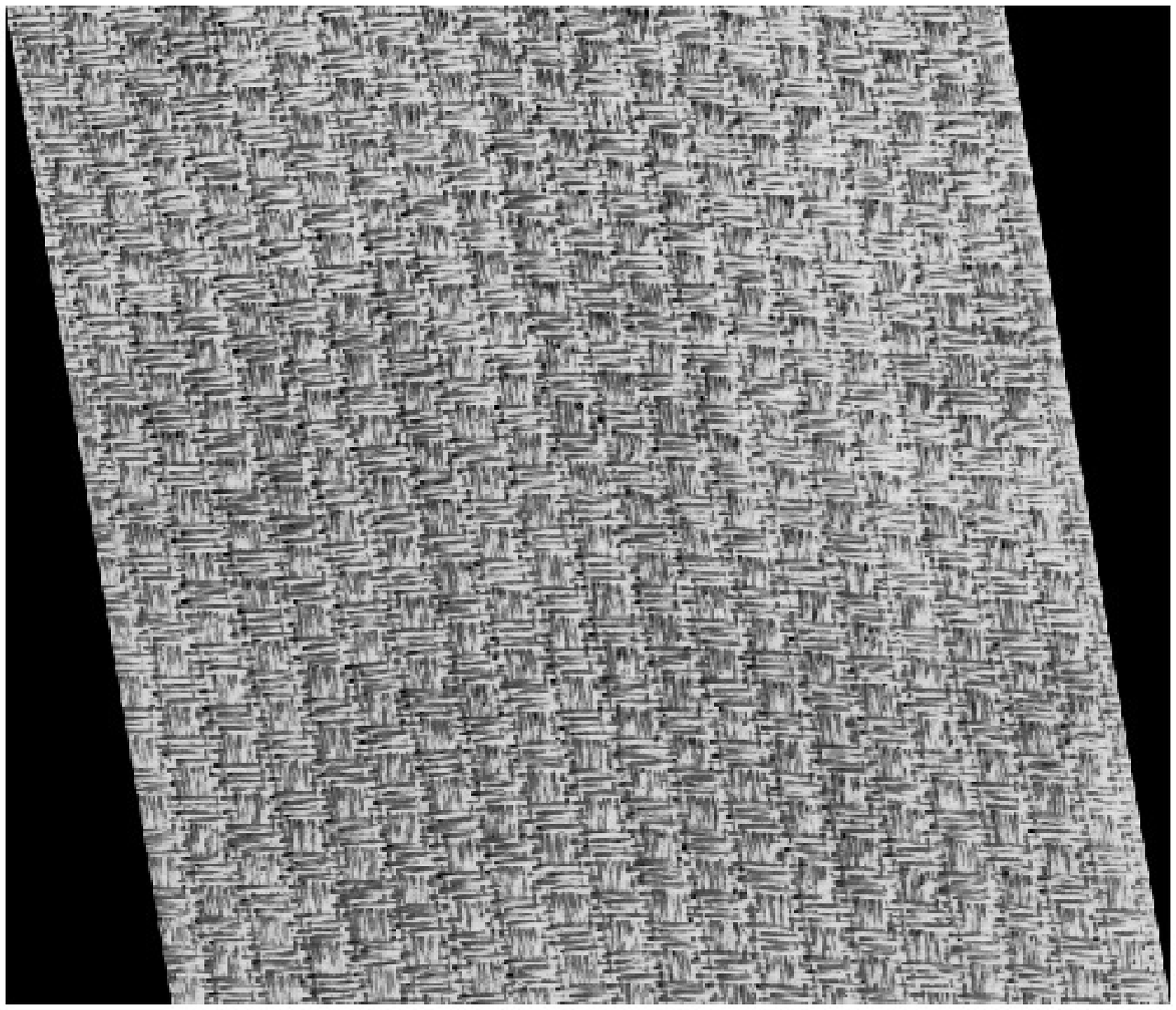}} \hspace{-0.6em}
\subfloat[]
{\includegraphics[width=.13\linewidth]{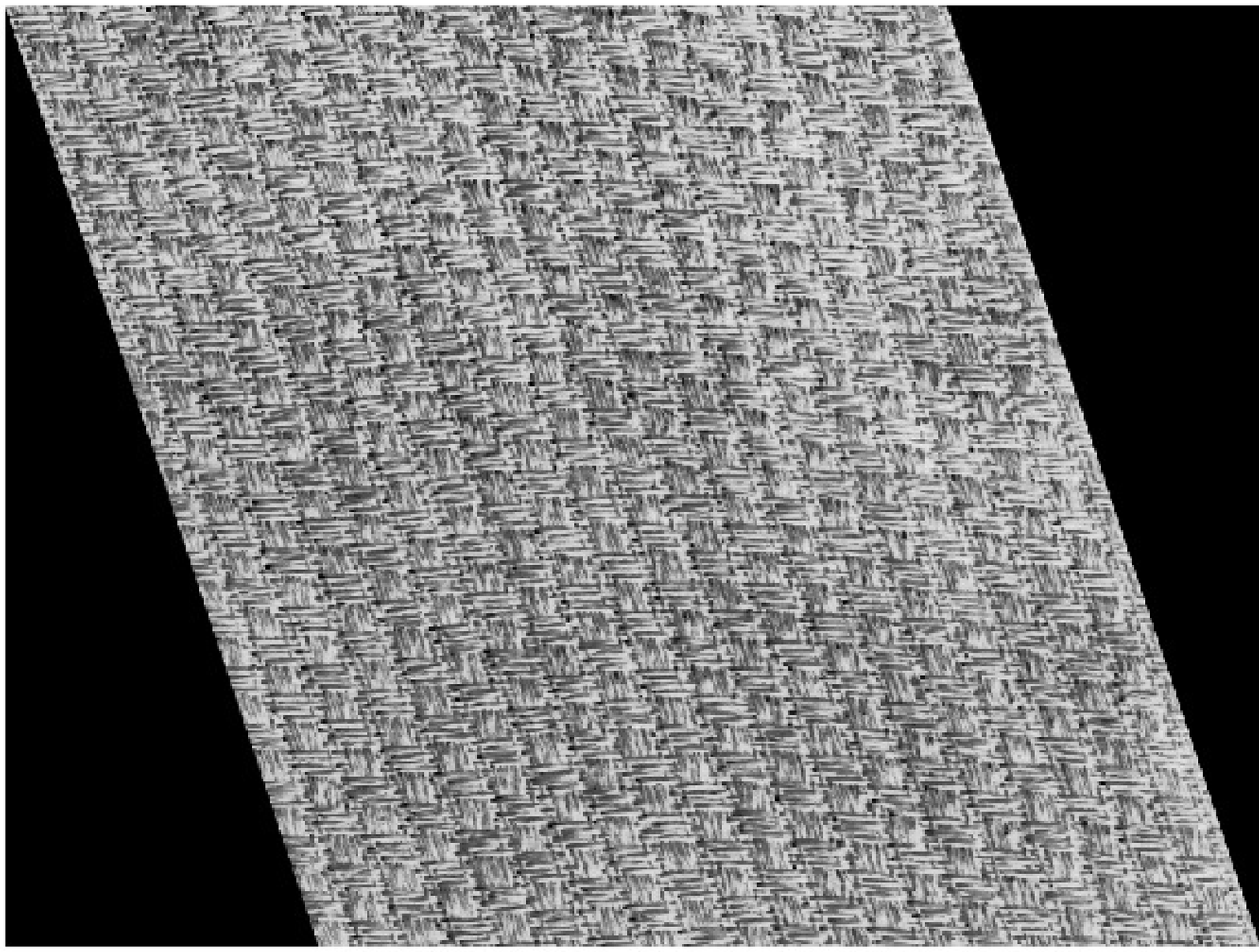}} \hspace{-0.6em}
\subfloat[]
{\includegraphics[width=.13\linewidth]{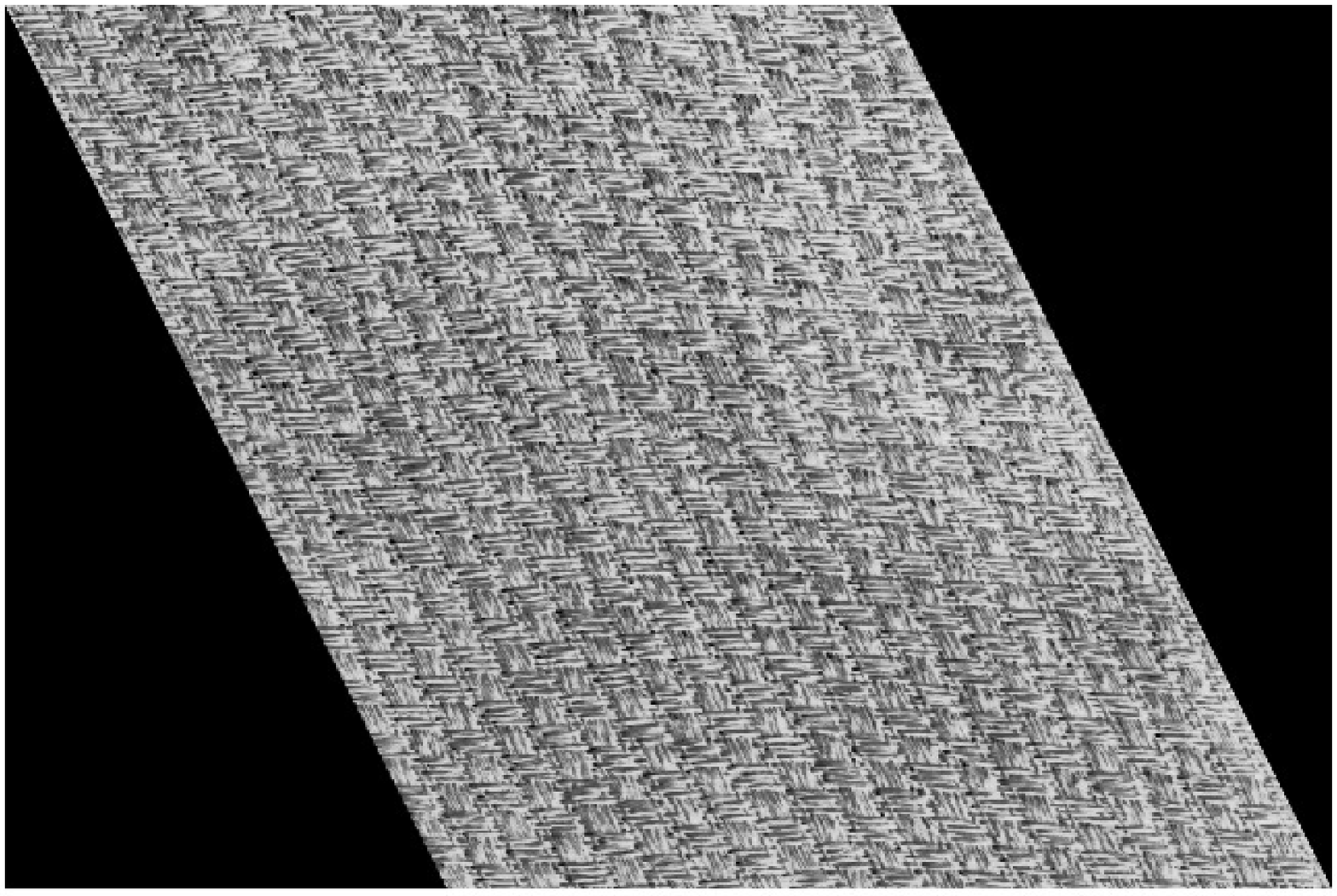}} \hspace{-0.6em}
\subfloat[]
{\includegraphics[width=.13\linewidth]{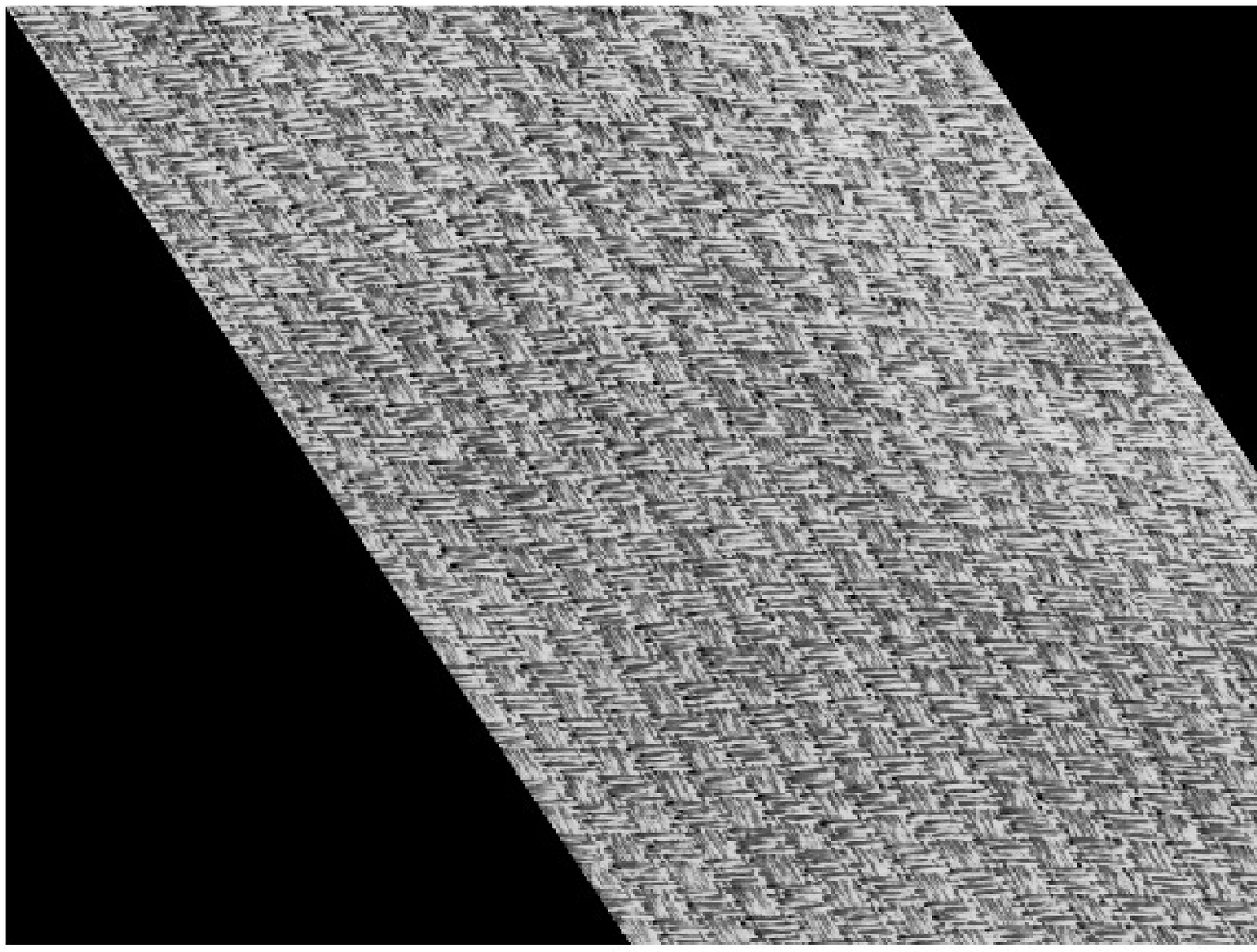}} \hspace{-0.6em}
\subfloat[]
{\includegraphics[width=.13\linewidth]{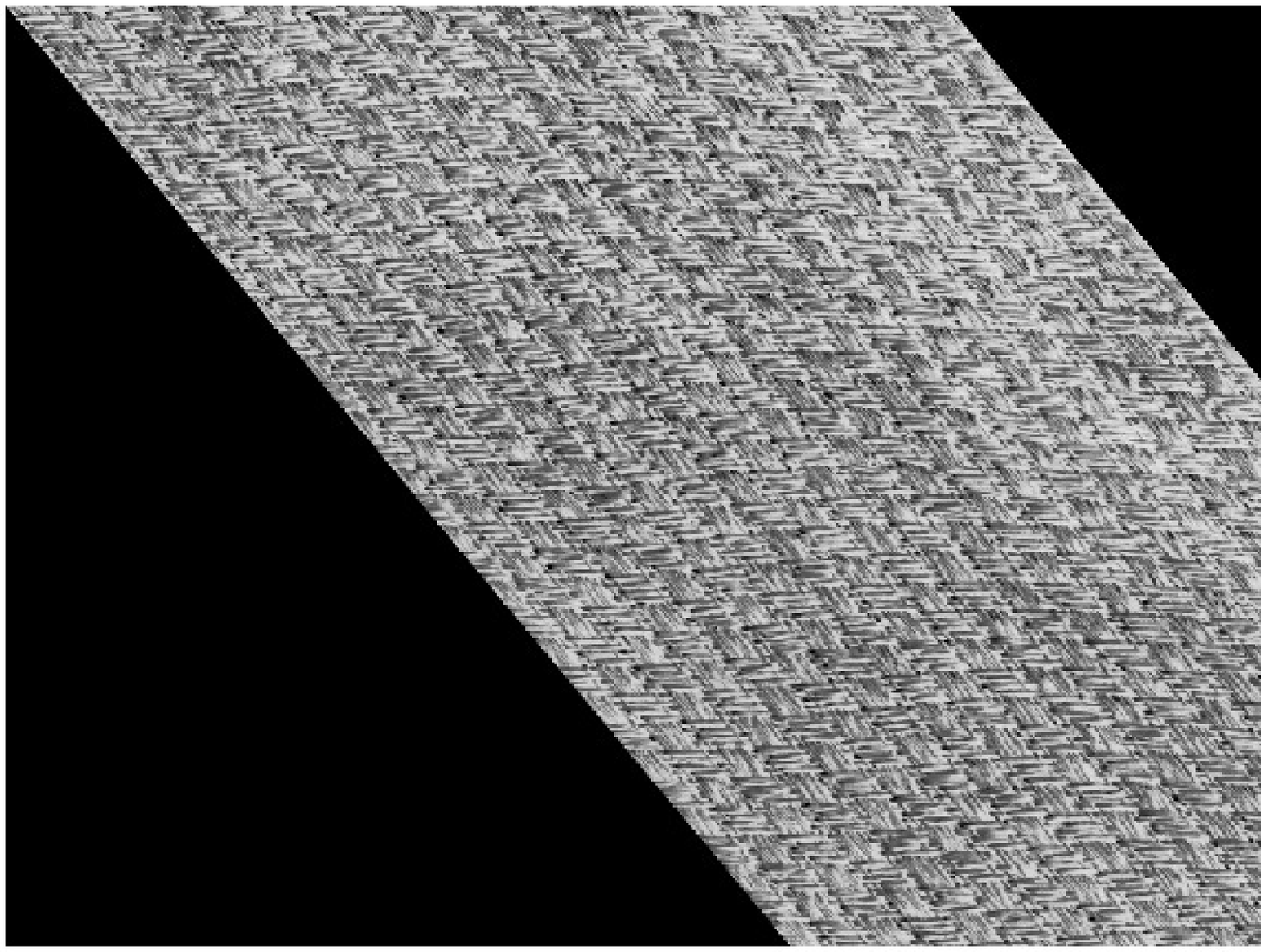}}\\ \vspace{-2.5em}
\subfloat[]
{\includegraphics[width=.13\linewidth]{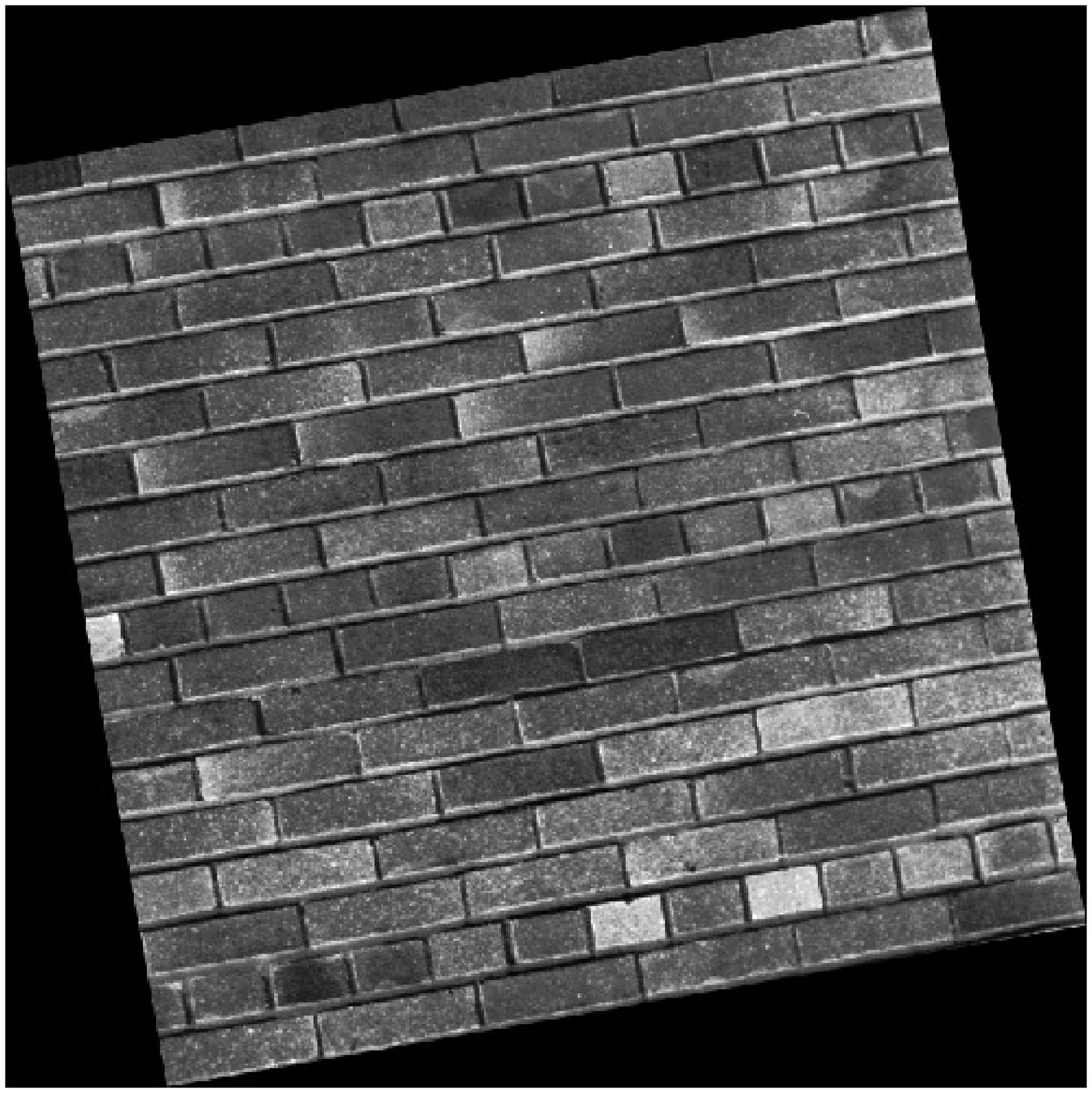}} \hspace{-0.6em}
\subfloat[]
{\includegraphics[width=.13\linewidth]{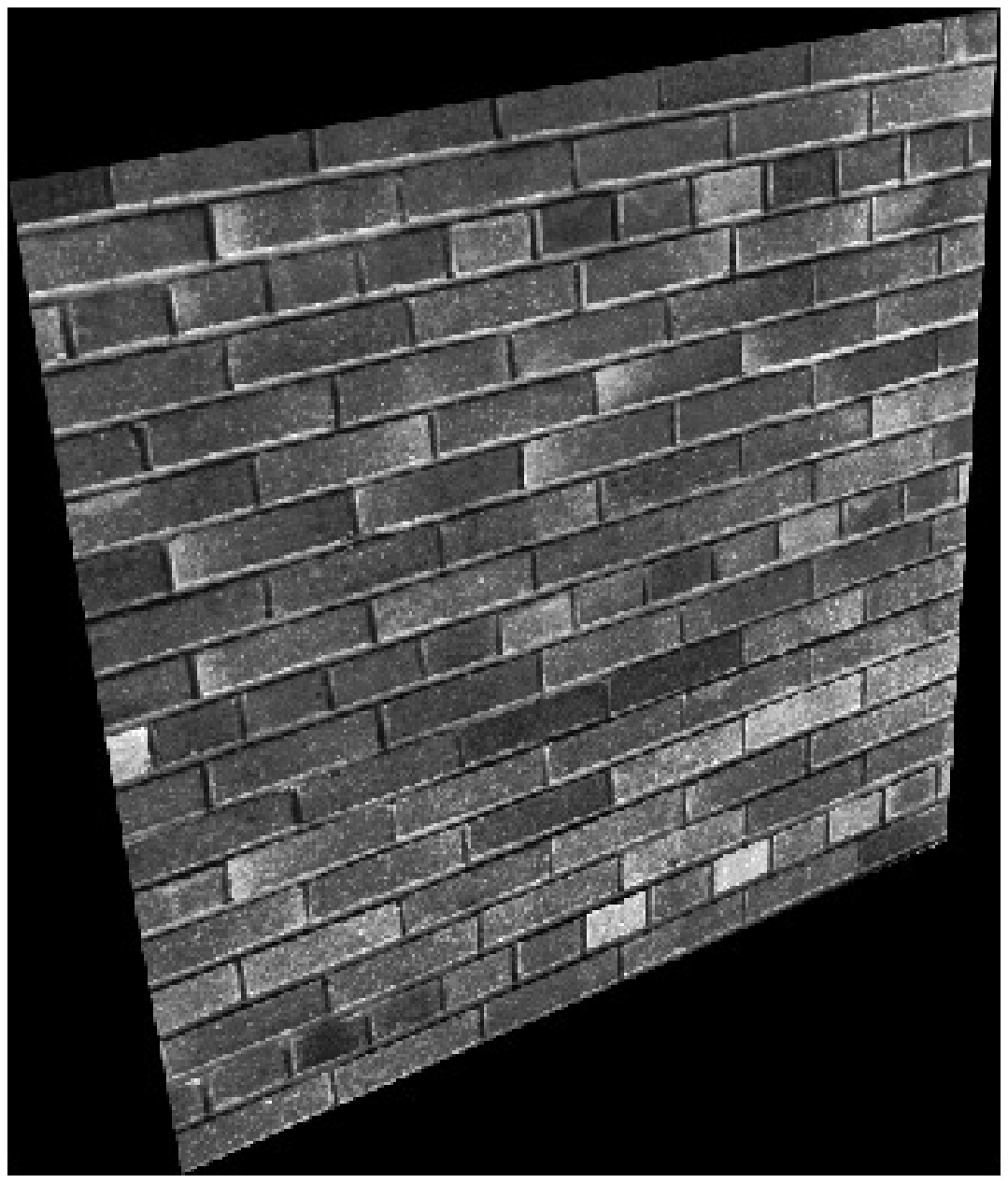}} \hspace{-0.6em}
\subfloat[]
{\includegraphics[width=.13\linewidth]{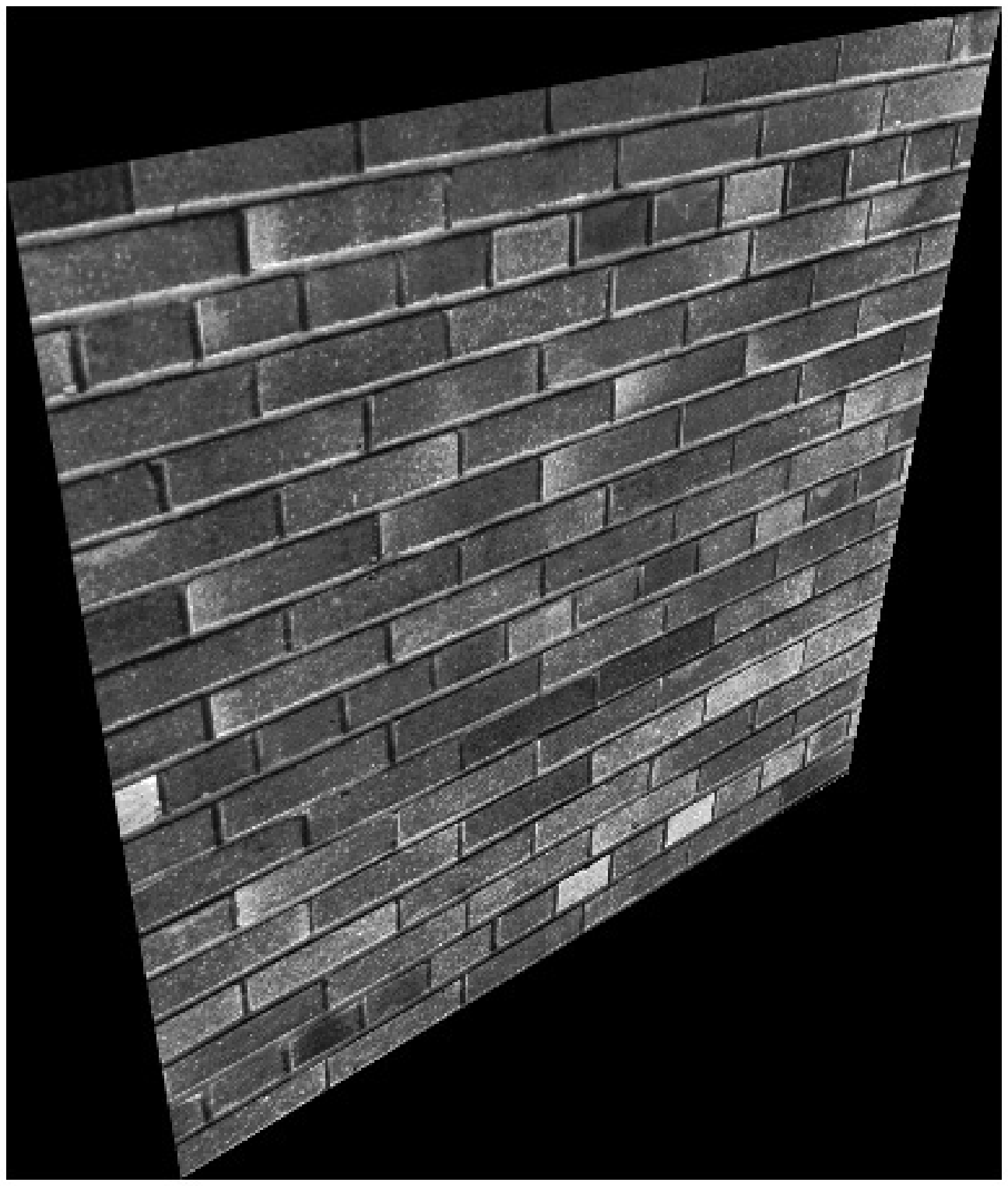}} \hspace{-0.6em}
\subfloat[]
{\includegraphics[width=.13\linewidth]{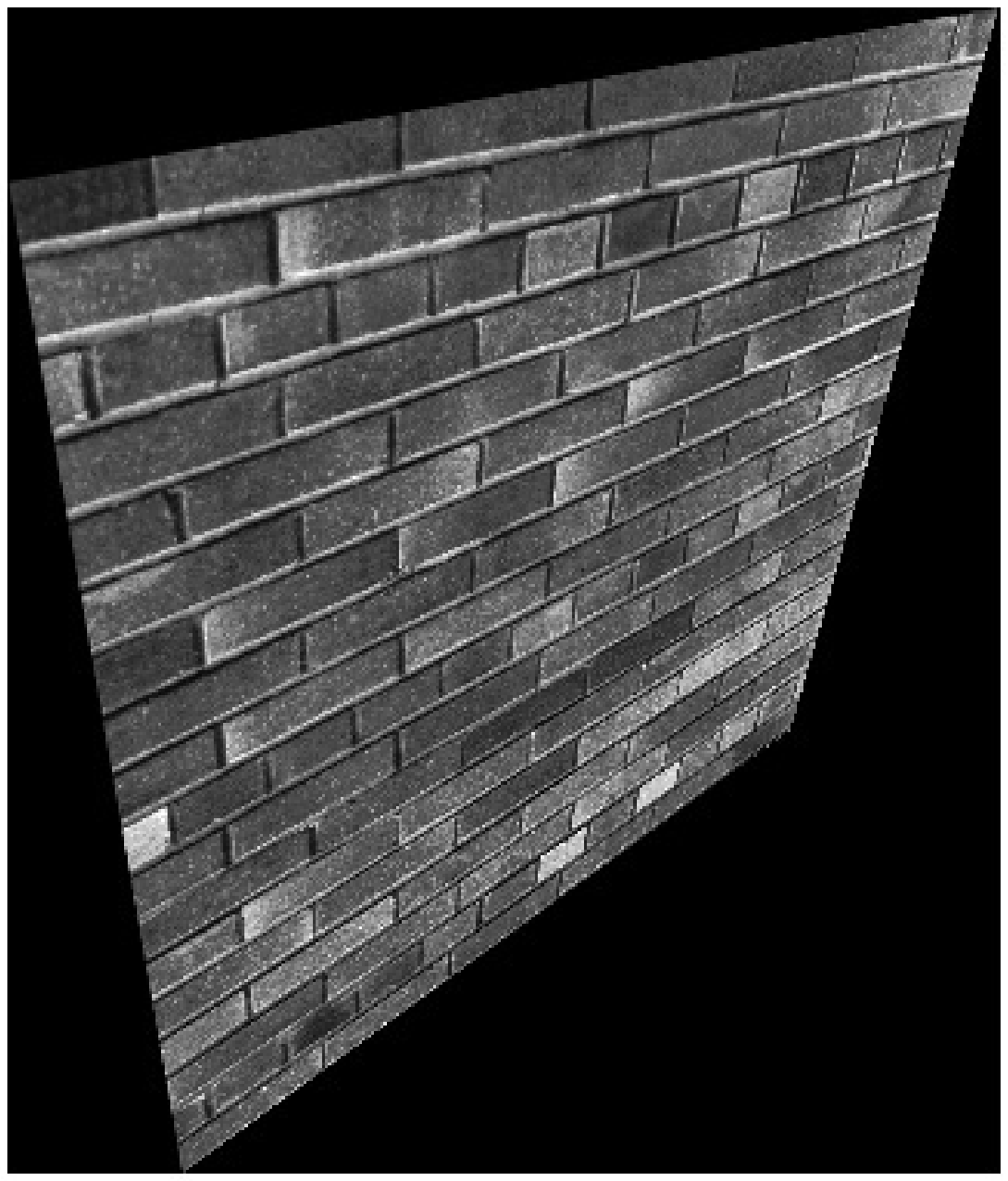}} \hspace{-0.6em}
\subfloat[]
{\includegraphics[width=.13\linewidth]{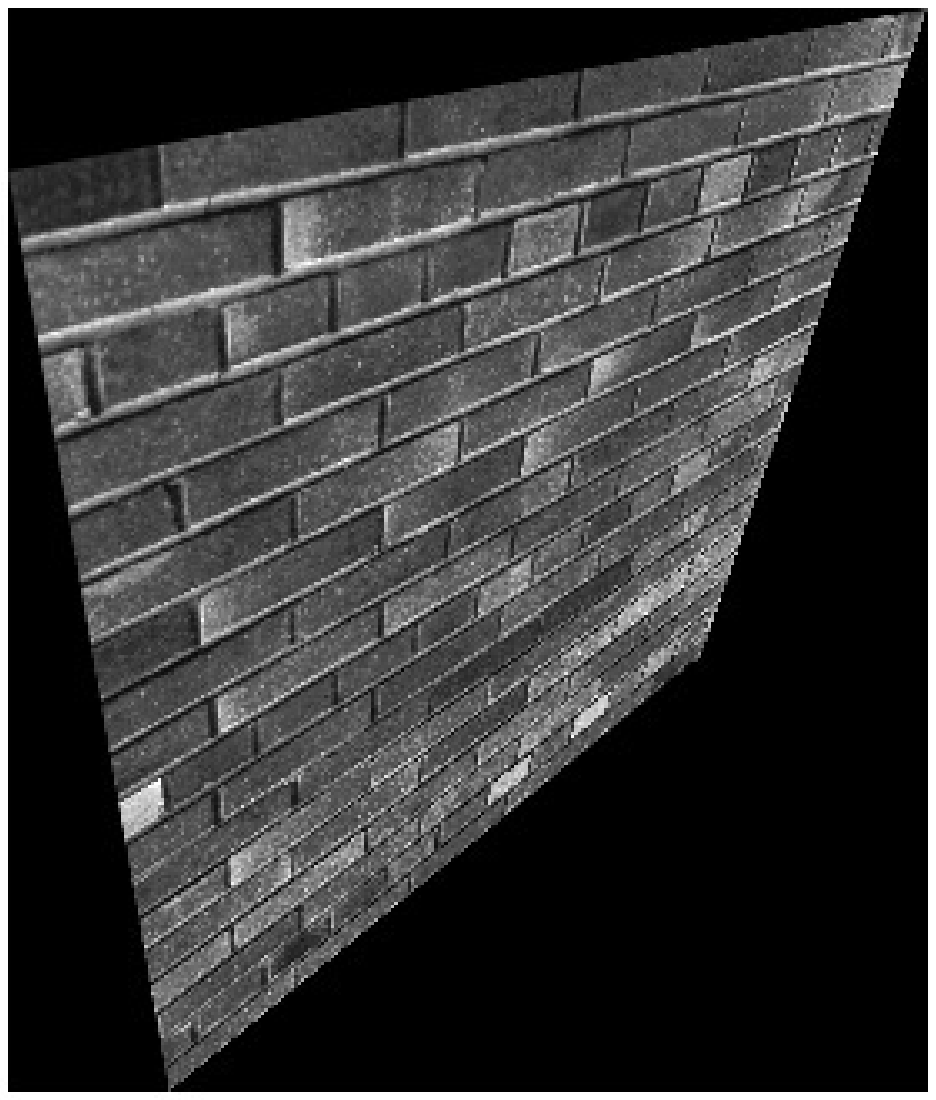}} \hspace{-0.6em}
\subfloat[]
{\includegraphics[width=.13\linewidth]{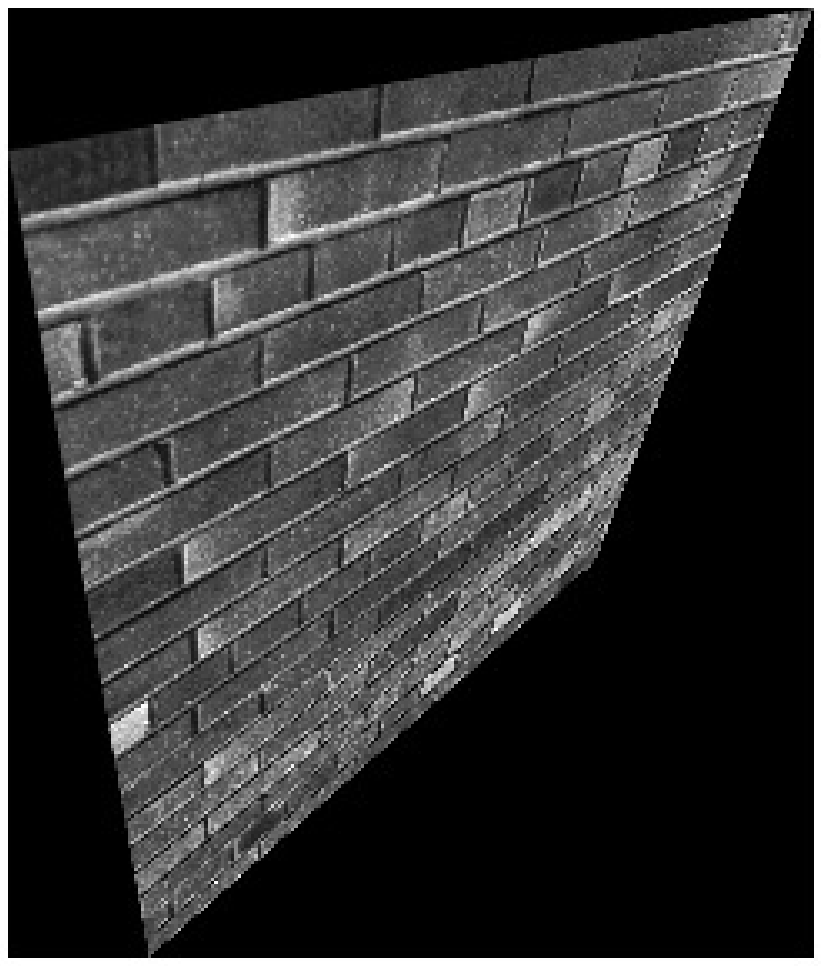}} \hspace{-0.6em}
\caption{Sample images in the Brodatz database and their deformations. The first row shows the 6 original images; in the second row, each image contains a unique texture but different regions of it have different lighting;
the third row shows the horizontal-shifted (distorted) images of an image;
the fourth row shows affine-transformed (change of viewpoints) images of an image.
}
\label{fig:Brodatz}
\end{figure*}

\citet{Tou09} show that region covariances generated by Gabor filters
effectively represent texture patterns in a region (patch). Given a patch of
size $60\times 60$, a Gabor filter of size 11$\times$11 with 8 parameters is
used to extract 2,500 feature vectors of length 8. This set of feature vectors
is then used to compute an 8$\times$8 covariance matrix for the specific patch.

Three clustering tests, one for each type of deformation, are carried out. In
each test, 300 transformed patches are generated equally from 3 different
textures and the region covariance is computed for each patch. Then clustering
algorithms are applied on the dataset of 300 region covariances belonging to 3
texture patterns. The way to generate transformed patches is described below.

\paragraph{\bf I. Lighting transformation:} A single lighting transformation
(demonstrated in Figure~\ref{fig:Brodatz}) is applied to three randomly drawn
images from the Brodatz database and 100 patches of size 60$\times$60 are
randomly picked from each of the 3 transformed images.

\paragraph{\bf II. Horizontal shearing:} Three randomly drawn images are
horizontally sheared by 100 different angles to get 3 sequences of 100 shifted
images. From each shifted image, a patch of size 60$\times$60 is randomly
picked.

\paragraph{\bf III. Affine transformation:} Three randomly drawn images are
affine transformed to create 3 sequences of 100 affine-transformed images. From
each transformed image, a patch of size 60$\times$60 is randomly picked.

%K\"{o}ppen et al.~\citep{Tou09} show that region covariances
%  generated by Gabor filters effectively represent texture patterns in a
%  patch. Given a patch, a Gabor filter of size $11\times 11$ with $8$ parameters
%  is used to extract feature vectors of length $8$. Afterwards, this set of
%  feature vectors is used to compute an $8\times 8$ covariance matrix for the
%  specific patch. For each type of transformation, the procedure to generate
%  patches and their region covariance matrices belonging to 3 different texture
%  patterns are described next.
%
%\paragraph{\bf i. Lighting transformation:} Lighting
%  transformation is applied to $3$ images randomly drawn from the pool of $112$
%  images, and $100$ patches of size $60\times 60$ are randomly picked from each
%  one of the $3$ transformed images. Then, regional covariance matrices for
%  these patches are computed.
%
%\paragraph{\bf ii. Horizontal shifting:} Three images are
%  randomly chosen and horizontally shifted by $100$ different angles to obtain
%  $3$ sequences of $100$ shifted images. From each shifted image, a patch of
%  size $60\times 60$ is randomly picked, and its covariance matrix is computed.
%
%\paragraph{\bf iii. Affine transformation:} We apply affine
%  transformations to 3 random images to create 3 sequences of 100
%  affine-transformed images. We pick a patch of size 60$\times$60 for each
%  affine-transformed image and compute its covariance matrix.
%
Figure~\ref{fig:PC_Cov} plots the projection of the embedded datasets generated by the above procedure
 onto their top three principal components (the embedding to Euclidean spaces is done by direct vectorization
 of the covariance matrices).
 The submanifold structure in each
 cluster can be easily observed.

\begin{figure}[htb!]
\centering
\subfloat[\footnotesize Lighting transformation]
{\label{fig:dim} \includegraphics[width=.3\linewidth]{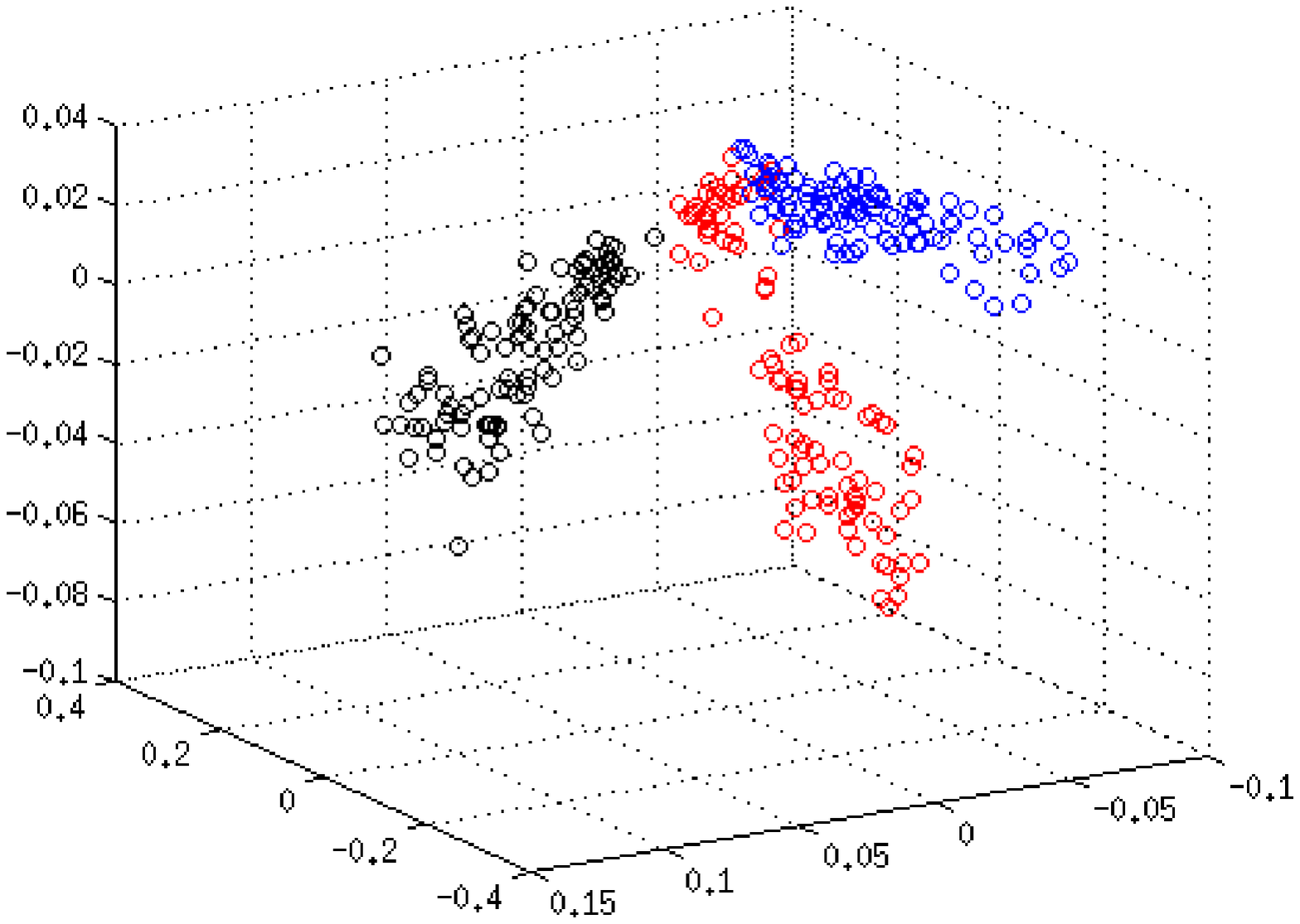}}
\subfloat[\footnotesize Horizontal shearing]
{\label{fig:hshift} \includegraphics[width=.3\linewidth]{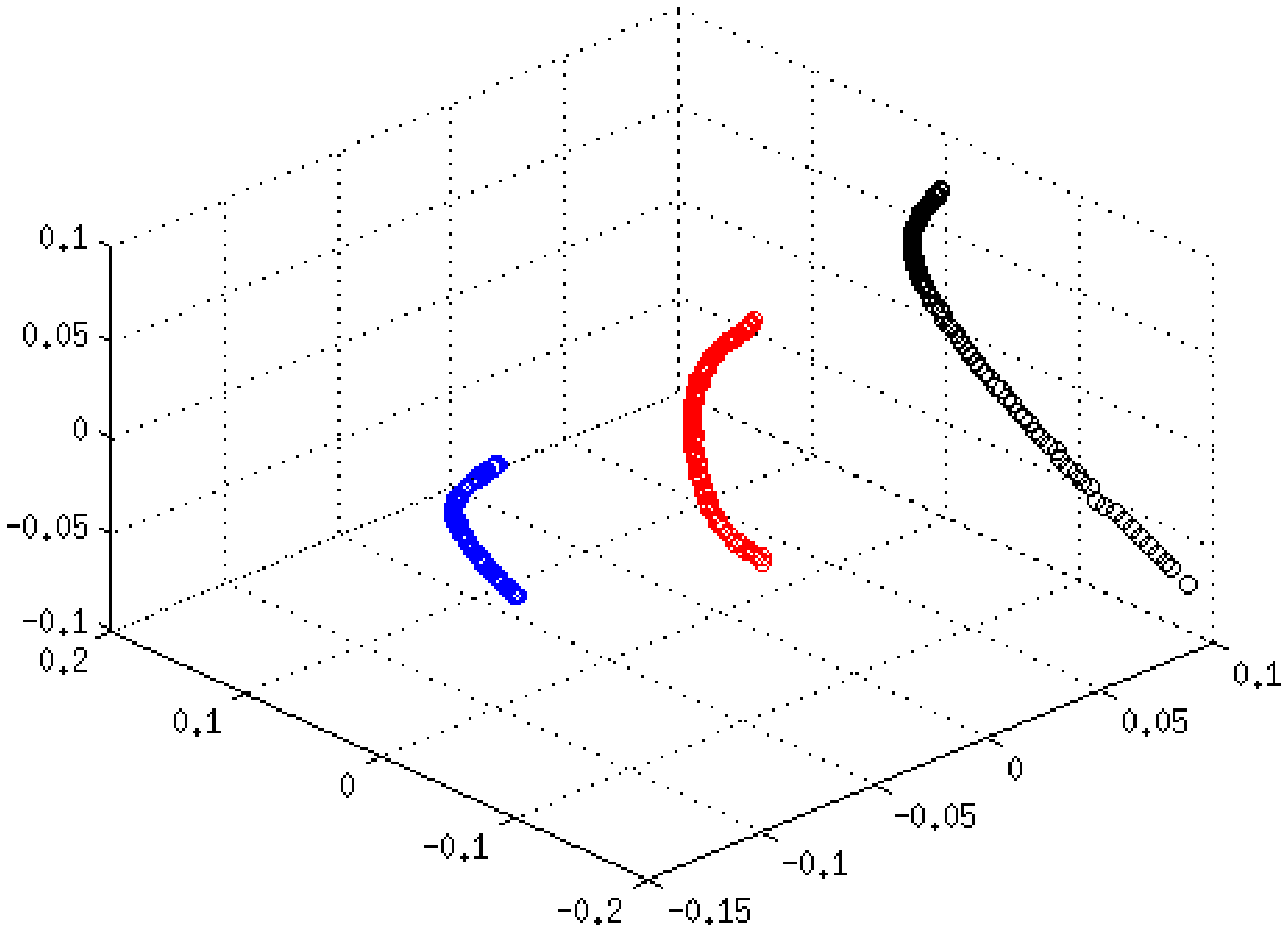}}
\subfloat[\footnotesize Affine transformation]
{\label{fig:affine} \includegraphics[width=.3\linewidth]{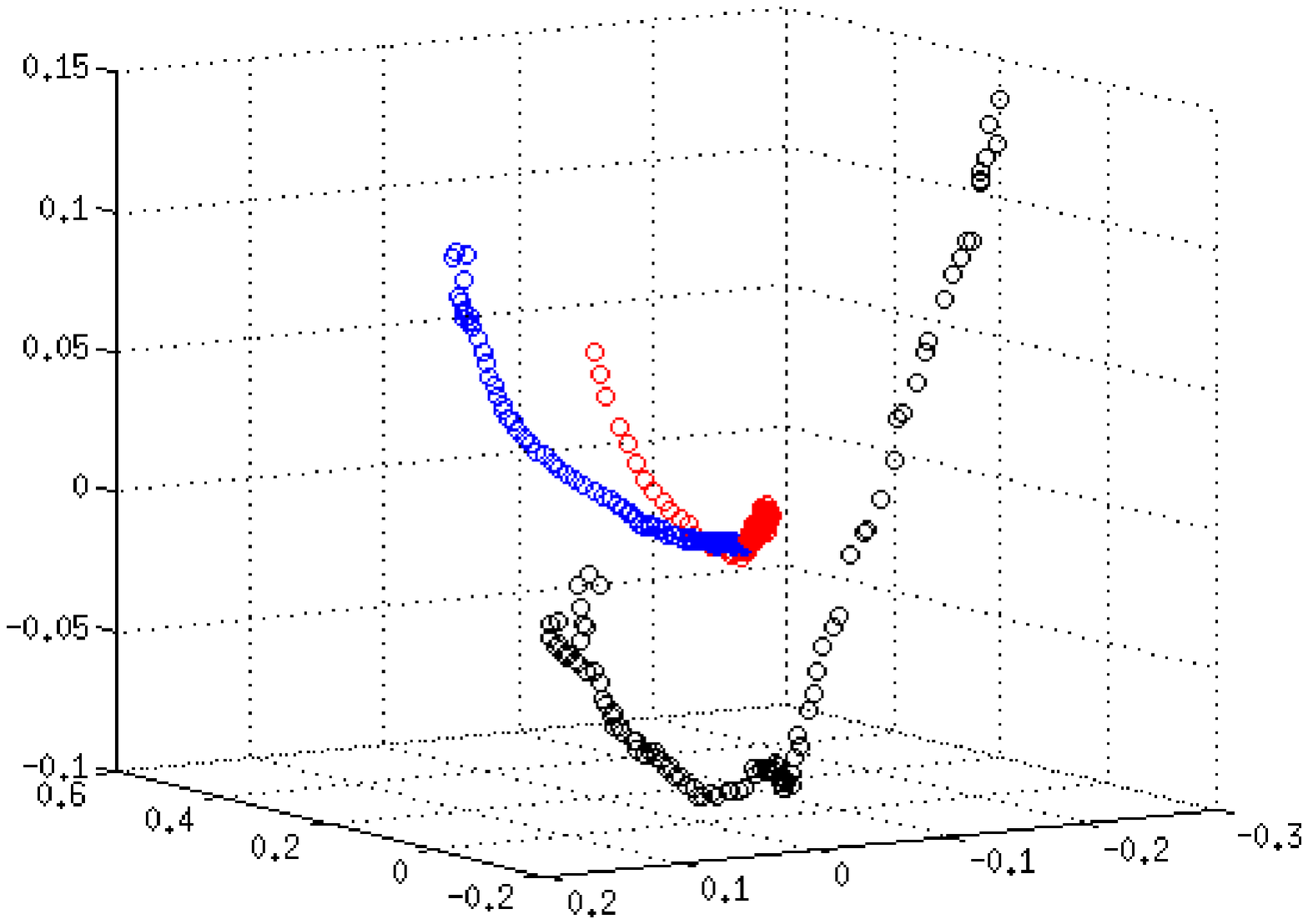}}
\caption{Projection of the covariance matrices of local patches of the
  transformed 3 images onto their top $3$ principal directions. For $3$ sample
  images, a dataset of $300$ covariance matrices is computed for each
  transformation type. The $8\times 8$ covariance matrices are identified as
  vectors in $\R^{64}$. The figure demonstrates the underlying structure of 3
  manifolds for the data generated with each kind of
  transformation.}\label{fig:PC_Cov}
\end{figure}

The procedure of generating the data is repeated $30$ times for each type of
transformation. GCT as well as the other three clustering methods are applied to
these datasets, and the average clustering rates are reported in
Table~\ref{table:texture}. GCT exhibits the best performance for all datasets
and for all types of transforms.

%\begin{center}
%    \begin{tabular}{| l | l | l | l | l |}
%    \hline
%    Methods & GCT & SMC & SCR & EKM\\ \hline
%    Lighting transformation & 0.73 & 0.53 & 0.68 & 0.67  \\ \hline
%    Horizontal shift & 0.95 & 0.61 & 0.85 & 0.76 \\ \hline
%    Affine tranformation & 0.83 & 0.53 & 0.82 & 0.76  \\ \hline
%    \end{tabular}
%\end{center}

\begin{table*}\centering
\ra{1.3}
\begin{tabular}{@{}rrrrr@{}}\toprule
Methods & GCT & SMC & SCR & EKM\\ \midrule

Lighting transformation & \bf{0.73} & 0.53 & 0.68 & 0.67  \\
Horizontal shifting & \bf{0.95} & 0.61 & 0.85 & 0.76 \\
Affine tranformation & \bf{0.83} & 0.53 & 0.82 & 0.76  \\
\bottomrule
\end{tabular}
\caption{Average clustering rates for each method over 30 datasets.}
\label{table:texture}
\end{table*}

%For the sequences 5c, 5m, 5v, 5v2, 5v3, we use 30 points to estimate the tangent spaces. For the sequences 10, 10v, 16c, 16v, we use 15 points.
%\begin{center}
%    \begin{tabular}{| l | l | l | l | l | l | l | l | l | l |}
%    \hline
%    Methods & 5c & 5m & 5v & 5v2 & 5v3 & 10 & 10v & 16c & 16v\\ \hline\hline
%    GCT & 0.85 & 0.80 & 0.83 & 0.73 & 0.70 & 0.80 & 0.84 & 0.86 & 0.84 \\ \hline
%    SMC & 0.83 & 0.70 & 0.79 & 0.71 & 0.75 & 0.75 & 0.77 & 0.78 & 0.79 \\ \hline
%    SCR & 0.78 & 0.71 & 0.79 & 0.67 & 0.75 & 0.76 & 0.78 & 0.70 & 0.76 \\ \hline
%    EKM & 0.74 & 0.69 & 0.77 & 0.68 & 0.76 & 0.75 & 0.78 & 0.69 & 0.73 \\ \hline
%    \end{tabular}
%\end{center}

\subsubsection{Clustering Dynamic Patterns.}\label{exp:Ballet} Spatio-temporal
data such as dynamic textures and videos of human actions can often be
approximated by linear dynamical models \citep{DBLP:journals/ijcv/DorettoCWS03,
  Turaga+11}. In particular, by leveraging the auto-regressive and moving
average (ARMA) model, we experiment here with two spatio-temporal databases:
Dyntex++ and Ballet. Following~\citet{Turaga+11}, we employ the ARMA model to
associate local spatio-temporal patches with linear subspaces of the same
dimension. We then apply manifold clustering on the Grassmannian in order to
distinguish between different textures and actions in the Dyntex++ and Ballet
database respectively.

\paragraph{\bf ARMA Model.} The premise of ARMA modeling is based on the
assumption that the spatio-temporal dataset under study is governed by a small
number of latent variables whose temporal variations obey a linear rule. More
specifically, if $\mb{f}(t)\in\R^p$ is the observation vector at time $t$ (in
our case, it is the vectorized image frame of a video sequence), then
\begin{equation}\label{equ:arma}
\begin{aligned}
\mb{f}(t)=\mb{C}\mb{z}(t)+\pmb{\epsilon}_1(t) \qquad \pmb{\epsilon}_1(t)\sim
N(\mb{0},\pmb{\Sigma}_1)\\
\mb{z}(t+1)=\mb{A}\mb{z}(t)+\pmb{\epsilon}_2(t) \qquad \pmb{\epsilon}_2(t)\sim
N(\mb{0},\pmb{\Sigma}_2)
\end{aligned}
\end{equation}
where $\mb{z}(t)\in\R^d$, $d\leq p$, is the vector of latent variables,
$\mb{C}\in\R^{p\times d}$ is the observation matrix, $\mb{A}\in\R^{d\times d}$
is the transition matrix, and $\pmb{\epsilon}_1(t)\in \R^p$ and
$\pmb{\epsilon}_2(t)\in \R^d$ are i.i.d.~sampled vector-values r.vs.\ obeying
the Gaussian distributions $\mathcal{N}(0,\pmb{\Sigma}_1)$ and
$\mathcal{N}(0,\pmb{\Sigma}_2)$, respectively.

We next explain the idea of~\citet{Turaga+11} to associate subspaces with
spatio-temporal data. Given data $\{\mb{f}(t)\}_{t=\tau_1}^{\tau_2}$, the ARMA
parameters $\mb{A}$ and $\mb{C}$ can be estimated according to the procedure in
\citet{Turaga+11}. Moreover, by arbitrarily choosing $\mb{z}(0)$, it can be
verified that for any $m\in\mathbb{N}$,
\[
\E \left[ \begin{array}{c}
\mb{f}(\tau_1)\\
\mb{f}(\tau_1+1)\\
\vdots \\
\mb{f}(\tau_1+m-1) \end{array} \right]
=\left[ \begin{array}{c}
\mb{C}  \\
\mb{C}\mb{A}\\
\vdots \\
\mb{C}\mb{A}^{m-1} \end{array} \right]\mb{z}(\tau_1).
\]
We then set $\mb{V} :=[\mb{C}^T, (\mb{CA})^T,..., (\mb{C}\mb{A}^{m-1})^T]^T\in
\mathbb{R}^{mp\times d}$, which is known as the $m$th order observability
matrix. If the observability matrix is of full column rank, which was the case
in all of the conducted experiments, the column space of $\mb{V}$ is a
$d$-dimensional linear subspace of $\R^{pm}$. In other words, the ARMA model
estimated from data $\{\mb{f}(t)\}_{t=\tau_1}^{\tau_2}$, $\tau_1\leq \tau_2$,
gives rise to a point on the Grassmannian $\G(mp,\ell)$. For a fixed dataset
$\{\mb{f}(t)\}_{t=1}^{\tau}$, different choices of $(\tau_1,\tau_2)$, s.t.\
$\tau_1, \tau_2 \leq \tau$, and several local regions within the image give rise
to different estimates of $\mb{A}$ and $\mb{C}$ and thus to different points in
$\G(mp,\ell)$.

\paragraph{\bf Dynamic textures.} The Dyntex++
database~\citep{DBLP:conf/eccv/GhanemA10} contains $3600$ dynamic textures
videos of size $50\times 50\times 50$, which are divided into $36$
categories. It is a hard-to-cluster database due to its low resolution. Three
videos were randomly chosen, each one from a distinct category from the
available $36$ ones.

Per video sequence, $50$ patches of size $40\times 40 \times 20$ are randomly
chosen.  Each frame of the patch is vectorized resulting into patches of size
$1600 \times 20$. To reduce the size to $30\times 20$, a (Gaussian) random
(linear) projection operator is applied to each patch. As a result, each patch
is reduced to the set
$\{\mb{f}(t)\}_{t=\tau_1}^{\tau_1+20}\subset\mathbb{R}^{30}$. We fix $d=3$ and
$m=3$ and use each such set $\{\mb{f}(t)\}_{t=\tau_1}^{\tau_1+20}$ to estimate
the underlying ARMA model. Consequently, $150$ points on $\G(90,3)$ are
generated, $50$ per video category.

We expect that points in $\G(90,3)$ of the same cluster lie near a submanifold
of $\G(90,3)$.  This is due to the repeated pattern of textures in space and
time (they often look like a shifted version of each other in space and
time). To visualize the submanifold structure, we isometrically embedded
$\G(90,3)$ into a Euclidean space \citep{Basri_nearest_subspace11}, so that
subspaces are mapped to Euclidean points. We then projected the latter points on
their top 3 principal components.  Figure~\ref{fig:dyntex} demonstrates this
projection as well as the submanifold structure within each cluster.

\paragraph{\bf Ballet database.} The Ballet
database~\citep{DBLP:journals/pami/WangM09} contains $44$ videos of $8$ actions
from a ballet instruction DVD. The frames of all videos are of size $301\times
301$ and their lengths vary and are larger than 100. Different performers have
different attire and speed. Three videos, each one associated with a different
action, were randomly chosen.

\begin{figure*}[htb!]
\centering
\includegraphics[width=0.5\textwidth]{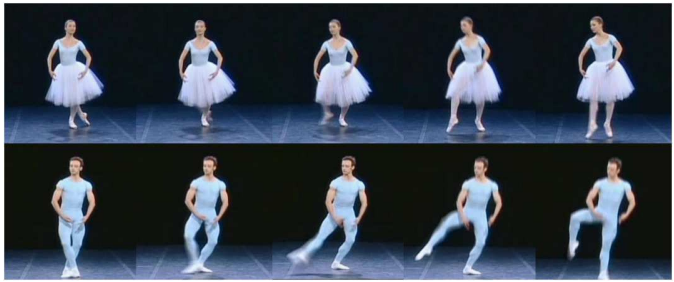}
\caption{Two samples of Ballet video sequences: the first and second rows are
  from videos demonstrating actions of hopping and leg-swinging,
  respectively.} \label{fig:ballet_images}
\end{figure*}

Spatio-temporal patches are generated by selecting $10$ consecutive frames of
size $301\times 301$ from each one of the following overlapping time intervals:
$\{1, \ldots, 10\}$, $\{4, \ldots, 13\}$, $\{7, \ldots, 16\}$, \ldots, $\{91,
\ldots, 100\}$. In this way, for each of the three videos, 31 spatio-temporal
patches of size $301\times 301\times 10$ are generated. As in the case of the
Dyntex++ database, video patches are vectorized and downsized to
spatio-temporal patches of size $30\times 10$. Following the previous ARMA
modeling approach, we set $d=3$ and $m=3$ and associate each such patch with a
subspace in $\G(90,3)$.  Consequently, 93 subspaces (31 per cluster) in the
Grassmannian $\G(90,3)$ are generated. Figure~\ref{fig:Ballet} visualizes the 3D
representation of the subspaces created from three random videos. Their
intersection represents still motion.

The procedure described above (for generating data by randomly choosing 3 videos
from the Dyntex++ and Ballet databases and applying clustering methods on
$\G(90,3)$) is repeated $30$ times. The average clustering accuracy rates are
reported in Table~\ref{table:dynamic}. GCT achieves the highest rates on both
datasets.

\begin{figure}[htb!]
\centering
\subfloat[\footnotesize Dyntex++]
{\label{fig:dyntex} \includegraphics[width=.4\linewidth]{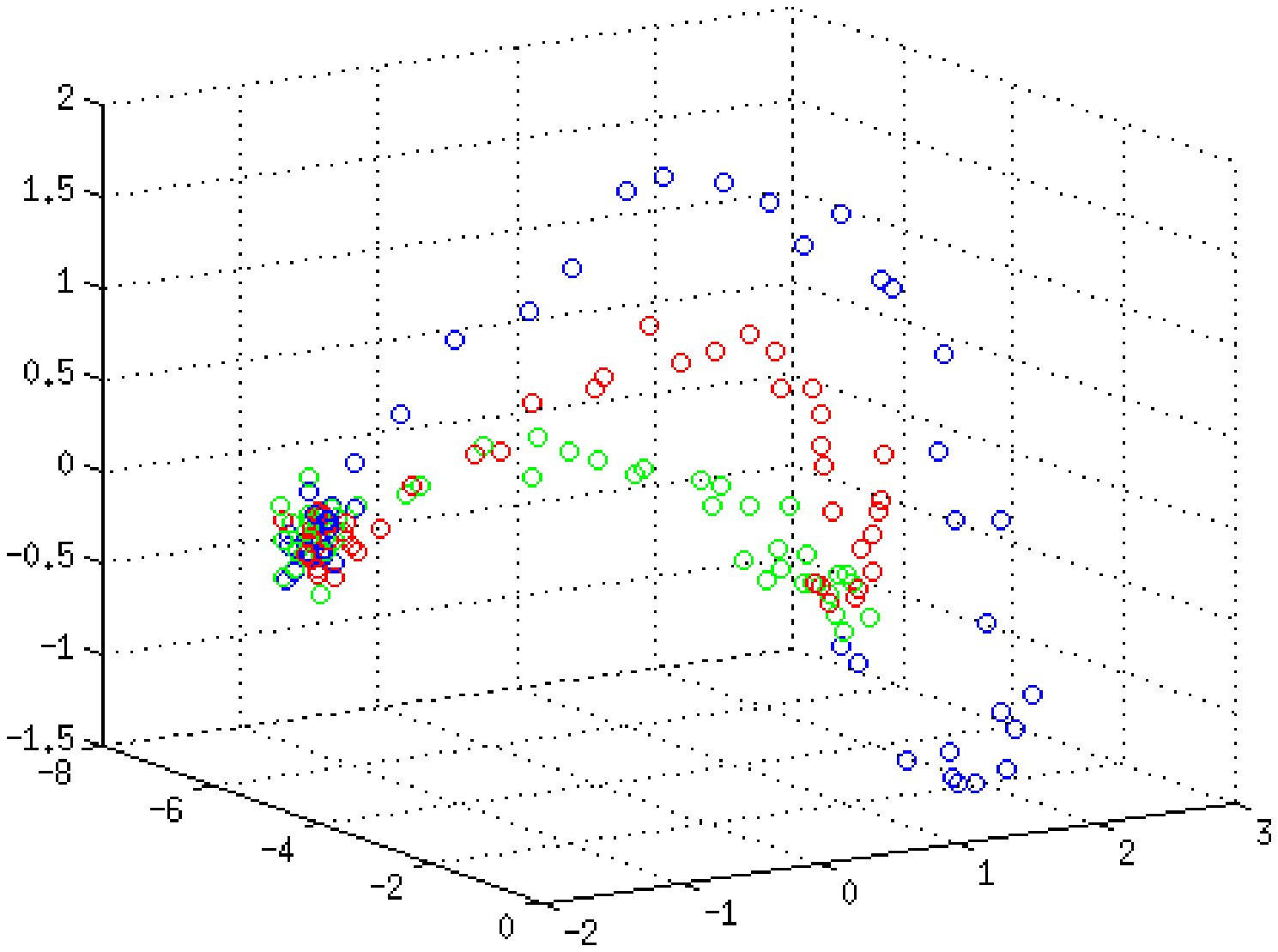}}
\subfloat[\footnotesize Ballet]
{\label{fig:Ballet} \includegraphics[width=.4\linewidth]{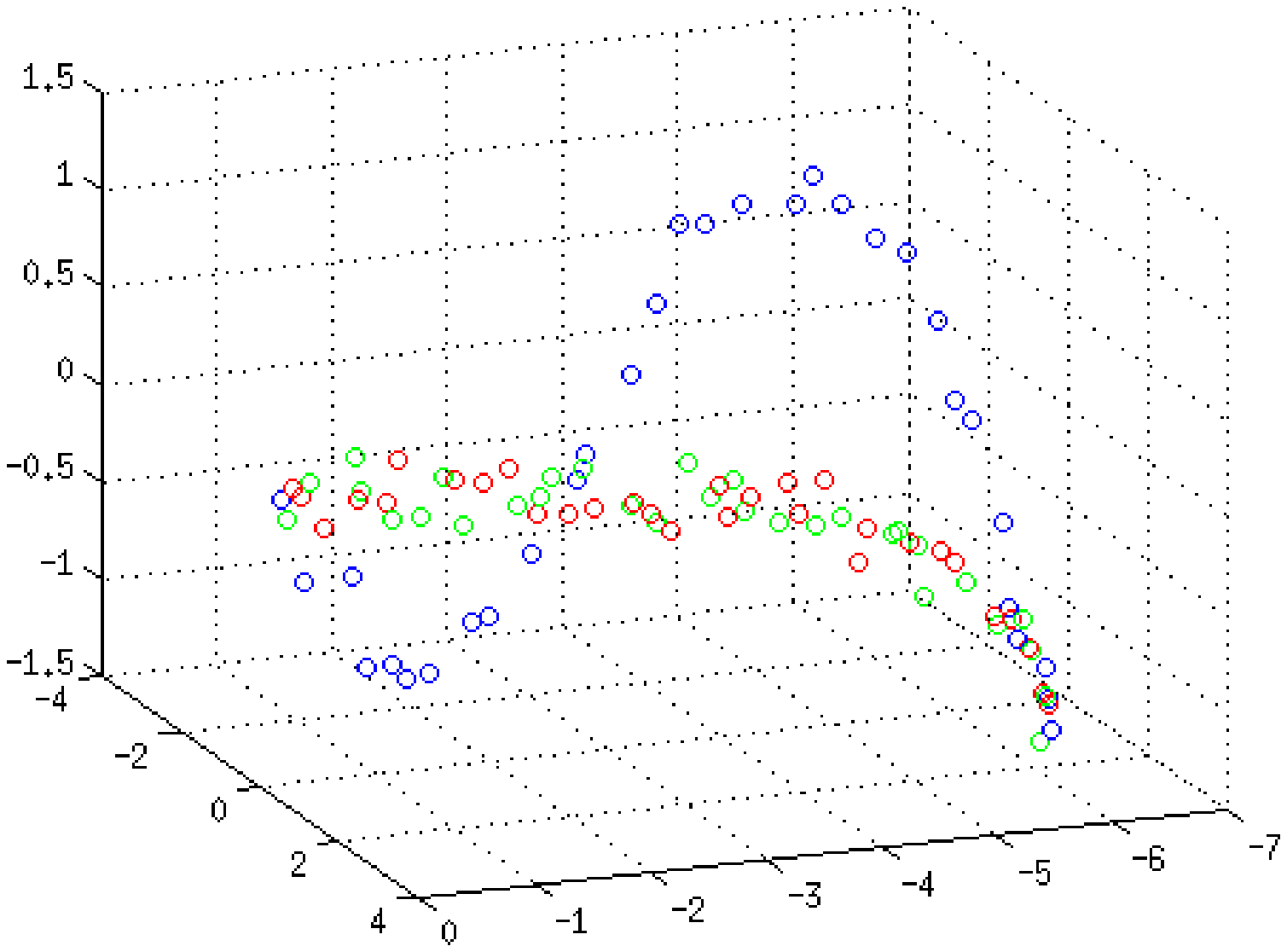}}
\caption{Projection onto top 3 principal components of the two embedded datasets
  (the embedding into Euclidean spaces is according
  to~\citet{Basri_nearest_subspace11}). A submanifold structure for each cluster
  is clearly depicted.}
\label{fig:PC_grass}
\end{figure}

%\begin{center}
%    \begin{tabular}{| l | l | l | l | l |}
%    \hline
%    Methods & GCT & SMC & SCR & EKM\\ \hline
%    Dyntex++ & 0.85 & 0.69 & 0.77 & 0.42  \\ \hline
%    Ballet & 0.81 & 0.76 & 0.68 & 0.47 \\ \hline
%    \end{tabular}
%\end{center}

\begin{table*}\centering
\ra{1.3}
\begin{tabular}{@{}rrrrr@{}}\toprule
Methods & GCT & SMC & SCR & EKM\\ \midrule

Dyntex++ & \bf{0.85} & 0.69 & 0.77 & 0.42  \\
Ballet & \bf{0.81} & 0.76 & 0.68 & 0.47 \\
\bottomrule
\end{tabular}
\caption{Average clustering accuracy rates for the Dyntex++ and Ballet
  datasets.}
\label{table:dynamic}
\end{table*}

\section{Proof of Theorem~\ref{theorem:all}}\label{sec:theory}

The idea of the proof is as follows. After excluding points sampled near the
possibly nonempty intersection of submanifolds, we form a graph whose vertices
are the points of the remaining set and whose edges are determined by
$\mb{W}$. The proof then establishes that the resulting graph has two connected
components, which correspond to the two different submanifolds $S_1$ and
$S_2$. Spectral clustering can exactly cluster such a graph with appropriate
choice of its tuning parameter $\sigma$, which can be specified by self-tuning
mechanism~\citep{zelnik04tuning}. This claim follows from
\citet{DBLP:conf/nips/NgJW01} and its unpublished supplemental material.

The basic strategy of the proof and its organization are described as
follows. Section~\ref{sec:addition_notation} presents additional notation used
in the proof. Section~\ref{sec:model_recap} reminds the reader the underlying
model of the proof (with some additional details). Section~\ref{sec:events}
eliminates undesirable events of negligible probability (it clarifies the term
$1-C_0 N\exp[-Nr^{d+2}/C'_0]$ in the statement of the theorem).
%The proof then begins in Section~\ref{sec:SameS} and ends in
                                %Section~\ref{sec:MainNoisy}.

The rest of the proof (described in Sections~\ref{sec:SameS}-\ref{sec:MainNoisy}) is briefly sketched as follows. For simplicity, we first assume no noise, i.e., $\tau=0$. We define a ``sufficiently large'' set $X^*$ (and its subsets $X_1^*$ and $X_2^*$)
by the following formula (which uses the notation $X_1=S_1\cap X$ and $X_2=S_2\cap X$):
\begin{equation}\label{equ:X^*}
X_1^*=\{x\in X_1 | B(x,r)\cap X_2 = \emptyset\}, \ X_2^*=\{x\in X_2 | B(x,r)\cap X_1 = \emptyset\} \text{ and } X^* = X_1^*\cup X_2^*.
\end{equation}
In the first part of the proof (see Section~\ref{sec:SameS}), we show that the
graphs of $X_1^*$ and $X_2^*$ (with weights $\mb{W}$) are respectively
connected. If we can show that the graphs of $X_1^*$ and $X_2^*$ are
disconnected from each other, then the proof can be concluded. To this end, the
subsequent auxiliary sets $\hat{X}_1$ and $\hat{X}_2$ will be instrumental in
the proof. We fix a constant $\delta$ (to be specified later
in~\eqref{eq:def_delta}), which depends on $r$, $\eta$ and the angles of
intersection of $S_1$ and $S_2$, and define
\begin{equation}\label{equ:Xhat}
\hat{X}_1=\{x\in X_1 | \dist_g(x, S_2) \geq \delta\},\ \hat{X}_2=\{x\in X_2 |
\dist_g(x, S_1) \geq \delta\} \text{ and } \hat{X}=\hat{X}_1\cup \hat{X}_2.
\end{equation}
We will verify that $X_1^*\subset \hat{X}_1$ and $X_2^*\subset \hat{X}_2$. In
fact, it will be a consequence of the second part of the proof. This part shows
that the graph of $\hat{X}^c$ is disconnected from the graph of $X^*_1$ as well
as graph of $X^*_2$.
%stopped here
Therefore, $X^*_1$ and $X^*_2$ cannot be connected via points in $\hat{X}^c$. At last,
we show that they also cannot be connected within $\hat{X}$. That is, we show in the third
part of the proof (Section~\ref{sec:diffangle}) that the graphs of $\hat{X}_1$
and $\hat{X}_2$ are disconnected from each other. These three parts imply that
the graphs of $X^*_1$ and $X^*_2$ form two connected components within $X^*$. By
definition, $X_1^*$ and $X_2^*$ are identified with $S_1$ and $S_2$
respectively. To conclude the proof (for the noiseless case), we estimate the
measure of the set $X^{*c}$, which was excluded. More precisely, we consider the
measure of the set $X_{S_1\cap S_2} \supset X^{*c}$, which we define as follows
\begin{equation}
\label{eq:s1caps2}
X_{S_1\cap S_2} =\{x\in X_1 | \dist_g(x, S_2) <r\}\cup \{x\in X_2 | \dist_g(x, S_1) <r\}.
\end{equation}
This measure estimate and the conclusion of the proof (to the noiseless case)
are established in Section~\ref{sec:MainNoiseFree}. Section~\ref{sec:MainNoisy}
discusses the generalization of the proof to the noisy case.

Various ideas of the proof follow~\citet{LocalPCA}, which considered
multi-manifold modeling in Euclidean spaces. Some of the arguments in the proof
of~\citet{LocalPCA} even apply to general metric spaces, in particular, to
Riemannian manifolds. We thus tried to maintain the notation
of~\citet{LocalPCA}.

However, the algorithm construction and the main theoretical analysis
of~\citet{LocalPCA} are valid only when the dataset $X$ lies in a Euclidean
space and it is nontrivial to extend them to a Riemannian manifold. Indeed, the
basic idea of~\citet{LocalPCA} is to compare local covariance matrices and use
this comparison to infer the relation between the corresponding data points,
over which those matrices were generated. However, comparing local covariance
matrices in the case where the ambient space is a Riemannian manifold is not
straightforward as in Euclidean spaces. This is due to the fact that local
covariance matrices are computed at different tangent spaces with different
coordinate systems. Instead we show that it is sufficient to compare the ``local
directional information'' (i.e., empirical geodesic angles) and ``local
dimension''. Both of these quantities are derived from the local covariance
matrices. However, due to the nonlinear mapping to the tangent spaces, which
distorts the uniform assumption within the ambient space, care must be taken in
using the inverse nonlinear map, i.e., the logarithm map.

\subsection{Notation}\label{sec:addition_notation} We provide additional
notation to the one in Section~\ref{sec:notation}. Readers are referred
to~\citet{docarmo92} for a complete introduction to Riemannian geometry.

Let $B(x,r)$ and $B_{x}(\mb{0},r)$ denote the $r$-neighborhoods of $x$ and
$\mb{0}$ in $M$ and $T_xM$ respectively. They are related by the exponential
map, $\Phi_x$, as follows: $B(x,r)=\Phi_x(B_{x}(\mb{0},r))$. We refer to the
coordinates obtained in the tangent space by the exponential map $\Phi$ as
normal coordinates. Using normal coordinates, $B_{x}(\mb{0},r)\subset T_x M$ is
endowed with the Riemannian metric $\dist_g$ and measure $\mu_g$. On the other
hand, the tangent space $T_xM$ can also be identified with $\R^D$ by choosing an
orthonormal basis. This provides Euclidean metric $\dist_E$ and measure $\mu_E$
on $T_x M$, in particular, on $B_{x}(\mb{0},r)$. There is a simple relation
between $\mu_E$ and $\mu_g$~\citep{docarmo92}:
\begin{equation}\label{equ:chart}
\mu_g(d\mb{y})=\mu_E(d\mb{y})+\OO(r^2)d\mb{y} \qquad\text{for}\quad \mb{y}\in B_{x}(\mb{0},r).
\end{equation}
Figure~\ref{fig:spheregeodesic} highlights the difference between $\dist_E$ and
$\dist_g$. It shows the tangent space $T_n \Sb^2$ of the north pole, $n$, of
$\Sb^2$ and the straight blue line connecting $\Phi_{n}^{-1}(x)$ and
$\Phi_n^{-1}(y)$ in $T_n$; it is the shortest path w.r.t.~$\dist_E$. On the
other hand, the shortest path w.r.t.~$\dist_g$ is clearly the equator (the
geodesic connecting $x$ and $y$), which is the black arc on $T_{n}\Sb^2$; it is
different than the blue line. In fact, only lines in $T_n\Sb^2$ connecting the
origin and other points on $T_n \Sb^2$ correspond to geodesics on $\Sb^2$ for a
general metric. As a consequence, the measures $\mu_g$ and $\mu_E$ induced by
$\dist_g$ and $\dist_E$ are also different.

\begin{figure*}[htb!]
\centering
\includegraphics[width=0.8\textwidth]{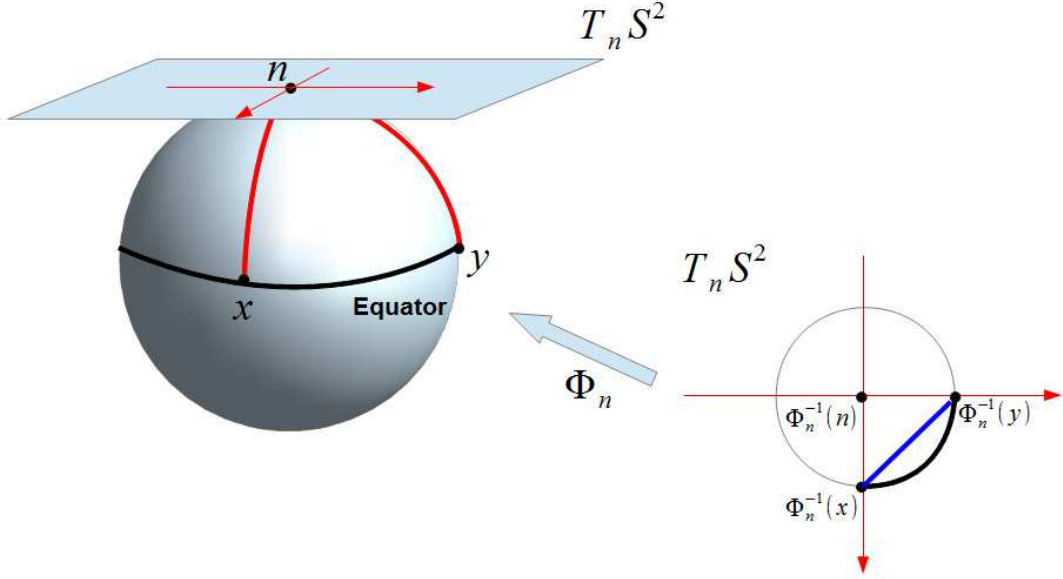}
\caption{Difference between the metrics $\dist_g$ and $\dist_E$ on $T_n \Sb^2$. The black arc in $T_n\Sb^2$ is a geodesic under the metric $g$. The blue segment in $T_n \Sb^2$ is a geodesic under the Euclidean metric $\dist_E$.}
\label{fig:spheregeodesic}
\end{figure*}

Given a submanifold $S\subset M$ (or a $\tau$-tubular neighborhood $S^\tau$ of
$S$), the metric tensor on $S$ (or $S^\tau$) inherited from $g$ induces a
measure $\mu_{gS}$ on $S$ (or $S^\tau$), which is called the uniform measure on
$S$ (or $S^\tau$). For simplicity we assume throughout most of the proof that
$\tau=0$ and thus mainly discuss the measure $\mu_{gS}$ on $S$. In
Section~\ref{sec:MainNoisy} we generalize the proof to the noisy case and thus
discuss $\mu_{gS}$ on $S^\tau$. The push-forward measure of $\mu_{gS}$ by
$\Phi_x^{-1}$ is a measure on $T_xM$, which is again denoted by $\mu_{gS}$. By
definition, the support of the push-forward measure $\mu_{gS}$ is
$\Phi_x^{-1}(S)$, which is a submanifold of $T_x M \equiv \R^D$. The Euclidean
metric $\dist_E$ similarly induces another measure $\mu_{ES}$, which is
supported on $\Phi_x^{-1}(S)$.

For a measure $\mu$ on $T_xM$ and a subset $H\subset T_xM$ of positive such
measure, the expected covariance matrix $\E_{\mu}\mb{C}_H$ is defined by
\begin{equation}
\label{eq:def_exp_cov}
\E_{\mu}\mb{C}_H=\frac{1}{\mu(H)}\int_{\mb{y}\in H}
\mb{y}\mb{y}^T\mu(d\mb{y})-\frac{1}{(\mu (H))^2}\int_{\mb{y}\in H}
\mb{y}\mu(d\mb{y})\cdot \int_{\mb{y}\in H} \mb{y}^T\mu(d\mb{y}).
\end{equation}
For the two compact submanifolds of the model, $S_1$ and $S_2$, we denote $S=
S_1\cup S_2$ and define the following two measures w.r.t.~S:
$\mu_{gS}=\mu_{gS_1}+\mu_{gS_2}$ and $\mu_{ES}=\mu_{ES_1}+\mu_{ES_2}$. The
covariance matrices w.r.t.~$\mu_{gS}$ and $\mu_{ES}$ are denoted by
$\E_{\mu_{gS}}\mb{C}_H$ and $\E_{\mu_{ES}}\mb{C}_H$, respectively. For
simplicity, when $H=B_x(\mb{0},r)\subset T_xM$, we denote them by
$\E_{\mu_{gS}}\mb{C}_x$ and $\E_{\mu_{ES}}\mb{C}_x$. For
$H=\Phi_z^{-1}(B(x,r))\subset T_zM$, we denote them by $\E_{\mu_{gS}}\mb{C}_x^z$
and $\E_{\mu_{ES}}\mb{C}_x^z$.

If a dataset $X\in M$ is given, let $\mb{C}_{x_0}$ denote the sample covariance
of the data $\Phi_{x_0}^{-1}(B(x_0,r)\cap X)$ on $T_{x_0} M$ and
$\mb{C}_{x_0}^z$ denote the sample covariance of the data
$\Phi_z^{-1}(B(x_0,r)\cap X)$ on $T_z M$. Let $\theta_{\min}(T_z S_1, T_z S_2)$
denote the minimal nonzero principal angle between the subspaces $T_z S_1, T_z
S_2\subset T_z M$\footnote{We only use $\theta_{\min}(T_z S_1, T_z S_2)$ when
  there is a nonzero principal angle. It is thus well-defined.} and let
\begin{equation}
\label{eq:def_theta0}
\displaystyle \theta_0(S_1, S_2) = \inf_{z\in S_1\cap S_2} \theta_{\min}(T_z S_1, T_z S_2).
\end{equation}
For $x\in S_1\cap S_2$, let $\theta_{\max}(T_x S_1, T_x S_2)$ denote the largest
principal angle between  $T_x S_1$ and $T_x S_2$ and let
\begin{equation}
\label{eq:def_theta_max}
\theta_{\max}(S_1, S_2)=\displaystyle \min_{x\in S_1\cap S_2}\theta_{\max}(T_x
S_1, T_x S_2).
\end{equation}

Recall that the notation
\[
Q_1(r) =Q_2(r) +\OO(r^n)
\]
means that there is a constant $C$ independent of $r$ such that
\begin{equation}\label{equ:O}
| Q_1(r) - Q_2(r) | \leq Cr^n.
\end{equation}
If $Q_1(r)$ and $Q_2(r)$ are matrices, then~\eqref{equ:O} applies to their
entries. If $x_i,x_j\in M$, we denote by $l'(x_i,x_j)$ the tangent vector of
$l(x_i,x_j)$ at $x_i$ (it was denoted by $\mb{v}_{ij}$ in
Section~\ref{sec:notation}). We denote the empirical geodesic angle between $x$
and $y$ by $\theta_{x,y}$ (where for data points $x_i$, $x_j$,
$\theta_{x_i,x_j}=\theta_{ij}$). Lastly, for a matrix $\mb{C}$,
$\lambda_k(\mb{C})$ stands for the $k$th largest eigenvalue of $\mb{C}$.

\subsection{A Generative Multi-Geodesic Model}\label{sec:model_recap}
We review in more details the generative model for two geodesic submanifolds
(see Section~\ref{sec:generativeMGM}). %The discussion about the noisy case is
                                %given at the end of the proof.
We first state the definition of geodesic submanifolds.

\begin{definition}\label{def:geodesic}
  For a Riemannian manifold $M$, a submanifold $S$ is called a geodesic
  submanifold if $\forall x,y\in S$, the shortest geodesic connecting $x$ and
  $y$ in $M$ is also contained in $S$.
\end{definition}

Let $S_1$, $S_2$ be two compact geodesic submanifolds of dimension $d$ in a
Riemannian manifold $(M,g)$ and recall that $S=S_1\cup S_2$.  Let $S_1^{\tau}$,
$S_2^{\tau}$ and $S^{\tau}$ denote $\tau$-tubular neighborhoods of $S_1$, $S_2$
and $S$ respectively. For example, $S_1^\tau=\{x \in M \ : \dist_g(x,S_1) \leq
\tau\}$, where $\dist_g(x,S_1):= \min_{y \in S_1} \dist_g(x,y)$.
%We normalize the measure $\mu_{gS}$ so that $\mu_{gS}(S^{\tau})=1$.
The dataset $X$ of size $N$ is i.i.d.~sampled from the normalized version of
$\mu_{gS}$ (by $\mu_{gS}(S^{\tau})$) on $S^{\tau}$. We recall the notation:
$X_1=S_1\cap X$ and $X_2=S_2\cap X$.  Fixing a point $x_i$, then
$\mb{x}_j^{(i)}$ is i.i.d.~sampled from the normalized push-forward $\mu_{gS}$
on $T_{x_i}M$.

\subsection{Local Concentration with High Probability}\label{sec:events}
We verify here the concentration of the local covariance matrices and the
existence of sufficiently large samples in local neighborhoods from the same
submanifold. We follow~\citet{LocalPCA}\footnote{For simplicity, we set the
  parameter $t$ of~\citet{LocalPCA} to be equal to $r$} and define on the
probability space $S^N$ (i.e., $(S_1\cup S_2)\times \cdots\times (S_1\cup
S_2)$), the following events $\Omega_1$ and $\Omega_2$:
\begin{equation}
\label{eq:def_omega1}
\Omega_1=\displaystyle\bigcap_{k=1}^2\{X=(x_1,\ldots, x_N)\in S^N: \#\{i :
x_i\in S_k\cap B(y,r/C_{\Omega})\}>nr^d/C_7, \forall y\in S_k\},
\end{equation}
\begin{equation}
\label{eq:def_omega2}
\Omega_2=\displaystyle \{X=(x_1,\ldots, x_N)\in S^N:\|\mb{C}_{x_i}-\E_{\mu_{gS}}
\mb{C}_{x_i}\|\leq r^3, \, i=1,\ldots,N\},
\end{equation}
where $C_{\Omega}$ and $C_7$ are specified in~\citet{LocalPCA} ($C_{\Omega}$
depends on $d$ and $\theta_0$ (defined in~\eqref{eq:def_theta0})
and $C_7$ depends on the covering number of $S$)
and $\mb{C}_{x_i}$ is the sample covariance of images
$\{\mb{x}_j^{(i)}\}_{j \in J(x_i, r)}$ on $T_{x_i}M$. We note that $\Omega_1$ is the
set of datasets of $N$ samples, where each dataset satisfies the following
condition: for any point in $S_i$ ($i=1,2$ is fixed), there are enough samples
that also belong to $S_i$ (their fraction is proportional to $r^d$). The set
$\Omega_2$ is the set of datasets of $N$ samples with sufficient concentration
of local covariance matrices. The following theorem of~\citet[page
35]{LocalPCA} ensures
that the event $\Omega=\Omega_1\cap \Omega_2$ is large.
It uses the constant $C_0 = 4d+2C_7$ and an absolute constant $C'_0$.

%Without causing ambiguity, a sample set $X$ satisfying the condition in the event $\Omega_i$ is denoted by $X\in \Omega_i$. In other words, $\Omega_1$ and $\Omega_2$ stand for the subsets of the probability space, whose elements (sample sets) satisfy the specified conditions respectively. Theorem~\ref{theorem:omega} proved in~\citet[page 35]{LocalPCA} is about the probability of the event $\Omega=\Omega_1\cap \Omega_2$.
\begin{theorem}\label{theorem:omega}
Let $\Omega=\Omega_1\cap \Omega_2$. Then,
\[
\P(\Omega^c)\leq C_0\cdot Ne^{-Nr^{d+2}/C'_0}.
\]
\end{theorem}
In view of this theorem, we assume in the rest of the proof that
\begin{equation}\label{equ:XInOmega}
X\in\Omega.
\end{equation}

\subsection{Ensuring Connectedness of $X_1^*$ and $X^*_2$}\label{sec:SameS}
The following proposition establishes WLOG the connectedness of the graph of the set $X^*_1$ (defined in~\eqref{equ:X^*}). It uses a constant $C_1$, which is clarified in the proof and depends on geometric properties of $S_1$ and $S_2$ and their angle of intersection.
\begin{proposition}\label{lemma:connect}
There exists a constant $C_1>1$ such that if
\begin{equation}\label{equ:constraint123}
 r< \frac{\min\left(\eta, \sigma_a, \sigma_d\right)}{C_1},
\end{equation}
then the graph with nodes at $X_1^*$ and edges given by $\mb{W}$ is connected.
%for $C'$ in Lemma~\ref{lemma:intercoord}, then $  X_1^*$ form a connected graph.
\end{proposition}

\subsubsection{Proof of Proposition~\ref{lemma:connect}}
Three different constants $C_8$, $C_9$ and $C_{10}$ appear in the proof. As clarified below, they depend on geometric properties of $S_1$ and $S_2$ and their angle of intersection. The constant $C_1$ is then determined by these constants as follows: $C_1=\max(\{C_i\}_{i=8}^{10})$.

The proof is divided into three parts. The first one shows that $\mathbf{1}_{\dim {T_{x_i}^E S}=\dim {T_{x_j}^E S }}=1$ for all $x_i,x_j$ in $  X_1^*$ if $r<\eta/C_8$. The second one shows that $\mathbf{1}_{(\theta_{ij}+\theta_{ji})<\sigma_a}=1$ for all $x_i,x_j$ in $  X_1^*$ if $\sigma_a \geq C_9 r$. The last one uses an argument of~\citet[page 38]{LocalPCA}. It claims that the graph with nodes at $X_1^*$ and weights given by the indicator function $\mb{1}_{\dist_g(x_i,x_j)<\sigma_d}$ is connected if $r\leq \sigma_d/C_{10}$.

\paragraph{\bf Part I:} We prove the following lemma, which clearly implies that $\mathbf{1}_{\dim {T_{x_i}^E S}=\dim {T_{x_j}^E S }}=1$ for $x_i,x_j\in X_1^*$.

\begin{lemma}\label{lemma:SameSdim}
There exists a constant $C_8>1$ such that if $x_0\in   X_1^*$, $r < \eta/C_8$ and $0<\eta<1$, then
\begin{equation}\label{equ:lemma:SameS6}
\dim {T_{x_0}^E S} =\dim {T_{x_0}S}.
\end{equation}
\end{lemma}
\begin{proof}
Recall that $\mb{C}_{x_0}$ denotes the sample covariance of the transformed data $\Phi_x^{-1}(X)\cap B_{x_0}(\mb{0},r)$.
We denote $H=B_{x_0}(\mb{0},r)\cap T_{x_0} S$ and note that
\begin{equation}\label{equ:lemma:SameS}
\begin{aligned}
&\E_{\mu_{gS}}\mb{C}_{x_0} =\frac{1}{\mu_{gS}(H)}\int_H \mb{yy}^T\mu_S(d\mb{y})-\frac{1}{(\mu_{gS}(H))^2}\int_H \mb{y}\mu_{gS}(d\mb{y})\cdot \int_H \mb{y}^T\mu_S(d\mb{y}) \\
&=\frac{1}{\mu_{ES}(H)}\int_H \mb{yy}^T\mu_{IS}(d\mb{y})-\frac{1}{(\mu_{ES}(H))^2}\int_H \mb{y}\mu_{ES}(d\mb{y})\cdot \int_H \mb{y}^T\mu_{IS}(d\mb{y})+\OO(r^4) \\
&=\E_{\mu_{ES}}\mb{C}_{x_0}+\OO(r^4).
\end{aligned}
\end{equation}
The first and third equalities of~\eqref{equ:lemma:SameS} follow from the definition of the expected covariance. The second equality of~\eqref{equ:lemma:SameS} follows from~\eqref{equ:chart} and the fact that $\|\mb{y}\| \leq r$. A slight generalization of Lemma~11 of~\citet{LocalPCA} implies that
\begin{equation}\label{equ:lemma:SameS2}
\E_{\mu_{ES}}\mb{C}_{x_0}=\frac{r^2}{d+2}\mb{P}_{T_{x_0} S},
\end{equation}
where $\mb{P}_{T_{x_0} S}$ is the orthogonal projector onto $T_{x_0} S$ in $T_{x_0}M$. Equation~\eqref{equ:lemma:SameS} and ~\eqref{equ:lemma:SameS2} imply that
\begin{equation}\label{equ:lemma:SameS3}
\|\E_{\mu_{gS}}\mb{C}_{x_0}-\frac{r^2}{d+2}\mb{P}_{T_{x_0} S}\|<C_Sr^4,
\end{equation}
where $C_S>0$ is a constant depending on the Riemannian metric $g$ (arising due to~\eqref{equ:chart}).
Using this constant $C_S$, we define
\begin{equation}
C_8=2(d+2)(C_S+1).
\label{eq:C5}
\end{equation}
We note that $C_8>1$. Combining this observation with the following two assumptions: $r<\eta/C_8$ and $0<\eta<1$, we conclude that $r<1$.

Combining the triangle inequality, \eqref{equ:XInOmega}, \eqref{equ:lemma:SameS3} and the fact that $r<1$, we conclude that
\begin{equation}\label{equ:lemma:SameS4}
\|\mb{C}_{x_0}-\frac{r^2}{d+2}\mb{P}_{T_{x_0}S}\| \leq \|\mb{C}_{x_0}-\E_{\mu_{gS}}\mb{C}_{x_0}\| + \|\E_{\mu_{gS}}\mb{C}_{x_0}-\frac{r^2}{d+2}\mb{P}_{T_{x_0}S}\| <r^3+C_Sr^4 \leq (C_S+1)r^3.
\end{equation}
The application of both Weyl's inequality~\citep{MR1061154} and~\eqref{equ:lemma:SameS4} results in the following lower bound of $\lambda_1(\mb{C}_{x_0})$ and upper bound of $\lambda_{d+1}(\mb{C}_{x_0})$:
\begin{equation}\label{equ:constraint0}
{\lambda_{d+1}(\mb{C}_{x_0})}< {(C_S+1)r^3} \ \text{ and } \ {\lambda_1(\mb{C}_{x_0})}> \frac{r^2}{d+2}-(C_S+1)r^3.
\end{equation}
It follows from~\eqref{eq:C5}, \eqref{equ:constraint0} and elementary algebraic manipulations that
\begin{align}\label{equ:constraint1}
&\frac{\lambda_{d+1}(\mb{C}_{x_0})}{\lambda_1(\mb{C}_{x_0})}< \frac{(C_S+1)r^3}{\frac{r^2}{d+2}-(C_S+1)r^3} = \frac{C_S+1}{1/(r(d+2)) -(C_S+1)}\\
& < \frac{C_S+1}{C_8/(d\eta+2\eta)-(C_S+1)}=\frac{\eta}{2-\eta}<\eta. \nonumber
\end{align}
Equation~\eqref{equ:lemma:SameS6} thus follows from~\eqref{equ:constraint1} and the thresholding of eigenvalues by $\eta\|\mb{C}_{x_0}\|$ in Algorithm~1.
\end{proof}

\paragraph{\bf Part II:} Next, we prove that $\mathbf{1}_{(\theta_{ij}+\theta_{ji})<\sigma_a}=1$ if $\sigma_a \geq C_9 r$.

\begin{lemma}\label{lemma:SameSang}
There exists a constant $C_9>1$ such that if $x, y\in   X_1^*$ and
\begin{equation}\label{equ:constraint2}
\sigma_a \geq C_9r,
\end{equation}
then $\mathbf{1}_{(\theta_{x,y}+\theta_{y,x})<\sigma_a}=1$.
\end{lemma}

\begin{proof}
We define
\begin{equation}
\label{eq:C6}
C_9=\sqrt{2}(C_S+1)(d+2)\pi,
\end{equation}
where $C_S$ is the constant introduced in~\eqref{equ:lemma:SameS3}.
We show that for $x,y\in X_1^*$:
\begin{equation}\label{equ:lemma:SamSang1}
\theta_{x,y}<\frac{C_9}{2} r,
\end{equation}
which immediately implies the lemma.

In order to prove~\eqref{equ:lemma:SamSang1}, we first apply the Davis-Kahan Theorem~\citep{1970SJNA....7....1D} and~\eqref{equ:lemma:SameS6} and then apply~\eqref{equ:lemma:SameS4} to obtain the following bound on the distance between the subspaces $T_{x}^E S$ and $T_{x} S$ (which are spanned by the top $d$ eigenvectors of $\mb{C}_{x}$ and $\frac{r^2}{d+2}\mb{P}_{T_{x} S}$,  respectively; this observation uses~\eqref{equ:lemma:SameS6}):
\begin{equation}\label{equ:lemma:SameS5}
\begin{aligned}
\|\mb{P}_{T_{x}^E S}-\mb{P}_{T_{x} S}\|<\frac{\sqrt{2}\|\mb{C}_{x}-\frac{r^2}{d+2}\mb{P}_{T_{x} S}\|}{\frac{r^2}{d+2}}<\sqrt{2}(C_S+1)(d+2)r.
\end{aligned}
\end{equation}
We remark that in applying the Davis-Kahan Theorem we made use of the following basic calculation of $\Delta$,
the $d$th spectral gap of $\frac{r^2}{d+2}\mb{P}_{T_{x} S}$:
$\Delta=\lambda_d(\frac{r^2}{d+2}\mb{P}_{T_{x} S})-\lambda_{d+1}(\frac{r^2}{d+2}\mb{P}_{T_{x} S})=\frac{r^2}{d+2}$.

Next, we recall that $\theta_{\max}(T_{x}^E S,T_{x} S)$ denotes the largest principal angle between $T_{x}^E S$ and $T_{x} S$. We note that Lemma~15 of~\citet{LocalPCA} (whose application
requires \eqref{equ:lemma:SameS6}), \eqref{equ:lemma:SameS5}, \eqref{eq:C6} and Jordan's inequality (lower bounding the $\sin$ function by $2/\pi$) imply that
\begin{equation}\label{equ:lemma:SamSang2}
\displaystyle \theta_{\max}(T_{x}^E S,T_{x} S)=\sin^{-1}(\|\mb{P}_{T_{x}^E S}-\mb{P}_{T_{x} S}\|)<\sin^{-1}(\sqrt{2}(C_S+1)(d+2)r)<\frac{C_9}{2} r.
\end{equation}

Since $\theta_{x,y}$ is the angle between $l'(x,y)\in T_{x} S$ and $T_{x}^E S$, \eqref{equ:lemma:SamSang1} follows from~\eqref{equ:lemma:SamSang2}. We can then conclude that if $\sigma_a \geq C_9 r$, then $\mathbf{1}_{(\theta_{ij}+\theta_{ji})<\sigma_a}=1$ for all $x_i,x_j$ in $  X_1^*$.

\end{proof}

\paragraph{\bf Part III:} By the construction of the affinity matrix $\mb{W}$ and Lemmata~\ref{lemma:SameSdim} and \ref{lemma:SameSang}, the connectivity between points $x_i,x_j\in   X_1^*$ is solely determined by the indicator function $\mathbf{1}_{\dist_g(x_i,x_j)<\sigma_d}$. It is obvious that if $\sigma_d > 4r$ then the graph with nodes in $X_1$ and weights $\mb{1}_{\dist_g(x_i,x_j)< \sigma_d}$ is connected (this can be done by finite covering of $S_1$ with balls of radius $r$). It follows from~\citet[pages 38-39]{LocalPCA} that the graph with nodes in $X_1^*$ is also connected if
\begin{equation}\label{equ:constraint3}
\displaystyle r\leq \sigma_d/C_{10}.
\end{equation}
%Similarly, $  X_2^*$ also forms a connected graph. Now, points in $I_{\times}$ is analyzed. Lemma~\ref{lemma:Intersection} shows that if points are near intersections, the estimated tangent dimensions are larger than $d$. Thus there is no edge between points in $I_{\times}$ and points in $ X^*$.

There is one component in the argument of~\citet[page 38]{LocalPCA} that requires careful adaptation to the Riemannian case. It is related to the determination of the constant $C_{10}$. This constant is set to be $(3C'+9)^{-1}$ (see~\citet[page 39]{LocalPCA}). In the Euclidean case, $C'$ is guaranteed by Lemma~18 of~\citet{LocalPCA}. The adaptation of this Lemma to the Riemannian case can be stated in the following lemma (it uses $\theta_0$, which was defined in~\eqref{eq:def_theta0}).
%Lemma~\ref{lemma:intercoord} is similar to Lemma~18 in~\citet{LocalPCA} and it follows from similar arguments.
\begin{lemma}\label{lemma:intercoord}
Let $(M,g)$ be a Riemannian manifold and $S_1$, $S_2$ be two compact geodesic submanifolds of dimension $d$ such that $\theta_0(S_1,S_2)>0$. Then there is a constant $C'$ such that
\begin{equation*}
\dist_g(x,S_1\cap S_2)\leq C' \max\{\dist_g(x,S_1), \dist_g(x,S_2)\}\quad \forall x\in S_1\cup S_2.
\end{equation*}
\end{lemma}
We prove Lemma~\ref{lemma:intercoord} in Appendix~\ref{sec:lemma:intercoord:proof}. The proof implies that $C'$ is determined by the geometric properties of $S_1$ and $S_2$ and the angle $\theta_0(S_1,S_2)$.

\subsection{Disconnectedness Between $\hat{X}^c$ and $X^*_1$ (or $X^*_2$)}

We show here that the points in $\hat{X}^c$ (where $\hat{X}$ is defined in~\eqref{equ:Xhat}) are not connected to the points of $X^*$. In Section~\ref{sec:SameS}, we showed that the estimated dimensions of local neighborhoods of points in $X^*$ equal $d$.
In this section, we show that the estimated dimensions of local neighborhoods of points in $\hat{X}^c$ are larger than $d$. Since $\mathbf{1}_{\dim {T_{x_i}^E S}=\dim {T_{x_j}^E S }}$ is a multiplicative term of $\mb{W}$, we conclude that $\hat{X}^c$ is disconnected from $X^*$.
The following main proposition of this section implies that WLOG the estimated tangent dimension at $\hat{X}^c\cap X_1$ is at least $d+1$ (it uses the angle $\theta_{\max}(S_1, S_2)$
defined in~\eqref{eq:def_theta_max}).
\begin{proposition}\label{lemma:Intersection}
There exists a constant $C_2>1$ depending only on $d$ and $\theta_{\max}(S_1, S_2)$ such that if $r < \eta$,
\begin{equation}
\label{eq:r_eta_cond}
\eta<C_2^{-\frac{d+2}{2}},
\end{equation}
\begin{equation}
\label{eq:def_delta}
\delta := r \sqrt{1-C_2\eta^{\frac{2}{d+2}}}
\end{equation}
and
\begin{equation}
\label{eq:bound_by_delta}
x \in \hat{X}^c \cap X_1,
 \text{ that is, }
\dist_g(x, S_2) < \delta,
\end{equation}
then
\begin{equation*}
\frac{\lambda_{d+1}(\mb{C}_x)}{\lambda_1(\mb{C}_x)}>\eta.
\end{equation*}
\end{proposition}

\begin{proof}
Let us first sketch the idea of the proof.
It is easier to estimate the local covariance matrices when the two manifolds are subspaces (see Lemma~21 of~\citet{LocalPCA}). However, for $x\in \hat{X}^c\cap X_1$, the logarithm map of $S$ into $T_x M$ does not result in two subspaces (see Figure~\ref{fig:transitionmap}). On the other hand, for $z$, the projection of $x$ onto $S_1\cap S_2$, the logarithm map of $S$ into $T_z M$ results in two subspaces, where the local covariance can be estimated more easily. Some difficulties arise due to the application of the logarithm map and the change of tangent spaces. In particular, the ball $B(x,r)$ becomes irregular in the domain $T_z M$.

We recall that $\mu_{gS}=\mu_{gS_1}+\mu_{gS_2}$ and $\mu_{ES}=\mu_{ES_1}+\mu_{ES_2}$. We arbitrarily fix $x_0\in S_1$ such that $\dist_g(x_0, S_2)<r$. We note that Lemma~\ref{lemma:intercoord} implies that
\begin{equation}\label{equ:lemma:Intersection1}
\dist_g(x_0,S_1\cap S_2)\leq C r.
\end{equation}
Let
\[
\displaystyle z=\argmin_{y\in S_1\cap S_2} \dist_g(x_0,y),
\]
where if argmin is not uniquely defined, then $z$ is arbitrarily chosen among all minimizers.
It follows from~\eqref{equ:lemma:Intersection1} that $\dist_g(x_0,z)\leq Cr$ and from this and the triangle inequality, it follows that
\begin{equation}\label{equ:inclusion}
B(x_0,r)\subset B(z,(C+1)r).
\end{equation}
Recall that $\Phi_{x_0}$ and $\Phi_z$ denote the normal coordinate charts around $x_0$ and $z$ respectively (see Figure~\ref{fig:transitionmap}); it is sufficient to restrict them to $B(x_0,r)$ and $B(z,(C+1)r)$ respectively. When using the chart $\Phi_z$, $S_1$ and $S_2$ correspond to two subspaces in $T_z M$, which we denote by $L_1$ and $L_2$ respectively.
On the other hand, when using the chart $\Phi_{x_0}$, $S_2$ corresponds to a manifold in $T_{x_0} M$, whereas $S_1$ still corresponds to a subspace.
\begin{figure*}[htb!]
\centering
\includegraphics[width=0.8\textwidth]{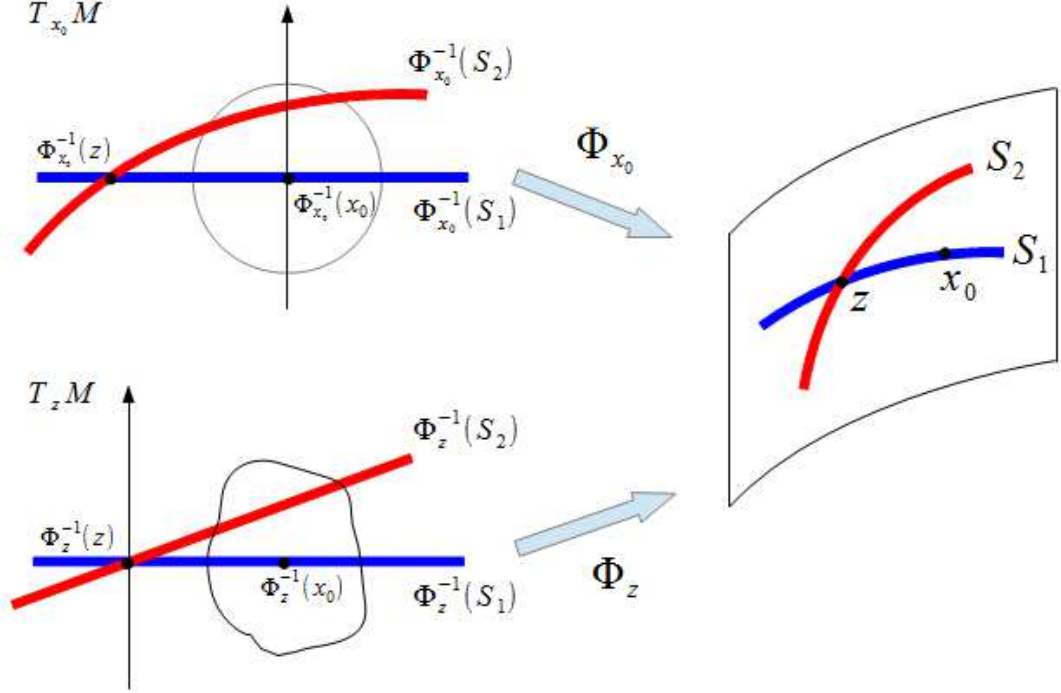}
\caption{The transition map between normal coordinates of $T_{x_0}M$ and $T_zM$. Notice the regular ball $B_{x_0}(\mb{0},r)$ in $T_{x_0} M$ is mapped to the irregular region in $T_z M$ because the exponential maps $\Phi_{x_0}$ and $\Phi_z$ are nonlinear.}
\label{fig:transitionmap}
\end{figure*}
It follows from~\eqref{equ:inclusion} and the invertibility of $\Phi_z$ that the composition map $\phi=\Phi_z^{-1}\circ\Phi_{x_0}$ embeds $B_{x_0}(\mb{0},r)$ into $B_z(\mb{0},(C+1)r)$ as shown in Figure~\ref{fig:transitionmap}. Recall that $\mb{C}_{x_0}$ denotes the sample covariance of the data $\Phi_{x_0}^{-1}(B(x_0,r)\cap X)$ in $T_{x_0} M$ and $\mb{C}_{x_0}^z$ denotes the sample covariance of the data $\Phi_z^{-1}(B(x_0,r)\cap X)$ in $T_z M$. Using the notation $O(D)$ for the set of orthogonal $D\times D$ matrices, we claim that
\begin{equation}\label{equ:lemma:Intersection2}
\displaystyle \exists \mb{R}\in O(D) \ \text{ s.t. } \ \mb{R} \left(\E_{\mu_{gS}}\mb{C}_{x_0}\right) \mb{R}^T=\E_{\mu_{gS}}\mb{C}_{x_0}^z+\OO(r^3).
\end{equation}
The technical proof of~\eqref{equ:lemma:Intersection2} is in Appendix~\ref{sec:lemma:Intersection2:proof}.

We estimate $\E_{\mu_{gS}}\mb{C}_{x_0}^z$ as follows. Let $H=\Phi_z^{-1}(B(x_0,r)\cap (S_1\cup S_2))$ and $H'=B_I(\Phi_z^{-1}(x_0),r)\cap (L_1\cup L_2)$ (see Figure~\ref{fig:changedomain}), where $B_I(\Phi_z^{-1}(x_0),r)$ is the $r$-ball with center $\Phi_z^{-1}(x_0)$ in $T_z M$, which uses the Euclidean distance $\dist_E$.

\begin{figure*}[htb!]
\centering
\includegraphics[width=0.8\textwidth]{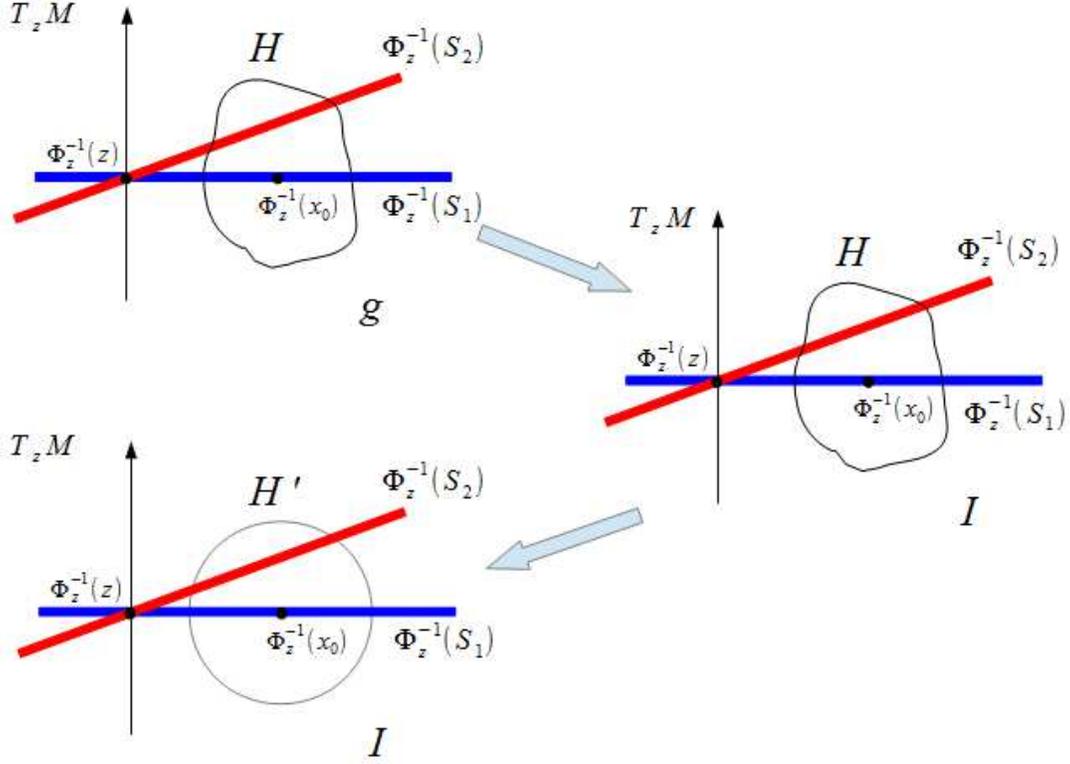}
\caption{Change of domain and metric}
\label{fig:changedomain}
\end{figure*}

The rest of the proof requires the following two technical observations
\begin{equation}\label{equ:interarea}
\mu_{ES}((H\setminus H')\cup (H'\setminus H))=\OO(r)\mu_{ES}(H).
\end{equation}
and
\begin{align}\label{equ:lemma:Intersection3}
\E_{\mu_{gS}}\mb{C}_{x_0}^z =\E_{\mu_{ES}}\mb{C}_{x_0}^z+\OO(r^3)=\E_{\mu_{ES}}\mb{C}_{H'}+\OO(r^3).
\end{align}
We prove~\eqref{equ:interarea} in Appendix~\ref{sec:interarea}.
The first equality of~\eqref{equ:lemma:Intersection3} follows from the definition of the expected covariance (see~\eqref{eq:def_exp_cov}), \eqref{equ:chart} and the fact that $\|\mb{y}\| \leq (C+1)r$. The second equality of~\eqref{equ:lemma:Intersection3}  follows from the definition of the expected covariance (see~\eqref{eq:def_exp_cov}), \eqref{equ:interarea} and the fact that $\|\mb{y}\| \leq (C+1)r$.

It follows from~\eqref{eq:def_omega2}, \eqref{equ:XInOmega}, \eqref{equ:lemma:Intersection2}, \eqref{equ:lemma:Intersection3} and the triangle inequality that
\begin{align}
\| \mb{R}\mb{C}_{x_0} \mb{R}^T - \E_{\mu_{ES}}\mb{C}_{H'}\|
\leq
&\| \mb{R}\mb{C}_{x_0} \mb{R}^T - \mb{R}\E_{\mu_{gS}} \mb{C}_{x_0} \mb{R}^T \|
+
\| \mb{R}\E_{\mu_{gS}} \mb{C}_{x_0} \mb{R}^T - \E_{\mu_{ES}}\mb{C}_{x_0}^z \| +
\nonumber \\
\label{equ:IntersectionAdapt}
&\| \E_{\mu_{ES}}\mb{C}_{x_0}^z - \E_{\mu_{ES}}\mb{C}_{H'} \|
\leq r^3 + \OO(r^3)+ \OO(r^3) \leq C'_S r^3
\end{align}
for a constant $C'_S>0$.

The combination of~\eqref{equ:IntersectionAdapt}, Weyl's inequality~\citep{MR1061154} for $\mb{R}(\mb{C}_{x_0})\mb{R}^T$ and $\E_{\mu_{ES}}\mb{C}_{H'}$, and the fact that $\mb{R}\mb{C}_{x_0} \mb{R}^T$ and $\mb{C}_{x_0}$ have the same eigenvalues implies that
\begin{equation}
\label{eq:Weyl_imply}
\lambda_{d+1}(\mb{C}_{x_0})\geq \lambda_{d+1}(\E_{\mu_{ES}}\mb{C}_{H'})-C'_S r^3, \quad \lambda_1(\mb{C}_{x_0})\leq \lambda_1(\E_{\mu_{ES}}\mb{C}_{H'})+C'_Sr^3.
\end{equation}
 Notice that $\theta_{\max}(S_1,S_2)\leq \theta_{\max}(L_1,L_2)$ by definition. Applying \eqref{eq:Weyl_imply} and Lemma~21 of~\citet{LocalPCA} to $\E_{\mu_{ES}}\mb{C}_{H'}$, where  $\theta_{\max}(L_1,L_2)$ is replaced by $\theta_{\max}(S_1,S_2)$ and proper scaling is used, results in
\begin{align}
\nonumber
\displaystyle \frac{\lambda_{d+1}(\mb{C}_{x_0})}{\lambda_1(\mb{C}_{x_0})} &\geq \frac{\frac{1}{8(d+2)}(1-\cos\theta_{\max}(S_1,S_2))^2(1-(\dist_g(x_0, S_2)/r)^2)_+^{d/2+1}-C'_S r^3}{1/(d+2)+(\dist_g(x_0, S_2)/r)(1-(\dist_g(x_0, S_2)/r)^2)_+^{d/2}+C'_S r^3}\\
& \geq \frac{\frac{1}{8(d+2)}(1-\cos\theta_{\max}(S_1,S_2))^2(1-(\dist_g(x_0, S_2)/r)^2)^{d/2+1}-C'_S r^3}{1/(d+2)+1+C'_S r^3}.
\label{eq:lower_bound_rat_beta}
\end{align}
We remark that the second inequality of~\eqref{eq:lower_bound_rat_beta} is derived by applying the bound: $\dist_g(x, S_2)<r$. In order to satisfy $\frac{\lambda_{d+1}(\mb{C}_{x_0})}{\lambda_1(\mb{C}_{x_0})}>\eta$, we require that
\begin{equation}
\label{eq:require_eta}
 \frac{\frac{1}{8(d+2)}(1-\cos\theta_{\max}(S_1,S_2))^2(1-(\dist_g(x_0, S_2)/r)^2)^{d/2+1}-C'_S r^3}{1/(d+2)+1+C'_S r^3}> \eta.
\end{equation}
Since $r<\eta<1$, we replace $C'_S r^3$ with $C'_S \eta$ in the numerator of~\eqref{eq:require_eta} and $C'_S r^3$ with $C'_S$ in the denominator of~\eqref{eq:require_eta} and slightly simplify the inequality to obtain the following stronger requirement:
\begin{equation}
\label{eq:require_eta1}
 {\frac{(1-\cos\theta_{\max}(S_1,S_2))^2}{8(d+2)}(1-(\dist_g(x_0, S_2)/r)^2)^{d/2+1}}
 >
 \left( \frac{1}{d+2}+1+2C'_S \right) \eta.
\end{equation}

Finally, setting
\begin{equation}
\label{eq:def_C2}
C_2=\left(\frac{8d+24+16(d+2)C'_S}{(1-\cos\theta_{\max}(S_1,S_2))^2}\right)^{\frac{2}{d+2}}
\end{equation}
we can rewrite~\eqref{eq:require_eta1} as follows
\begin{equation}
\label{eq:conclude_prop8}
(1-(\dist_g(x_0, S_2)/r)^2)^{\frac{d+2}{2}}> C_2^{\frac{d+2}{2}} \eta.
\end{equation}
We immediately conclude~\eqref{eq:conclude_prop8} (and consequently the lemma) from~\eqref{eq:def_delta} and \eqref{eq:bound_by_delta}.
\end{proof}

We end this section with an immediate corollary of Proposition~\ref{lemma:Intersection}, which is crucial in order to follow the proof.
\begin{corollary}\label{cor:intersection}
The following relations are satisfied:
\[
X^*_1\subset \hat{X}_1 \quad \text{and} \quad X^*_2\subset \hat{X}_2.
\]
\end{corollary}

\begin{proof}
It follows from Lemma~\ref{lemma:SameSdim} and Proposition~\ref{lemma:Intersection} that $X^*_1\cap \hat{X}^c=\emptyset$. Therefore $X_1^*\subset \hat{X}_1$. Similarly, $X_2^*\subset \hat{X}_2$.
\end{proof}

\subsection{The Disconnectedness of $X_1^*$ and $X_2^*$}\label{sec:diffangle}
We show here that the graphs with nodes at $X^*_1$ and $X^*_2$ are disconnected. The idea is to show that the function $\mathbf{1}_{\dist_g(x_i,x_j)<\sigma_d}\mathbf{1}_{\theta_{ij}+\theta_{ji}<\sigma_a}$ (and thus the weight $\mb{W}$) is zero between two points in $\hat{X}_1\supset X^*_1$ and $\hat{X}_2\supset X^*_2$ for appropriate choice of constants. This and Proposition~\ref{lemma:Intersection} imply that the graphs associated with $X^*_1$ and $X^*_2$ are disconnected. We first establish a lower bound on the empirical geodesic angle in Lemma~\ref{lemma:DiffS} and then conclude that there is no direct connection between the sets $\hat{X}_1$ and $\hat{X}_2$ in Corollary~\ref{cor:disconnectedness}.

\begin{lemma}\label{lemma:DiffS}
There exist constants $C_3 >0$ and $C_4>0$ such that if
$x_1\in \hat{X}_1$, $x_2\in \hat{X}_2$,
\begin{equation}\label{equ:DiffS1}
\dist_g(x_1,x_2)<\sigma_d,
\end{equation}
\begin{equation}
\label{eq:sigma_d_low_bound}
\text{and } \ \sigma_d <{C_4}^{-\frac{1}{2}},
\end{equation}
then the angle between the estimated tangent subspace $T_{x_1}^{E}S_1$ and the line segment $l_{12}^{(1)}$, which connects the origin and $\mb{x}_2^{(1)}$ (the image of $x_2$ by $\log_{x_1}$) in $T_{x_1} M$ is bounded below as follows:
\begin{equation}\label{equ:DiffS2}
\angle(l_{12}^{(1)},T_{x_1}^{E}S_1)>\min(\sin^{-1}(\delta/2\sigma_d)-C_3\eta^{d/(d+2)}-C_3 r, \pi/6).
\end{equation}
\end{lemma}

\begin{proof}
The proof develops various geometric estimates that eventually conclude~\eqref{equ:DiffS2}. Let
\[
\displaystyle \mb{x}_3=\argmin_{\mb{x}\in T_{x_1}S_1}\dist_g(\mb{x},\mb{x}_2^{(1)})\quad \text{and}\quad
\displaystyle \mb{x}_4=\argmin_{\mb{x}\in T_{x_1}S_1}\dist_E(\mb{x},\mb{x}_2^{(1)}),
\]
where $\dist_E$ is defined with respect to the normal coordinate chart in $T_x M$ (see Figure~\ref{fig:DiffS}).
\begin{figure*}[htb!]
\centering
\includegraphics[width=.50\textwidth]{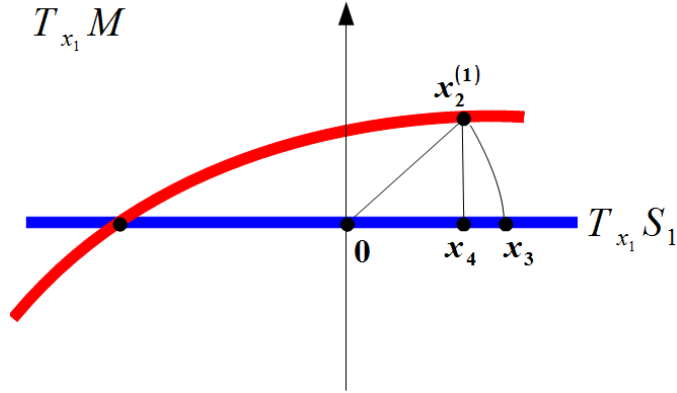}
\caption{The normal coordinate chart at $x_1$}
\label{fig:DiffS}
\end{figure*}

We note that by definition $\mb{x}_4$ is the projection of $\mb{x}_2^{(1)}$ onto $T_{x_1} S_1$ and thus
\begin{equation}\label{equ:pre_DiffS3}
\dist_E(\mb{x}_4,\mb{0})<\dist_g(\mb{x}_2^{(1)},\mb{0}).
\end{equation}
Combining~\eqref{equ:pre_DiffS3} with the fact that $\dist_E$ and $\dist_g$ are
the same on lines through the origin in $T_{x_1} M$ and then
applying~\eqref{equ:DiffS1}, we obtain that
\begin{equation}\label{equ:DiffS3}
\dist_g(\mb{x}_4,\mb{0})<\dist_g(\mb{x}_2^{(1)},\mb{0})<\sigma_d.
\end{equation}
Furthermore, combining the following two facts: $\mb{x}_3$ is a minimizer of
$\dist_g(\pmb{\cdot}, \mb{x}_2^{(1)})\in T_{x_1} S_1$ and
$x_2\in \hat{X}_2$, we obtain that
\begin{equation}\label{equ:DiffS4}
\delta\leq \dist_g(x_2,
\Phi_{x_1}(\mb{x}_3))=\dist_g(\mb{x}_2^{(1)},\mb{x}_3)<\dist_g(\mb{x}_2^{(1)},\mb{x}_4).
\end{equation}

We prove in Appendix~\ref{sec:DistInNeighbor} that there exists a constant $C_4>0$, which depends only on the Riemannian manifold $M$, such that
\begin{equation}\label{equ:proof:DiffS5:1}
\forall R>0, \ \mb{x},\mb{y}\in B_{x_1}(\mb{0},R), \quad |\dist_E(\mb{x},\mb{y})-\dist_g(\mb{x},\mb{y})|<C_4 R^2\dist_E(\mb{x},\mb{y}).
\end{equation}
Applying~\eqref{equ:proof:DiffS5:1} (with $R= \sigma_d$) first and~\eqref{equ:DiffS4} next we obtain that
\begin{align}
\label{equ:DiffS5}
\dist_E(\mb{x}_2^{(1)},\mb{x}_4) &>\dist_g(\mb{x}_2^{(1)},\mb{x}_4)-C_4\sigma_d^2\dist_E(\mb{x}_2^{(1)},\mb{x}_4) \nonumber\\
&>\delta-C_4\sigma_d^2\dist_E(\mb{x}_2^{(1)},\mb{x}_4)
\end{align}
and consequently
\begin{equation}
\label{eq:ima_show_tell}
\dist_E(\mb{x}_2^{(1)},\mb{x}_4)>\frac{\delta}{1+C_4\sigma_d^2}.
\end{equation}
It follows from~\eqref{eq:sigma_d_low_bound} and~\eqref{eq:ima_show_tell} that
\begin{equation}\label{equ:DiffS6}
\displaystyle\sin(\angle(l_{12}^{(1)},T_{x_1}S_1))=\frac{\dist_E(\mb{x}_2^{(1)},\mb{x}_4)}{\dist_E(\mb{x}_2^{(1)},\mb{0})}>\frac{\delta}{\sigma_d+C_4\sigma_d^3}>\delta/2\sigma_d.
\end{equation}

Our proof concludes from~\eqref{equ:DiffS6} and the following two claims:
\begin{equation}\label{equ:DiffS7}
\sin(\theta_{\max}(T_{x_1}^ES_1,T_{x_1}S_1))\leq C'_3\eta^{d/(d+2)}+C'_3r
\end{equation}
and
\begin{equation}\label{equ:DiffS8}
\angle(l_{12}^{(1)},T_{x_1}^ES_1) \geq \min(\angle(l_{12}^{(1)},T_{x_1}S_1) - \frac{2\pi\sqrt{d}}{3}\sin(\theta_{\max}(T_{x_1}^ES_1,T_{x_1}S_1)), \pi/6).
\end{equation}

Inequalities~\eqref{equ:DiffS7} and~\eqref{equ:DiffS8} are verified in Appendices~\ref{proof:equ:DiffS7} and~\ref{proof:equ:DiffS8} respectively, where we also carefully analyze how the constant $C'_3$ depends on the underlying Riemannian manifold (see~\eqref{eq:def_c3prime}). Combining~\eqref{equ:DiffS6}, \eqref{equ:DiffS7} and~\eqref{equ:DiffS8}, we conclude~\eqref{equ:DiffS2} by letting $\displaystyle C_3=\frac{2\pi\sqrt{d}}{3}C'_3$.

\end{proof}

The desired disconnectedness of $X^*_1$ and $X^*_2$ immediately follows from
Lemma~\ref{lemma:DiffS} in the following way:
\begin{corollary}\label{cor:disconnectedness}
The graphs with nodes at $X^*_1$ and $X^*_2$ respectively and weights in
$\mb{W}$ are disconnected if the angle threshold $\sigma_a$ is chosen such that
\begin{equation}\label{equ:constraint4}
\sigma_a<\min(\sin^{-1}(\delta/2\sigma_d)-C_3\eta^{d/(d+2)}-C_3r, \pi/6)
\end{equation}
and the distance threshold $\sigma_d$ satisfies~\eqref{eq:sigma_d_low_bound}.
\end{corollary}

\begin{proof}
  When $\sigma_a$ and $\sigma_d$ satisfy~\eqref{equ:constraint4}
  and~\eqref{eq:sigma_d_low_bound} respectively, Lemma~\ref{lemma:DiffS} implies
  that if $x_i\in \hat{X}_1$ and $x_j\in \hat{X}_2$, then
  $\mathbf{1}_{\dist_g(x_i,x_j)<\sigma_d}\mathbf{1}_{\theta_{ij}+\theta_{ji}<\sigma_a}=0$. In
  other words, there is no direct connection between $X^*_1$ and $X^*_2$ through
  $\hat{X}$. On the other hand, Lemma~\ref{lemma:SameSdim} and
  Proposition~\ref{lemma:Intersection} imply that $X^*_1$ and $X^*_2$ cannot be
  connected through points in $\hat{X}^c$ (since points in $X^*$ and $\hat{X}^c$
  have different local estimated dimensions).  We thus conclude that $X^*_1$ and
  $X^*_2$ are disconnected.
\end{proof}

\subsection{Conclusion of Theorem~\ref{theorem:all} for the Noiseless Multi-Geodesic Model}\label{sec:MainNoiseFree}

Due to Theorem~\ref{theorem:omega} we replace $X$ with $X \cap \Omega$ and obtain a statement for $X$ with probability at least $1-C_0\cdot Ne^{-Nr^{d+2}/C'_0}$.
Proposition~\ref{lemma:connect} and Corollary~\ref{cor:disconnectedness} imply that (with probability at least $1-C_0\cdot Ne^{-Nr^{d+2}/C'_0}$) $X^*$ has two connected components. They require that the parameters of TGCT satisfy~\eqref{equ:constraint123}, \eqref{eq:sigma_d_low_bound} and~\eqref{equ:constraint4}. Additional requirement is specified in~\eqref{eq:r_eta_cond} (in Proposition~\ref{lemma:Intersection} which implies Corollary~\ref{cor:disconnectedness}). We also note that the requirement $r<\eta<1$, which also appears in some of the auxiliary lemmata, follows from~\eqref{equ:constraint123}, \eqref{eq:r_eta_cond} and the fact that $C_1>1$ and $C_2>1$. These requirements, i.e., \eqref{equ:constraint123}, \eqref{eq:r_eta_cond}, \eqref{eq:sigma_d_low_bound} and~\eqref{equ:constraint4}, are sufficient and equivalent to~\eqref{eq:constant_all_requirements} when $\tau=0$.

Next, we explain why one can choose parameters that satisfy these requirements at the end of this section.
The only problem is to make sure that the last inequality of~\eqref{eq:constant_all_requirements} (equivalently, \eqref{equ:constraint4}) is satisfied.
Given a sufficiently small $r>0$ satisfying~\eqref{equ:constraint123}, we let $\sigma_d = \alpha r$ for some fixed $\alpha>0$.
The RHS of~\eqref{equ:constraint4} tends to $\min(\sin^{-1}(\frac{1}{2 \alpha}), \pi/6)$ as $r$ and $\eta$ approach zero. We note that the lower bound of $\sigma_a$ is $C_1 r$. Therefore, if $r$ and $\eta$ are sufficiently small so that $\min(\sin^{-1}({1}/{2 \alpha}), \pi/6)/2$ is lower than the RHS of~\eqref{equ:constraint4} and $C_1 r<\min(\sin^{-1}({1}/{2 \alpha}), \pi/6)/2$, then $\sigma_a$ can be chosen from the interval $[C_1r, \min(\sin^{-1}({1}/{2 \alpha}), \pi/6)/2]$.

In order to conclude the proof in this case we upper bound the expected portion of points $\# X^{*c} / \# X$, where $\# X^{*c}$ and $\# X$ denote he cardinality of $X^{*c}$ and $X$ respectively. For this purpose we use the set $X_{S_1\cap S_2} \supset X^{*c}$, which was defined in~\eqref{eq:s1caps2} in the following way:
\begin{align}\label{equ:measurebound1}
&\E\left(\frac{\# X^{*c}}{\# X}\right) \leq
\E\left(\frac{\# X_{S_1\cap S_2}}{\# X}\right)
= \frac{\mu_{gS}(\{x\in S_1 | \dist_g(x, S_2) <r\})}{\mu_{gS}(S)}+   \\
\nonumber
& \frac{\mu_{gS}(\{x\in S_2 | \dist_g(x, S_1) <r\})}{\mu_{gS}(S)}
\leq \frac{\mu_{gS}(\{x\in S_1\cup S_2| \dist_g(x, S_1\cap S_2) \leq C'r\})}{\mu_{gS}(S)}
\\
\nonumber
& \leq C_6 r^{d-\dim{S_1 \cap S_2}}.
\end{align}
The first equality of~\eqref{equ:measurebound1} follows from the fact that the dataset $X$ is i.i.d.~sampled from $\mu_{gS}$. The second inequality of~\eqref{equ:measurebound1} follows from Lemma~\ref{lemma:intercoord}. The last one follows from Theorem~1.3 in~\citet{Gray1982201},
where $C_6$ is a constant depending only on the geometry of the underlying generative model (e.g., the mean curvature and volume of $S_1\cap S_2$).

\subsection{Conclusion of Theorem~\ref{theorem:all} for the Noisy Multi-Geodesic Model}\label{sec:MainNoisy}

The above analysis also applies when the generative multi-geodesic model has noise level $\tau$ and $\tau$ is sufficiently smaller than $r$, that is,
\begin{equation}\label{equ:constraint5}
\tau<C_5 r,
\end{equation}
where $C_5\ll 1$. Indeed, in this case the estimates of tangent spaces and geodesics are sufficiently close to the estimates without noise. The only difference is that
the last bound in~\eqref{equ:measurebound1} has to be replaced with
$C_6 (r+\tau)^{d-\dim{S_1 \cap S_2}}$. This requires though a sufficiently small noise level (set by $C_5$).
Precise bound on $\tau$ is not trivial. Furthermore, the analysis employed here is not optimal. We can thus only claim in theory robustness to very small
levels of noise, whereas robustness to higher levels of noise is studied in the experiments.

%  (instead of $C_6 r^{d-\dim{S_1 \cap S_2}}$).

%Theorem~\ref{theorem:all} states the guarantee for the dataset of noise level $\tau$.

\section{Conclusions}

Aiming at efficiently organizing data embedded in a non-Euclidean space
according to low-dimensional structures, the present paper studied multi-manifold modeling
in such spaces.
The paper solves this clustering (or modeling) problem by
proposing the novel GCT algorithm. GCT thoroughly exploits the geometry of the
data to build a similarity matrix that can effectively cluster the data (via
spectral clustering) even when the underlying submanifolds intersect or have
different dimensions. In particular, it introduces the novel idea in
non-Euclidean multi-manifold modeling of using directional information from
local tangent spaces to avoid neighboring points of clusters different than that
of the query point. Theoretical guarantees for successful clustering were
established for a variant of GCT, namely TGCT for the MGM setting,
which is a non-Euclidean
generalization of the widely-used framework of hybrid-linear modeling.
Unlike TGCT, GCT combined
directional information from local tangent spaces with sparse coding, which aims
to improve the clustering result by the use of more succinct representations of
the underlying low-dimensional structures and by increasing robustness to
corruption. Geodesic information is only used locally and thus in practice the algorithm can fit well in practice to MMM and not MGM.
Validated against state-of-the-art existing methods for the
non-Euclidean setting, GCT exhibited notable performance in clustering
accuracy. More specifically, the paper tested GCT on synthetic and real data of
deformed images clustering, action identification in video sequences, brain
fiber segmentation in medical imaging and dynamic texture clustering.

\section{Acknowledgments}

This work was supported by the Digital Technology Initiative, a seed grant
program of the Digital Technology Center, University of Minnesota, NSF awards
DMS-09-56072, DMS-14-18386 and Eager-13-43860, the University of Minnesota
Doctoral Dissertation Fellowship Program, and the Feinberg Foundation Visiting
Faculty Program Fellowship of the Weizmann Institute of Science.

%\newpage

\appendix

\section{Competing Clustering Algorithms and Their Implementation Details}\label{sec:clusterings}
Section~\ref{sec:clusterings_review} reviews the competing methods of GCT (in the Riemannian setting) and Section~\ref{sec:clusterings_param}
describes the implementation of both GCT and the competing algorithms, in particular, the choice of all parameters.
\subsection{Review of Competing Algorithms}
\label{sec:clusterings_review}
The first competing algorithm is sparse manifold clustering (SMC). This
algorithm was first suggested by~\citet{ElhamifarV_nips11} for clustering
submanifolds embedded in Euclidean spaces and later modified by~\citet{6619442} for clustering submanifolds of the sphere.
We adapt it to the current setting of clustering submanifolds of a Riemannian manifold and still refer to it as SMC.
Its basic idea is as
follows: For each data-point $x$, a local neighborhood is mapped to the tangent space $T_x M$
by the logarithm map and a sparse coding task
is solved in $T_x M$ to provide weights for the spectral-clustering similarity matrix.

The second competing algorithm is spectral clustering with Riemannian metric (SCR)
by~\citet{GOH_VIDAL08}. It applies spectral clustering with the weight matrix $W$ whose entries are $\mb{W}_{ij}=e^{-\dist_g^2(x_i,x_j)/(2\sigma^2)}$ (see page 4
of~\citet{GOH_VIDAL08}). That is, it replaces the usual Euclidean metric in standard spectral clustering with the Riemannian one.

The third competing scheme is the embedded K-means. It embeds the given dataset, which
lies on a Riemannian manifold, into a Euclidean spaces (as explained next) and then applies the classical K-means to the embedded dataset. In the experiments, Grassmannian manifolds are embedded by a well-known isometric
embedding into Euclidean space~\citep{Basri_nearest_subspace11}; the manifolds of symmetric $n \times n$ PD matrices are embedded by vectorizing their elements into elements of $\reals^{\binom{n+1}{2}}$; and data in the sphere $\mathbb{S}^D$ is already embedded in $\reals^{D+1}$.

\subsection{Implementation Details for All Algorithms}
\label{sec:clusterings_param}
 GCT follows the scheme of
Algorithm~\ref{alg:experiment}. For all algorithms, the number $K$ of clusters was known in all experiments
The input parameters of GCT
are set as follows: The
neighborhood radius $r$ at a point $x$ is chosen to be the average distance of
$x$ to its $n$th nearest point over all $x$, where $n\in \{15,16,\ldots,30\}$;
the distance and angle thresholds $\sigma_d$ and $\sigma_a$ are set to $1$ in
all experiments (we did not notice a big difference of the results when their
values are changed). The dimension of the local tangent space is determined by
the largest gap of eigenvalues of each local covariance matrix
(more precisely, it is the number of eigenvalues until this gap).

Since there are no online available codes for SMC, SCR and EKM, we wrote our own implementations
and will post them (as well as our implementation of GCT) on the supplemental webpage when the paper is accepted for publication.
The spectral clustering code in GCT, SMC and SCR, as well as the $K$-means code in EKM are taken from the implementations of~\citet{spectral_code}.
To make a faithful comparison, the input parameter $r$ of SMC is the same as GCT (in particular, we use the radius of neighborhood and not the number of neighbors).
SMC also implicitly sets $\sigma_d=1$. There are no other parameters for SMC.
We remark that~\citet{ElhamifarV_nips11} formed the weight matrix $\mb{W}$ as follows: $\mb{W}_{ij}=|\mb{S}_{ij}|+|\mb{S}_{ji}|$, where $|\mb{S}_{ij}|$ and $|\mb{S}_{ji}|$ are the sparse coefficients.
However, this weight was unstable in some experiments and above a certain level of noise SMC often collapsed in some of the random repetition of the experiments. In such cases, we used instead (for all repetitive experiments for the same data set) the weights $\mb{W}_{ij}=\exp{(|\mb{S}_{ij}|+|\mb{S}_{ji}|)}$
suggested in~\citet{6619442} (which are similar to the ones of GCT).
In the case of no collapse with the former weights, we tried both weights and noticed that the weights $\mb{W}_{ij}=|\mb{S}_{ij}|+|\mb{S}_{ji}|$ always yielded more accurate results for SMC; we thus used them then even though they can give an advantage over GCT, which uses exponential weights. 
Overall, the weight
$\mb{W}_{ij}=|\mb{S}_{ij}|+|\mb{S}_{ji}|$ was used in the synthetic datasets II-VI of Section~\ref{subsec:synthetic}.
The exponential weight was used in the rest of the experiments, that is, in synthetic dataset I and in the real or stylized applications.
It was also used for dataset VI in Figure~\ref{fig:robust_sph} under noise levels mostly higher than the $0.025$ noise level used in Section~\ref{subsec:synthetic}. The collapse
phenomenon is evident in Figure~\ref{fig:robust_sph} for noise levels above $0.05$.

The SCR algorithm has only one parameter $\sigma_d$ which is set to 1 (similarly to the analogous parameter of GCT). EKM has no input parameters.

\section{Computation of Logarithm Maps and Distances}\label{sec:logarithmmaps}

We discuss the complexity of computing logarithm maps for
Grassmannians, symmetric PD matrices and spheres. We remark though that it is possible to compute the logarithm maps
for data sampled from more general Riemannian manifolds and without knowledge of the manifold, but at a significantly slower rate~\citep{MemoliS05}.
We also show that once the logarithm map is computed, then in all these cases the computation of the geodesic distances is of lower or equal order.

A fast way to compute the logarithm map of the Grassmannian $\G(p,\ell)$ (whose dimension is $D=\ell(p-\ell)$) is
provided in \citet{Gallivan03efficientalgorithms}. It requires a $p\times \ell$
matrix $L$, with orthogonal columns, and a $p \times p$ orthonormal matrix $R$
for each subspace, where the subspace is spanned by the columns of $L$, with $L$
comprising the first $k$ columns of $R$. Given two pairs ($L_1, R_1$) and
($L_2$,$R_2$) for two subspaces, one needs to compute $\log_{L_1} (L_2)$.
This computation, which is clarified in~\citet{Gallivan03efficientalgorithms}, includes the singular value decomposition of $L_1^TL_2$ and
$R_1^TL_2$. In total, the complexity is $\OO(p^2 \ell)$, or equivalently, $\OO((D/\ell+\ell)^2 \ell)$ (since $D=\ell(p-\ell)$).

For the set of $p \times p$ symmetric PD matrices (whose dimension is $D=p(p+1)/2$), \citet{DBLP:conf/icml/HoXV13} computes the logarithm $\log_{M_1}(M_2)$ of any
such matrices $M_1$ and $M_2$ by first finding the Cholesky
decomposition $M_1=GG^T$ and then computing $\log_{M_1}(M_2)=G\log(GM_2G)G$,
where the latter $\log$ is the matrix logarithm. The complexities of all
major operations (i.e., Cholesky decomposition, the matrix logarithm and the matrix
multiplication) are $\OO(p^3)$. Therefore, the total complexity is also of order $\OO(p^3)$, or equivalently, $\OO(D^{1.5})$
(since the dimension of the set of symmetric PD matrices is $D=p(p+1)/2$).

The formula for finding the logarithm map on $\Sb^D$ is (see~\citet{6619442})
\[
\displaystyle \log_{x_i} (x_j) = \frac{x_j - (x_i^T x_j) x_i}{\sqrt{1-
    (x_i^Tx_j)^2}} \cos^{-1}(x_i^Tx_j),
\]
where $x_i^Tx_j$ is the (Euclidean) dot-vector product. Since it involves inner products and basic operations (also coordinatewise), it takes $\OO(D)$ operations to compute it.

For $x_1,x_2\in M$, $\dist_g(x_1,
x_2)=\|\log_{x_1}(x_2)\|_2$. Once we have the image $\log_{x_1}(x_2)$ (which is
a vector in the tangent space), the Riemannian distance is computed as the
Euclidean norm of the image vector, which involves a computation of order $\OO(D)$.
Since the algorithm already computes the logarithm maps, the additional cost for computing the geodesic distances are of lower order than the logarithm maps in all 3 cases.

\section{Computational complexity of GCT and TGCT}\label{sec:complexity.GCT}

The computational complexity of GCT is examined per data-point
$x_i$. It involves the computation of Riemannian distances and the logarithm map, which depends on the Riemannian manifold $M$ (see estimates in Section~\ref{sec:logarithmmaps}). The complexity of
computing the Riemannian distance between $x_i$ and $x_j$ and the logarithm map for $x_j$ w.r.t.~$x_i$ are denoted by CR and CL respectively (their computational complexity
for the cases of the sphere, Grassmannian and PD matrices were discussed in Appendix~\ref{sec:logarithmmaps}).
A major part of GCT occurs in the $r$-neighborhood of $x_i$ (WLOG), where $r$ was defined as the average distance to the $30$th nearest point from the associated data-point.
To facilitate the analysis of computational complexity, we use instead of $r$ the parameter $k$ of $k$-nearest-neighbors ($k$-NN) around $x_i$. Due to the choice of $r$, we assume that $k \sim 30$.

The complexity for computing the $k$-NN of $x_i$ is
$\OO(N\cdot\text{CR}+k\log(N))$, where $\OO(N\cdot\text{CR})$ refers to the
complexity of computing $N-1$ distances, and $\OO(k\log(N))$ refers to the effort of
identifying the $k$ smallest ones. The second step of Algorithm~\ref{alg:experiment} is to solve the sparse
optimization task in \eqref{equ:l1minimization}. Notice that due to
$\|\cdot\|_2$, only the inner products of data-points are necessary to form the
loss function in \eqref{equ:l1minimization}, which entails a complexity of order
$\OO(D)$. Given that only $k$-NN are involved in \eqref{equ:l1minimization} and
that their inner products are required to form the loss,
\eqref{equ:l1minimization} is a small scale convex optimization task that can be
solved efficiently by any off-the-shelf solver such as the popular alternating
direction method of multipliers \citep{glowinski.marrocco.75, gabay.mercier.76}
or the Douglas-Rachford algorithm \citep{Bauschke.Combettes.book}. The third
step of Algorithm~\ref{alg:experiment} is to find the top eigenvectors of the sample covariance matrix defined by
the $k$ neighbors of $x_i$. As shown in Section~\ref{sec:algebraic_trick} below the complexity of this step is
$\OO(D+k^3)$. Finally, to compute geodesic angles, $\OO(N\cdot\text{CL}+ND)$
operations are necessary. Considering all $N$ data-points, the total complexity
for the main loop of GCT is $\OO(N^2(\text{CR} +
\text{CL}+D)+kN\log(N)+ND+Nk^3)$.
After the main loop, spectral clustering is invoked on the $N\times N$ affinity
matrix $\mb{W}$. The main computational burden is to identify $K$ eigenvectors
of an $N\times N$ matrix, which entails complexity of order $\OO(KN^2)$ ($K$ is
the number of clusters). In summary, the complexity of GCT is $\OO(N^2(\text{CR}
+ \text{CL}+ D+K)+kN\log(N)+ND+Nk^3)$.

Note that in TGCT, the weights of
non-neighboring points are set equal to zero, and geodesic angles are computed
only for neighboring points, reducing thus the complexity of this step to
$\OO(N)$. Moreover, the affinity matrix is sparse in TGCT, effecting thus a
potential decrease in the complexity of spectral clustering to the order of6
$\OO(N\log N)$~\citep{Knyazev01, Kushnir+10}.
Therefore, TGCT's complexity becomes $\OO(N^2
\text{CR}+(k+1)N\log(N)+kN(\text{CL}+D)+Nk^3)$. The only step that contributes
to $N^2$ in TGCT comes from $k$-NN. This complexity can be reduced by approximate nearest search.
For example, for both the Sphere and the Grassmannian, \citet{lsh_ans_iccv13} established an
$\OO(N^{\rho})$ algorithm for approximate nearest neighbor search, where $\rho>0$ is a sufficiently small parameter.
Therefore the total complexity of TGCT for these special cases can be of order
$\OO(N^{1+\rho} \text{CR}+(k+1)N\log(N)+kN(\text{CL}+D)+Nk^3)$ (this includes also the preprocessing for the approximate nearest neighbors algorithm).

\subsection{An Algebraic Trick for Fast Computation of the Tangent Subspace}
\label{sec:algebraic_trick}
Consider the $D \times k$ data matrix
$\mb{X}$ at a specific neighborhood with $k$ points. We
need to identify a few principal eigenvectors of the $D \times D$ covariance
matrix $\mb{XX}^T$. One can avoid such a costly direct computation (when $D$ is large)
by leveraging the following elementary facts from linear algebra: (i) If $(\lambda,\mb{v})$ is
an eigenvalue-eigenvector pair of $\mb{X}^T\mb{X}$, then $(\lambda, \mb{Xv})$ is
an eigenvalue-eigenvector pair of $\mb{XX}^T$, and (ii)
$\text{rank}(\mb{X}^T\mb{X}) = \text{rank}(\mb{XX}^T)$. These facts suggest that
the spectra of $\mb{X}^T\mb{X}$ and $\mb{XX}^T$ coincide, and thus it is
sufficient to compute the eigendecomposition of the much smaller $k\times k$
matrix $\mb{X}^T\mb{X}$, with complexity $\OO(k^3)$, which renders the overall
cost of eigendecomposition equal to $\OO(D+k^3)$, including, for example, the
cost of computing $\mb{Xv}$.

\section{Supplementary Details for the Proof of Theorem~\ref{theorem:all}}

\subsection{Proof of Lemma~\ref{lemma:intercoord}}\label{sec:lemma:intercoord:proof}
Suppose on the contrary that such a constant does not exist. Then there is a sequence $\{x_n\}_{n=1}^{\infty}\subset S_1\cup S_2$ such that
\begin{equation}\label{equ:intercoord}
\dist_g(x_n,S_1\cap S_2)\geq n\max\{\dist_g(x,S_1), \dist_g(x,S_2)\}.
\end{equation}
By picking a subsequence if necessary, assume WLOG that $\{x_n\}_{n=1}^{\infty}\subset
S_1$. Since $S_1$ is compact, there is always a
convergent subsequence. Therefore, one may assume that $\{x_n\}_{n=1}^{\infty}\subset
S_1$ is also convergent. We show that it converges to a point $z\in S_1\cap
S_2$.

Since $S_1\cup S_2$ and $S_1\cap S_2$ are compact, $\dist_g(x_n,S_1\cap S_2)$ is
bounded. Equation~\eqref{equ:intercoord} implies that
$\dist_g(x_n,S_2)\rightarrow 0$ as $n$ approaches infinity. Suppose
$\{x_n\}_{n=1}^{\infty}$ converges to a point $y\notin S_1\cap S_2$. Then
$\dist_g(x_n,S_2)\rightarrow \dist_g(y,S_2)>0$ since $y\notin S_2$. This is a
contradiction.

Now that $\{x_n\}_{n=1}^{\infty}$ converges to $z\in S_1\cap S_2$, one may
assume $\{x_n\}_{n=1}^{\infty}$ is in the normal coordinate chart $\Phi_z$ of
$B(z,r)$ for some fixed $r>0$. Denote $\mb{y}_n=\Phi_z^{-1}(x_n)$,
$L_1=\Phi_z^{-1}(S_1)$ and $L_2=\Phi_z^{-1}(S_2)$. Since both $S_1$ and $S_2$
are geodesic submanifolds, $L_1$ and $L_2$ are two subspaces in $T_z M$. The
sequence $\{\mb{y}_n\}_{n=1}^{\infty}\subset L_1$ approaches the
origin. Lemma~17 of~\citet{LocalPCA} states that
\[
\displaystyle \dist_E(\mb{y}_n, L_1\cap L_2) \leq
\frac{\dist_E(\mb{y}_n,L_2)}{\sin \theta_{\min}(L_1,L_2)},
\]
where $\theta_{\min}(L_1,L_2)$ is the minimal nonzero principal angle between
$L_1$ and $L_2$. Let $H$ be a subset of $B_z(\mb{0},r)$ and arbitrarily fix a
point $\mb{u}\in H$. It follows from~\eqref{equ:proof:DiffS5:1} (applied with
$R=\OO(r)$) that
\[
\dist_E(\mb{y}_n,\mb{u})(1-\OO(r^2))< \dist_g(\mb{y}_n,\mb{u}) <\dist_E(\mb{y}_n,\mb{u})(1+\OO(r^2)).
\]
Since the term $\OO(r^2)$ depends only on the metric $g$, not on $\mb{y}_n$ or
$\mb{u}$, it is easy to see that
\begin{equation}\label{equ:intercoord:proof1}
\dist_E(\mb{y}_n,H)(1-\OO(r^2))< \dist_g(\mb{y}_n,H) <\dist_E(\mb{y}_n,H)(1+\OO(r^2)).
\end{equation}
If we let $H=L_1\cap L_2$ then~\eqref{equ:intercoord:proof1} implies that
\[
\displaystyle \dist_g(\mb{y}_n, L_1\cap L_2) \leq
\frac{(1+\OO(r^2))\dist_g(\mb{y}_n,L_2)}{(1-\OO(r^2))\sin
  \theta_{\min}(L_1,L_2)}.
\]
This is equivalent to
\begin{equation*}
\begin{aligned}
\displaystyle \dist_g(x_n, S_1\cap S_2) &\leq
\frac{(1+\OO(r^2))\dist_g(x_n,S_2)}{(1-\OO(r^2))\sin \theta_{\min}(L_1,L_2)} <
\frac{2}{\sin \theta_0}\dist_g(x_n,S_2)
\end{aligned}
\end{equation*}
for a fixed small $r$. This contradicts~\eqref{equ:intercoord}.

\subsection{Proof
  of~\eqref{equ:lemma:Intersection2}}\label{sec:lemma:Intersection2:proof}

The measures $\mu_{x_0}$ and $\mu_z$ are used to denote the induced measures on
$\Phi_{x_0}^{-1}(B(x_0,r)\cap (S_1\cup S_2))$ and $\Phi_z^{-1}(B(x_0,r)\cap
(S_1\cup S_2))$ by $\mu_{gS_1}+\mu_{gS_2}$. Let $H=\Phi_{x_0}^{-1}(B(x_0,r)\cap
(S_1\cup S_2))$ and $\phi_{x_0}=\Phi_z^{-1}\circ \Phi_{x_0}$ be the transition
map. Note that
\begin{equation}\label{equ:intercovar}
\begin{aligned}
\displaystyle &\E_{\mu_{gS}}\mb{C}_{x_0}^z =\E_{\mu_z}((\mb{y}-\E_{\mu_z}\mb{y})\cdot (\mb{y}-\E_{\mu_z}\mb{y})^T)\\
&=
\frac{1}{\mu_z(\phi_{x_0}(H))^3}\int_{\mb{y}\in \phi_{x_0}(H)}
\left( \int_{\mb{u}\in \phi_{x_0}(H)} (\mb{y}-\mb{u})\mu_z(d\mb{u})\cdot \int_{\mb{u}\in \phi_{x_0}(H)} (\mb{y}-\mb{u})^T\mu_z(d\mb{u}) \right) \mu_z(d\mb{y}). \\
\end{aligned}
\end{equation}
%In~\eqref{equ:intercovar}, $\phi$ is an abbreviation for $\phi_{x_0}$.
Let $\mb{y}=\phi_{x_0}(\mb{x})$ and $\mb{u}=\phi_{x_0}(\mb{v})$. We note that
$\mb{x},\mb{v}\in B(\mb{0},r)$ and $\mb{y},\mb{u}\in B(\mb{0},(C'+1)r)$.
It follows from the triangle inequality, double application
of~\eqref{equ:proof:DiffS5:1} (first with $R=(C'+1)r$ and next with $R=r$), the
elementary bound
$\dist_E(\mb{r},\mb{s})  \leq 2 \diam(M)$, where $\mb{r}, \mb{s}$ are images by the logarithm map of points in $M$ and $\diam(M)$ is the diameter of $M$ and the identity $l_g(\mb{y},\mb{u})=l_g(\mb{x},\mb{v})$ (which holds since $\phi_{x_0}$ preserves the Riemannian distance) that
\begin{align}\label{equ:intercovar1}
&|\|\mb{y}-\mb{u}\|_2 - \|\mb{x}-\mb{v}\|_2| = |\|\mb{y}-\mb{u}\|_2 - l_g(\mb{y},\mb{u}) + l_g(\mb{y},\mb{u}) - \|\mb{x}-\mb{v}\|_2| \\
\nonumber
&\leq |\|\mb{y}-\mb{u}\|_2 - l_g(\mb{y},\mb{u})\| + \|l_g(\mb{x},\mb{v}) - \|\mb{x}-\mb{v}\|_2| \leq 2 C_4\diam(M)[(C'+1)^2+1]r^2.
\end{align}
Applying Taylor's expansion to $\mb{y}=\phi_{x_0}(\mb{x})$, and using the fact that $\|\mb{x}\|_2 \leq r$, we note that
\begin{equation}
\label{eq:tyler_expansion}
\|\mb{y}- \mb{b}_{x_0}-\mb{A}_{x_0}\mb{x}\|_2 \leq C''_S r^2,
\end{equation}
where $\mb{b}_{x_0}$ and $\mb{A}_{x_0}$ depend only on $x_0$ and $C''_S$ is a constant depending on the Riemannian metric $g$.
Applying the triangle inequality, \eqref{equ:intercovar1} and \eqref{eq:tyler_expansion} (first with $\mb{y}=\phi_{x_0}(\mb{x})$ and next with $\mb{u}=\phi_{x_0}(\mb{v})$
instead of $\mb{y}$) we conclude that for all $\mb{x},\mb{v}\in B_{x_0}(\mb{0},r)$
\begin{align}\label{equ:intercovar2}
&| \|\mb{A}_{x_0}(\mb{x}-\mb{v})\|_2- \|\mb{x}-\mb{v}\|_2 |\\
\nonumber
&\leq |\|\mb{y}-\mb{u}\|_2 - \|\mb{x}-\mb{v}\|_2|+ \|\mb{y}- \mb{b}_{x_0}-\mb{A}_{x_0}\mb{x}\|_2 + \|\mb{u}- \mb{b}_{x_0}-\mb{A}_{x_0}\mb{v}\|_2 \\
\nonumber
&\leq [2 C_4\diam(M)((C'+1)^2+1)+2C''_S] r^2.
\end{align}
In particular, suppose $\|\mb{x}-\mb{v}\|_2=r$, then~\eqref{equ:intercovar2} implies that
for any unit-length vectors $\mb{w} \in \reals^D$ ($\reals^D$ is identified with $T_{x_0}$)
\begin{equation}\label{equ:intercovar4}
| \|\mb{A}_{x_0} \mb{w} \|_2 -1 | \leq [2 C_4\diam(M)((C'+1)^2+1)+2C''_S] r.
\end{equation}

We prove below in Appendix~\ref{lemma:orthogonal} that there exists an orthogonal matrix $\mb{R}_{x_0}$  such that
\begin{equation}\label{equ:intercovar3}
\mb{A}_{x_0} = \mb{R}_{x_0}+\OO(r).
\end{equation}
This leads to
\[
\mb{y}= \mb{b}_{x_0}+\mb{R}_{x_0}\mb{x}+ \OO(r^2) \ \text{ and } \ \mb{u}= \mb{b}_{x_0}+\mb{R}_{x_0}\mb{v}+ \OO(r^2).
\]
Consequently,
\begin{equation}
\label{eq:change_variables}
\mb{y} -\mb{u}= \mb{R}_{x_0}(\mb{x}-\mb{v}) + \OO(r^2).
\end{equation}
We also note that since $\mu_z$ and $\mu_{x_0}$ are induced from $\mu$, then
\begin{equation}
\label{eq:measure_preserve}
\mu_z(\phi_{x_0}(H))=\mu_{x_0}(H)
\end{equation}
At last, \eqref{equ:lemma:Intersection2} is concluded by applying~\eqref{equ:intercovar} (first with $\mb{y}$ and $\mb{u}$ and next with $\mb{x}$ and $\mb{v}$ while using appropriate change of variables), \eqref{eq:change_variables} and~\eqref{eq:measure_preserve}.

\subsubsection{Proof of~\eqref{equ:intercovar3}}\label{lemma:orthogonal}

We show that if $\mb{A}$ is an $D \times D$ matrix such that $| \|\mb{A} \mb{w}
\|_2 -1 | \leq C r$ for all unit-length vectors $\mb{w}\in \reals^D$ and a fixed
constant $C>0$, then there exists an orthogonal matrix $\mb{R}$ such that
$\mb{A} = \mb{R}+ \OO(r)$. In other words, the $ij$th entries of $\mb{A}$ and
$\mb{R}$ satisfy
\begin{equation}
\label{eq:A_ij_R}
| \mb{A}_{ij} -\mb{R}_{ij} | \leq f(C,D) r
\end{equation}
for a bounded function $f$ (we only show below that the RHS of~\eqref{eq:A_ij_R} is bounded by a constant times $r$, but it is not hard to see that this constant depends on $C$ and $D$; this dependence is used later in~\eqref{eq:use_f} in order to provide a clearer idea of the constant $C'''_S$).

By performing Gram-Schmidt orthogonalization on rows, the matrix $\mb{A}$ can be written as a product of an upper triangular matrix $\mb{U}$ and an orthogonal matrix $\mb{R}$ (this is the $\mb{RQ}$ decomposition of $\mb{A}$, but with $\mb{U}$ and $\mb{R}$ used instead of $\mb{R}$ and $\mb{Q}$ respectively). Since $\mb{R}$ preserves the length of vectors, the condition on $\mb{A}$ becomes
\begin{equation}\label{equ:u}
| \|\mb{U} \mb{w} \|_2 -1 | \leq C r,
\end{equation}
for all unit-length vectors $\mb{w}$. It is enough to show that up to a change of sign of the rows of $\mb{R}$: $\mb{U}=\mb{I}+\OO(r)$. This is proved by induction on $D$.

If $D=1$, then $\mb{U}$ is a $1 \times 1$ matrix. Let $\mb{w}= 1$. In this case~\eqref{equ:u} implies that $\mb{U}=\pm1+\OO( r)$.
By possible change of sign of $\mb{R}$ we conclude that $\mb{U}= 1+\OO( r)$.

We assume that the claim is true for $D=k-1$. Let $\mb{U}$ be a $k \times k$ upper rectangular matrix and express it as follows:
\[
\mb{U} = \left( \begin{array}{cc}
\mb{V}_{k-1\times k-1} & \mb{x}_{k-1\times 1} \\
\mb{0}_{1\times k-1} & \mb{U}_{kk}  \end{array} \right),
\]
where $\mb{V}$ is $(k-1) \times (k-1)$ upper triangular matrix, $\mb{0}_{1\times k-1}$ is a row vector of $k-1$ zeros, $\mb{x}_{k-1\times 1}$ is a column vector in $\reals^{k-1}$  and
$\mb{U}_{kk} \in \reals$. We assume that $\mb{U}$ satisfies~\eqref{equ:u} and show that $\mb{U}=\mb{I}+\OO(r)$ by basic estimates with different choices of $\mb{w} \in \reals^k$ used in~\eqref{equ:u}.

Assume first that $\mb{w}=[\mb{v}^T, 0]^T$, where $\mb{v} \in \reals^{k-1}$ is of unit-length. Then \eqref{equ:u} implies that
\begin{equation}
\label{eq:prop_V}
| \|\mb{V} \mb{v} \|_2 -1 | \leq C r.
\end{equation}
The induction hypothesis and~\eqref{eq:prop_V} results in the estimate
\begin{equation}
\label{eq:V_I}
\mb{V} =\mb{I} +\OO(r)
\end{equation}
up to a change of sign in the first $k-1$ rows of $\mb{R}$ (the rotation associated with $\mb{U}$).

Next, we show that  $\mb{U}_{kk} = 1+\OO(r)$.  We first let $\mb{w}=[\mb{0}_{1\times k-1}, 1]^T$; in this case \eqref{equ:u} implies that
\begin{equation}\label{equ:u0}
 \sqrt{\|\mb{x}\|_2^2+\mb{U}_{kk}^2}-1 = \OO(r),
\end{equation}
which leads to
\begin{equation}\label{equ:u1}
\|\mb{x}\|_2^2,\, |\mb{U}_{kk}|^2 \leq 1+\OO(r).
\end{equation}
We next let $\mb{w}=[-\mb{x}^T, 1]^T/\| [-\mb{x}^T, 1]^T \|_2$. Then~\eqref{equ:u}, with $\|\mb{x}\|_2^2$ being bounded by $1+\OO(r)$, implies that
\begin{equation}\label{equ:u2}
\sqrt{\mb{U}_{kk}^2+\OO(r^2)} - \| [-\mb{x}^T, 1]^T \|_2   = \OO(r).
\end{equation}
Moving the second term of the LHS of~\eqref{equ:u2} to the RHS of~\eqref{equ:u2} and squaring both sides result in
\begin{equation}\label{equ:u3}
\mb{U}_{kk}^2 \geq \| [-\mb{x}^T, 1]^T \|_2^2 - \OO(r) \geq 1-O(r).
\end{equation}
The combination of~\eqref{equ:u1} and~\eqref{equ:u3} implies that
\begin{equation}\label{equ:u4}
|\mb{U}_{kk}^2-1| \leq \OO(r).
\end{equation}
Since $\mb{U}_{kk} \geq 0$ WLOG (otherwise one can change the sign of the $k$th row of $\mb{R}$)
and since $|\mb{U}_{kk}^2-1|$ is a Lipschitz function on $\mb{U}_{kk}$, \eqref{equ:u4} implies that
\begin{equation}\label{equ:u5}
|\mb{U}_{kk}-1| \leq \OO(r).
\end{equation}
In other words, $\mb{U}_{kk}=1+\OO(r)$.

At last, we show that $\mb{x}_i = \OO(r)$. Moving the second term of the LHS of~\eqref{equ:u0} to the RHS of~\eqref{equ:u0} and squaring both sides
result in
\begin{equation}
\label{eq:x_and U}
\|{x}\|_2^2+\mb{U}_{kk}^2 =1+ \OO(r).
\end{equation}
It follows from~\eqref{equ:u5} and~\eqref{eq:x_and U} that $ \|\mb{x}\|_2^2 = \OO(r)$, which implies that
\begin{equation}
\label{eq:x_sqrt_r}
\mb{x}_i = \OO(\sqrt{r}).
\end{equation}
Denote the standard basis of $\R^k$ by $\{\mb{e}_i\}_{i=1}^k$, that is,
$\mb{e}_1=[1,0,\ldots, 0]^T,\ldots, \mb{e}_k=[0,\ldots,0,1]^T$. Let $\mb{w}_i=\frac{\sqrt{2}}{2}\mb{e}_{i}+\frac{\sqrt{2}}{2}\mb{e}_k$. Plugging $\mb{w}_i$ into \eqref{equ:u} and further simplification result in
\begin{equation}\label{equ:u6}
\left[ \frac{1}{2}(\mb{x}_1^2+\ldots+\mb{x}_{k-1}^2)+1+\frac{1}{2}\mb{x}_i+\OO(r) \right]^{1/2}=1+\OO(r).
\end{equation}
Further application of~\eqref{eq:x_sqrt_r} into~\eqref{equ:u6} yields the equality
\begin{equation}\label{equ:u7}
\left[ 1+\frac{1}{2}\mb{x}_i+\OO(r) \right]^{1/2}=1+\OO(r).
\end{equation}
Finally, squaring both sides of~\eqref{equ:u7} and simplifying concludes the desired estimate
\begin{equation}
\label{eq:x_r}
\mb{x}_i =\OO(r).
\end{equation}
Equations~\eqref{eq:V_I}, \eqref{equ:u5} and~\eqref{eq:x_r} imply that $\mb{U}=\mb{I}+\OO(r)$ (up to a change of signs of the rows of $\mb{R}$) and thus conclude the induction and consequently~\eqref{equ:intercovar3}.

\subsection{Proof of~\eqref{equ:interarea}}\label{sec:interarea}

Let $H_1=B_I(\Phi_z^{-1}(x_0),r-\OO(r^2))\cap \Phi_z^{-1}(S_1 \cup S_2)$ and $H_2=B_I(\Phi_z^{-1}(x_0),r+\OO(r^2))\cap \Phi_z^{-1}(S_1 \cup S_2)$. It follows from~\eqref{equ:proof:DiffS5:1} (applied with $R= \OO(r)$) that
\begin{equation}
\label{eq:triple_intersection1}
B_I(\Phi_z^{-1}(x_0),r-\OO(r^2)) \subset \Phi_z^{-1}(B(x_0,r))\subset B_I(\Phi_z^{-1}(x_0),r+\OO(r^2)).
\end{equation}
The intersection of all sets in~\eqref{eq:triple_intersection1} with $L_1 \cup L_2 = \Phi_z^{-1}(S_1 \cup S_2)$ and the definitions of $H_1$, $H_2$ and $H'$ result in the set inequality
\begin{equation}
\label{eq:triple_intersection2}
H_1 \subset H' \subset H_2.
\end{equation}
Thus,
\begin{equation}\label{sec:interarea:equ1}
H\setminus H'\subset H_2\setminus H',\quad H'\setminus H\subset H'\setminus H_1.
\end{equation}
By first applying~\eqref{sec:interarea:equ1} (or its consequence $(H_2\setminus H')\cup (H'\setminus H_1)=H_2 \setminus H_1$) and then direct estimates (whose details are excluded) we obtain that
\begin{equation}\label{sec:interarea:equ2}
\mu_{ES}((H_2\setminus H')\cup (H'\setminus H_1)) = \mu_{ES}(H_2\setminus H_1) = \OO(r)\mu_{ES}(H_1).
\end{equation}
Finally, \eqref{equ:interarea} follows from~\eqref{sec:interarea:equ1} and~\eqref{sec:interarea:equ2}.

\subsection{Proof of~\eqref{equ:proof:DiffS5:1}}
\label{sec:DistInNeighbor}

Denote by $l(t)$ the parameterized line segment in $T_{x_1}M$ connecting $l(0)=\mb{x}$ and $l(1)=\mb{y}$, where $\mb{x}$ and $\mb{y}$ are specified in~\eqref{equ:proof:DiffS5:1}.
We note that
\begin{align}\label{equ:proof:DiffS5:2}
\dist_g(\mb{x},\mb{y}) &=\int_0^1 \sqrt{l'(t)^T g(l(t))l'(t)}dt=\int_0^1 \sqrt{l'(t)^T (I+\OO(R^2))l'(t)}dt \nonumber \\
&=\dist_E(\mb{x},\mb{y})+\OO(R^2)\dist_E(\mb{x},\mb{y}).
\end{align}
Equation~\eqref{equ:proof:DiffS5:2} clearly implies~\eqref{equ:proof:DiffS5:1}, where $C_4>0$ depends only on the Riemannian manifold $M$.

\subsection{Proof of \eqref{equ:DiffS7}}\label{proof:equ:DiffS7}

We first claim that for any $\alpha>0$
\begin{equation}\label{eq:proof:equ:DiffS7:1}
\sin(\theta_{\max}(T_{x_1}^ES_1,T_{x_1}S_1)) \leq
\|\mb{P}_{T_{x_1}^ES_1}-\mb{P}_{T_{x_1}
  S_1}\|<\frac{\sqrt{2}\|\mb{C}_{x_1}-\frac{\alpha r^2}{d+2}\mb{P}_{T_{x_1}
    S_1}\|}{\frac{\alpha r^2}{d+2}}.
\end{equation}
The first inequality of~\eqref{eq:proof:equ:DiffS7:1} follows from Lemma~15 in~\citet{LocalPCA}.
Whereas the second inequality follows from the
Davis-Kahan Theorem~\citep{1970SJNA....7....1D}.

For the rest of the proof we upper bound the RHS of~\eqref{eq:proof:equ:DiffS7:1}.
We work
in the tangent space $T_z M$, where $z$ is defined as
\[
\displaystyle z=\argmin_{y\in S_1\cap S_2} \dist_g(x_1,y).
\]
Similarly as in the proof of Proposition~\ref{lemma:Intersection}, if argmin is not uniquely
defined, $z$ is arbitrarily chosen among all minimizers.
Let the composition map $\phi_{x_1}=\Phi_z^{-1}\circ \Phi_{x_1}$ be the
transition map from $T_{x_1} M$ to $T_z M$.
Note that $\phi_{x_1}$ maps the subspace $T_{x_1}S_1$ to another subspace $T_z
S_1$.
Let $\mb{R}_{x_1}(L_1)$ denote the image of $L_1$ in $T_z M$ under the
rotation matrix $\mb{R}_{x_1}$ (here we identify both $T_{x_1}M$ and $T_z M$
with $\R^D$ via their normal coordinate charts).
Using the new terminology the main term in the RHS of~\eqref{eq:proof:equ:DiffS7:1} can be expressed as follows
\begin{equation}\label{eq:proof:equ:DiffS7:7}
\|\mb{C}_{x_1}-\frac{\alpha r^2}{d+2}\mb{P}_{T_{x_1} S_1}\|
=
\|\mb{R}_{x_1}\mb{C}_{x_1}\mb{R}_{x_1}^T-\frac{\alpha
    r^2}{d+2}\mb{P}_{\mb{R}_{x_1}(T_{x_1} S_1)}\|.
\end{equation}

The RHS of~\eqref{eq:proof:equ:DiffS7:7} can be bounded by the triangle inequality and~\eqref{equ:IntersectionAdapt} as follows
\begin{align}\label{eq:proof:equ:DiffS7:4}
&\|\mb{R}_{x_1}\mb{C}_{x_1}\mb{R}_{x_1}^T-\frac{\alpha r^2}{d+2}\mb{P}_{\mb{R}_{x_1}(T_{x_1} S_1)}\|
\leq  \|\mb{R}_{x_1}\mb{C}_{x_1}\mb{R}_{x_1}^T -
\E_{\mu_{ES}}\mb{C}_{H'}\|
\nonumber \\
&+ \|\E_{\mu_{ES}}\mb{C}_{H'} - \frac{\alpha r^2}{d+2}\mb{P}_{T_z S_1}\|
+\|\frac{\alpha r^2}{d+2}\mb{P}_{T_z S_1} - \frac{\alpha r^2}{d+2}\mb{P}_{\mb{R}_{x_1}(T_{x_1} S_1)}\| \nonumber \\
& \leq  C'_S r^3 + \|\E_{\mu_{ES}}\mb{C}_{H'} - \frac{\alpha r^2}{d+2}\mb{P}_{T_z S_1}\| + \|\frac{\alpha r^2}{d+2}\mb{P}_{T_z S_1} - \frac{\alpha r^2}{d+2}\mb{P}_{\mb{R}_{x_1}(T_{x_1} S_1)}\| .
\end{align}

Next, we bound the last term in the RHS of~\eqref{eq:proof:equ:DiffS7:4}.
It follows from~\eqref{equ:intercovar4}, \eqref{eq:tyler_expansion}, \eqref{eq:A_ij_R}
(which implies \eqref{equ:intercovar3}) that for $\mb{y}=\phi_{x_1}(\mb{x})$
\begin{equation}\label{eq:proof:equ:DiffS7:2}
\|\mb{y}- \mb{b}_{x_1}-\mb{R}_{x_1}\mb{x}\|_2 \leq C'''_S r^2 \quad \forall \|\mb{x}\|_2 \leq r,
\end{equation}
where
\begin{equation}
\label{eq:use_f}
C'''_S=D \cdot f( 2 C_4\diam(M)((C'+1)^2+1)+2C''_S , D)+C''_S.
\end{equation}
It is immediate to see that $\mb{b}\in T_z S_1$ by letting $\mb{x}=\mb{0}$
in the Taylor's expansion.
If $\mb{v}\in
\mb{R}_{x_1}(T_{x_1} S_1)$ is a vector such that $\|\mb{v}\|_2=r$ and $\theta(\mb{v},
T_z S_1)=\theta_{\max}(\mb{R}_{x_1}(T_{x_1} S_1), T_z S_1)$,
then~\eqref{eq:proof:equ:DiffS7:2} and the fact that
$\phi_{x_1}(\mb{R}_{x_1}^{-1}\mb{v})-\mb{b}\in T_z S_1$ imply that $\dist(\mb{v}, T_z S_1) \leq
C'''_S r^2$. Consequently,
\begin{equation}
\label{eq:last_maybe}
\| \mb{P}_{\mb{R}_{x_1}(T_{x_1} S_1)}-\mb{P}_{T_z S_1} \|=\displaystyle
\sin(\theta_{\max}(\mb{R}_{x_1}(T_{x_1} S_1), T_z S_1)) = \frac{\dist(\mb{v},
  T_z S_1)}{\|\mb{v}\|_2} \leq C'''_S r.
\end{equation}

If $\alpha_0 = (1+(1-\delta^2(x_1))_+^{d/2})^{-1}$ (the same as in Lemma~21 of~\citet{LocalPCA}), then the argument in~\citet[page 41]{LocalPCA} shows that
\begin{equation}\label{eq:proof:equ:DiffS7:5}
\|\E_{\mu_{ES}}\mb{C}_{H'} - \frac{\alpha_0 r^2}{d+2}\mb{P}_{T_z S_1}\| \leq 2 C_2^{\frac{d}{2}} \eta^{\frac{d}{d+2}} r^2.
\end{equation}
Inequalities~\eqref{eq:proof:equ:DiffS7:4} (with $\alpha=\alpha_0$), \eqref{eq:last_maybe} (with $\alpha=\alpha_0$) and~\eqref{eq:proof:equ:DiffS7:5} imply that
\begin{equation}\label{eq:proof:equ:DiffS7:6}
\|\mb{R}_{x_1}\mb{C}_{x_1}\mb{R}_{x_1}^T-\frac{\alpha_0 r^2}{d+2}\mb{P}_{\mb{R}_{x_1}(T_{x_1} S_1)}\| \leq C'_S r^3+2 C_2^{\frac{d}{2}} \eta^{\frac{d}{d+2}} r^2 + \frac{C'''_S \alpha_0}{d+2} r^3.
\end{equation}
Plugging~\eqref{eq:proof:equ:DiffS7:7} (with $\alpha=\alpha_0$) and~\eqref{eq:proof:equ:DiffS7:6} in~\eqref{eq:proof:equ:DiffS7:1}  (with $\alpha=\alpha_0$) and applying the fact that $\frac{1}{2}\leq \alpha_0 \leq 1$ yield
\begin{equation}\label{eq:proof:equ:DiffS7:8}
\sin(\theta_{\max}(T_{x_1}^ES_1,T_{x_1}S_1)) < 2\sqrt{2}(d+2)(C'_S r+2 C_2^{\frac{d}{2}} \eta^{\frac{d}{d+2}} + \frac{C'''_S}{d+2} r).
\end{equation}
Let
\begin{equation}
\label{eq:def_c3prime}
C'_3 = 2\sqrt{2}(d+2)\max(2 C_2^{\frac{d}{2}}, C'_S+\frac{C'''_S}{d+2}),
\end{equation}
then~\eqref{equ:DiffS7} clearly follows from~\eqref{eq:proof:equ:DiffS7:8} and~\eqref{eq:def_c3prime}.

\subsection{Proof of~\eqref{equ:DiffS8}}\label{proof:equ:DiffS8}
We prove\eqref{equ:DiffS8}, while generalizing the setting to work with two subspaces $L_1$, $L_2$ and a line $l$. Let
$\angle(l, L_1)=\theta_1$ and $\angle(l, L_2)=\theta_2$. Assume that
\begin{equation}
\label{eq:theta_alpha}
\theta_1\leq \alpha
\end{equation}
for an arbitrarily fixed $0 <\alpha< \pi/2$. We use the fact that $\sin(\theta)$ is
a concave function. If $\theta_2 >
\alpha$, then
\begin{equation}\label{equ:proof:DiffS8:equ2}
\displaystyle \frac{1-\sin(\alpha)}{\pi/2-\alpha} \leq
\frac{\sin(\theta_2)-\sin(\alpha)}{\theta_2-\alpha} <
\frac{\sin(\theta_2)-\sin(\theta_1)}{\theta_2-\theta_1}.
\end{equation}
On the other hand, the fact that $\sin^{-1}(x)$ is a Lipschitz function over the
interval $[0, \sin(\alpha)]$ implies that if $\theta_2 \leq \alpha$,
\begin{equation}\label{equ:proof:DiffS8:equ3}
\displaystyle |\theta_2-\theta_1| \leq \frac{1}{\cos(\alpha)}
|\sin(\theta_2)-\sin(\theta_1)| \quad \text{for $\theta_1,\theta_2\in
  [0,\alpha]$.}
\end{equation}
Equation~\eqref{equ:proof:DiffS8:equ2} and~\eqref{equ:proof:DiffS8:equ3} imply that
\begin{equation}\label{equ:proof:DiffS8:equ4}
\displaystyle |\theta_2-\theta_1| \leq \max\left(\frac{\pi/2-\alpha}{1-\sin(\alpha)}, \frac{1}{\cos(\alpha)}\right) |\sin(\theta_2)-\sin(\theta_1)|.
\end{equation}

If $\alpha=\pi/6$, then~\eqref{equ:proof:DiffS8:equ4} and Lemma~3.2
of~\citet{lerman.zhang.lp.recovery.14} lead to the inequality
\begin{equation}
\displaystyle |\theta_2-\theta_1| \leq \frac{2\pi\sqrt{d}}{3}\sin(\theta_{\max}(L_1, L_2)).
\end{equation}
Thus,
\begin{equation}\label{equ:proof:DiffS8.equ5}
\theta_1 \geq \theta_2-\frac{2\pi\sqrt{d}}{3}\sin(\theta_{\max}(L_1, L_2))
\end{equation}
as long as \eqref{eq:theta_alpha} holds.
If \eqref{eq:theta_alpha} is not assume, \eqref{equ:proof:DiffS8.equ5} can be replaced with
\begin{equation}
\theta_1 \geq \min(\theta_2-\frac{2\pi\sqrt{d}}{3}\sin(\theta_{\max}(L_1, L_2)), \pi/6) \quad \forall \theta_1\in[0,\pi/2],
\end{equation}
which translates to \eqref{equ:DiffS8}.

%%%%%%%%%%%%%%%%%%%%%%%%%%%%%%%%%%%%%%%%%%%%%%%%%%%%%
%%%%%%%%%%%%%%%%%%%%%%%%%%%%%%%%%%%%%%%%%%%%%
\newpage
%\bibliography{./bib/biblio_jmlr}

\end{document}